\documentclass[11pt,letterpaper]{article}

\usepackage{booktabs}

\usepackage{latexsym}
\usepackage{amssymb,amsmath, bm,pgfplots,tikz,bbm}
\usepackage{graphicx}
\usepackage{marvosym}
\usepackage{multirow,float}
\usepackage{algpseudocode}
\usepackage{caption}
\usepackage{mathtools}
\usepackage{tocloft}
\usepackage{etoc}
\usepackage{subcaption}
\usepackage{comment}
\usepackage{enumitem}

\usepackage{natbib}
\usepackage{color}
\usepackage[bookmarksopen=true, bookmarksnumbered=true,
pdfstartview=FitH, breaklinks=true, urlbordercolor={0 1 0}, citebordercolor={0 0 1}]{hyperref}

\usepackage{dcolumn}
\newcolumntype{.}{D{.}{.}{-1}}
\newcolumntype{d}[1]{D{.}{.}{#1}}

\usepackage{theorem}
\theoremstyle{plain}
\theoremheaderfont{\scshape}
\newtheorem{theorem}{Theorem}
\newtheorem{proposition}{Proposition}
\newtheorem{assumption}{Assumption}

\newtheorem{lemma}{Lemma}
\newtheorem{remark}{Remark}
\newcommand{\qed}{\hfill \ensuremath{\Box}}
\newcommand{\indep}{\mbox{$\perp\!\!\!\perp$}}

\providecommand{\norm}[1]{\lVert#1\rVert}
\newenvironment{proof}{\vspace{1ex}\noindent{\bf Proof}\hspace{0.5em}}
{\hfill\qed\vspace{1ex}}
\usepackage{kantlipsum}
\allowdisplaybreaks

\usepackage{rotating}
\usepackage[table]{xcolor}%
\usepackage{arydshln}
\usepackage{threeparttable}

\usepackage[compact]{titlesec}
\usepackage[listings,breakable]{tcolorbox}

\usepackage[ruled,linesnumbered,vlined]{algorithm2e}
\usepackage{pifont}

\newtheorem{definition}{Definition}

\newcommand\E{\mathbb{E}}
\newcommand\R{\mathbb{R}}
\renewcommand\P{\mathbb{P}}
\newcommand\cP{\mathcal{P}}
\newcommand\V{\mathbb{V}}

\newcommand\bR{\bm{R}}
\newcommand\bW{\bm{W}}

\newcommand\bz{\bm{z}}
\newcommand\bZ{\bm{Z}}

\newcommand\bY{\bm{Y}}
\newcommand\boldf{\bm{f}}

\newcommand\cT{\mathcal{T}}
\newcommand\cL{\mathcal{L}}

\newcommand\cY{\mathcal{Y}}
\newcommand\cZ{\mathcal{Z}}

\usetikzlibrary{decorations.markings}
\usetikzlibrary{decorations.pathmorphing}
\usetikzlibrary{shapes.geometric, arrows}
\usetikzlibrary{arrows,decorations.pathmorphing,backgrounds,positioning,fit,matrix}
\usetikzlibrary{shapes,decorations,arrows,calc,arrows.meta,fit,positioning}
\tikzset{auto,node distance =1 cm and 1 cm,semithick,
	state/.style ={circle, draw, minimum width = 0.7 cm},
	point/.style = {circle, draw, inner sep=0.04cm,fill,node contents={}},
	bidirected/.style={Latex-Latex,dashed},
	el/.style = {inner sep=2pt, align=left, sloped}
}

\newcommand{\tp}{{\scriptscriptstyle\mathsf{T}}}
\graphicspath{{figs/}} 
\usepackage{comment}
\usepackage{bbm}

\begin{document}

\title{\bf Surrogate Representation Inference for Text and Image Annotations\thanks{I thank  Naoki Egami, Gary King, Naijia Liu, Luke Miratrix, Jacob Montgomery, Soichiro Yamauchi, the members of Imai research group, and an anonymous reviewer of Harvard’s Alexander and Diviya Magaro Peer Pre-Review Program for helpful comments. I especially thank Kosuke Imai for his continuous encouragement and detailed comments.  }}  
\author{ Kentaro Nakamura\thanks{Ph.D. student, John F. Kennedy School of Government, Harvard University. Email: \href{mailto:knakamura@g.harvard.edu}{knakamura@g.harvard.edu }; URL: \href{https://k-nakam.github.io/}{https://k-nakam.github.io/} }}
\date{
Last Updated: \today
}
\maketitle\thispagestyle{empty}

\begin{abstract}
\noindent As researchers increasingly rely on machine learning models and LLMs to annotate unstructured data, such as texts or images, various approaches have been proposed to correct bias in downstream statistical analysis.  However, existing methods tend to yield large standard errors and require some error-free human annotation. In this paper, I introduce \emph{Surrogate Representation Inference} (SRI), which assumes that unstructured data fully mediate the relationship between human annotations and structured variables.
The assumption is guaranteed by design provided that human coders rely only on unstructured data for annotation.  Under this setting, I propose a neural network architecture that learns a low-dimensional representation of unstructured data such that the surrogate assumption remains to be satisfied. When multiple human annotations are available, SRI can be extended to further correct non-differential measurement errors that may exist in human annotations. Focusing on text-as-outcome settings, I formally establish the identification conditions and semiparametric efficient estimation strategies that enable learning and leveraging such a low-dimensional representation. Simulation studies and a real-world application demonstrate that SRI reduces standard errors by over 50\% when machine learning classification accuracy is moderate and provides valid inference even when human annotations contain non-differential measurement errors.
\end{abstract}


\newpage 
\vspace{1cm}

\section{Introduction}
With a wide availability of texts and images, the use of these unstructured data for statistical inference has become increasingly common in many fields, including political science (e.g. \citealt{grimmer2022text, stukal_why_2022, wratil_government_2023}), economics (e.g. \citealt{gentzkow2019measuring, hassan2019firm, ash2023are}), and sociology (e.g. \citealt{kozlowski2019geometry, macanovic2022text}). While researchers use machine learning and large language models (LLMs) to annotate such data, their predictions contain errors that invalidate the downstream statistical inference.
As a result, various approaches have been proposed to correct the bias in the downstream analysis by combining the human annotations and machine learning or LLM predictions (e.g. \citealt{fong_machine_2021, angelopoulos2023prediction, egami_neulips_2023, egami2024using, zrnic_cross-prediction-powered_2024, carlson_unifying_2025}). However, the existing methodology tends to yield large standard errors, especially when the quality of machine learning predictions is moderate.

In this paper, I introduce a new methodological framework, \emph{Surrogate Representation Inference} (SRI), which improves efficiency and yields statistically valid estimates. SRI improves efficiency based on the assumption that unstructured data work as a surrogate variable, fully mediating the relations between human annotations and structured variables of interest, such as predictors. This assumption is guaranteed by design if human coders rely only on texts and images for the annotation task and do not directly look at other relevant structured variables. Since unstructured data are high-dimensional, I propose using low-dimensional representations that still satisfy the surrogacy assumption. I refer to these as \emph{surrogate representations}. Because surrogate representation can leverage surrogate structure, the proposed SRI methodology can improve efficiency.

To estimate surrogate representation and perform subsequent statistical inference, I introduce a neural network architecture that jointly predicts human annotations and structured variables. 
As a concrete example, I demonstrate the proposed SRI methodology in the setting where the text-based variable, such as sentiments or topics, serves as an outcome.
I show that the covariate-adjusted outcome mean (i.e., the average value of the text-based variable for each level of the predictors after adjusting for the set of control variables)
is nonparametrically identified using surrogate representation. This surrogate representation can be estimated with the proposed neural network architecture. Drawing on the literature of semiparametric statistics (e.g. \citealt{tsiatis_semiparametric_2006, chernozhukov_doubledebiased_2018, kennedy_semiparametric_2023}), I derive the semiparametric efficient estimation strategies tailored to the proposed neural network architectures and formalize the conditions under which the SRI estimator is valid and efficient.

The SRI framework can be further extended to adjust for measurement errors in human annotations when multiple human annotations are available. Specifically, I show that when the measurement error in human annotations is \emph{non-differential} (i.e., independent across coders given the true label and unrelated to the textual content except through the true label), the quantity of interest can be nonparametrically identified and estimated from multiple annotations using the literature of double negative control (e.g., \citealt{kuroki_measurement_2014, tchetgen_introduction_2020, zhou2024causal}). Since perfect human annotations are not required, once researchers can create a solid coding rule, they can then leverage existing crowd-sourcing platforms, such as Amazon Mechanical Turk and CrowdFlower, to generate human-annotated labels with relative ease.

I finally apply the proposed SRI methodology to the existing research on the framing of immigrants in the U.S. Congressional speeches  \citep{card_computational_2022}. I compare the proposed SRI estimator with the existing bias correction algorithm (\citealt{angelopoulos2023prediction, egami_neulips_2023,egami2024using}) that leverages both machine learning predictions and human annotations. I find that the direct use of machine learning predictions inflates the partisan difference in the framing of immigrants. Both SRI and the existing bias-correction approach yield consistent point estimates, but because tone classification is challenging, SRI significantly reduces the standard errors.

In Appendix \ref{sec::simulation}, I also validate the performance of the proposed estimator through simulation studies, in which I systematically vary the quality of the machine learning predictions. The findings reveal that SRI outperforms the existing approach across all settings except the case that the machine learning predictions are already almost perfect (above 90\% classification accuracy), with the efficiency gains becoming particularly pronounced when classification accuracy for proxy is moderate (e.g., below 85\%). I further find that when human annotations contain classification errors, the existing bias-correction approach, which relies on the assumption of perfect human annotations, suffers from substantial bias and improper confidence interval coverage, but the proposed SRI estimator with multiple human labels gives valid estimates with the smaller bias and root mean square error (RMSE) compared to the existing bias-correction methodology.

The remainder of this paper is organized as follows. Section \ref{sec::example} presents a motivating example drawn from the actual published research. Section \ref{sec::method} formalizes the proposed methodology with error-free human annotations. Section \ref{sec:proximal} then relaxes this assumption and deals with more realistic setting in which humans annotations contain non-differential measurement errors. 
Section \ref{sec::application} then applies the proposed method to the existing study introduced in the motivating example. Finally, Section \ref{sec::conclusion} concludes by summarizing the advantages and limitations of SRI and outlining directions for future research.\\

\noindent \textbf{Related Literature.}
There is a growing body of literature that aims to conduct statistical inference with a limited amount of ground-truth labels by leveraging both labeled data and machine learning predictions (e.g., \citealt{ angelopoulos2023prediction, egami_neulips_2023, egami2024using, zrnic_cross-prediction-powered_2024, carlson_unifying_2025}). A key strength of this literature is that most of them do not assume any structure of machine learning prediction errors. While this enables researchers to apply the bias-correction framework in any setting, the resulting standard error can be excessively large, especially when the machine learning quality is moderate. In addition, most of the proposed literature assumes that human annotations are error-free. The only exception in the literature is \cite{egami2024using}, which addressed this issue using multiple coders and a quasi-Bayesian approach to model uncertainty in annotations, in a manner akin to multiple imputation. However, that method does not offer formal guarantees of consistent estimation for the target parameter of interest.


This paper is also related to the literature on surrogate outcomes.
While recent studies develop methods using multiple surrogate outcomes (e.g., \citealt{price_estimation_2018, athey_surrogate_2019, wang_model-free_2020}), they primarily focus on low-dimensional, structured surrogate variables.
These methods do not address settings in which the surrogate variable is high-dimensional and unstructured, such as texts or images, nor do they study how dimensionality reduction can be carried out while preserving the surrogate structure needed for inference.


Another relevant literature is double negative control and proximal causal inference (e.g., \citealt{hu_instrumental_2008, kuroki_measurement_2014, miao_identifying_2018, tchetgen_introduction_2020}).
These methods provide identification strategies when key variables are not directly observed, and have been extended to a variety of settings, including synthetic controls \citep{qiu_doubly_2024}, hidden treatments \citep{zhou2024causal}, and networks \citep{egami_identification_2023}.
However, identification and semiparametric estimation strategies for the case of \emph{hidden outcomes} have not been formally derived in the existing literature, to the best of my knowledge.


Finally, this paper relates to representation learning for statistical inference with unstructured data.
A recurring theme in this literature is the need to reduce the dimensionality of texts or images while preserving the information that is relevant for downstream tasks (e.g., \citealt{veitch_adapting_2020, pryzant_causal_2021, wang_desiderata_2022, gui_causal_2023, imai_causal_2024, imai2025genai}).
These methods emphasize that unstructured data are often too high-dimensional to use directly, and therefore require learned representations to extract only the relevant information, such as confounding variable.\\


\noindent \textbf{Contribution of this paper.} This paper addresses the above gaps in the following three ways. First, I introduce Surrogate Representation Inference, a framework that leverages an assumption that can be guaranteed by design: human coders rely only on the text they are annotating and do not directly observe the structured variables of interest. Throughout the paper, I show that this assumption can be exploited to improve efficiency in downstream statistical inference.

Second, the SRI framework allows human annotations to contain non-differential measurement error, which is ruled out in almost all existing approaches. I provide identification results and a semiparametric estimation strategy with a formal consistency guarantee for the target parameter. I also empirically demonstrate that SRI substantially reduces bias when only error-prone human annotations are available. To facilitate practical implementation, I introduce a procedure for testing whether annotation errors are non-differential in Appendix \ref{sec::guidance}.

Finally, SRI extends surrogate-outcome ideas to settings with high-dimensional, unstructured surrogate variables by drawing on the representation-learning literature. Because texts and images are high-dimensional, I propose learning low-dimensional representations that continue to function as surrogates, thereby reducing dimensionality while preserving the surrogate structure required for efficient inference.

\section{Motivating Example: Framing of Immigrants in the U.S. Congress}\label{sec::example}
Statistical inference with unstructured data often faces two fundamental challenges that can significantly affect the validity of research. First, machine learning predictions are subject to bias, and their direct use in downstream statistical inference may yield biased estimates and incorrect standard errors. The magnitude of this bias tends to increase as the quality of machine learning predictions declines, but even with modern LLMs, achieving high predictive performance remains difficult for many concepts that are not present in the training data \citep{halterman2024codebook}. Second, human annotations cannot be assumed as the ground truth, since annotators can make mistakes, especially when coding large volumes of text.

These challenges are clearly illustrated in the analysis of elite discourse on immigration by \cite{card_computational_2022}. The authors examined how congressional rhetoric on immigration evolved over 140 years and how this evolution intersected with partisan divides in anti-immigration sentiment. Given the central role of elite rhetoric in shaping public opinion, the political rhetoric on immigration has been a central topic in political science and computational social science (e.g. \citealt{abrajano_media_2017, hajdinjak_migration_2020, simonsen_when_2025}). To provide a long-term historical perspective, \cite{card_computational_2022} compiled a dataset of more than 200,000 immigration-related congressional speeches (1880–2021) from the Congressional Record and related archives.

The text-based variable of interest in this application is the tone toward immigrants expressed in congressional speeches.
The outcome was coded into three categories: pro-immigration (valuing immigrants and favoring less-restrictive policies), neutral, and anti-immigration (emphasizing threats or restrictions). To begin their analysis, the authors constructed a set of human annotations. Specifically, they randomly assigned 3,643 immigration-related speech segments to at least two independent coders. Inter-coder reliability was assessed using Krippendorff’s $\alpha$ across three historical periods. Because tone classification is inherently complex and non-binary, agreement levels were modest, with an average $\alpha$ of approximately 0.5 (see Table \ref{accuracy}, first row).
To aggregate these annotations, the authors employed a Bayesian hierarchical model that accounts for annotator bias and estimates the probability distribution of the underlying “true” labels \citep{luo_detecting_2020}, and then the label with the highest posterior probability is assigned as the inferred “true” label.

Using the inferred labels, the authors trained a supervised machine learning classifier to categorize each speech segment as pro-immigration, neutral, or anti-immigration. The classifier was built on a pre-trained language model (RoBERTa; \citealt{liu_roberta_2019}). To incorporate uncertainty in the inferred labels, they employed a weighted training scheme in which class probabilities served as observation weights. The resulting model achieved a classification accuracy of about 65\% across all periods (see Table \ref{accuracy}, second row). In this application, classification accuracy refers to overall agreement between machine-learning predictions and aggregated human annotations. Finally, the predicted labels were regressed on the party affiliation of the speaker (Democrat, Republican, or other), revealing that Democrats consistently framed immigrants more positively than Republicans.

\begin{table}[t]
    \centering
    \begin{tabular}{cccc}
         \toprule
         & Early (1873-1934) & Mid (1935-1956) & Modern (1957-2020) \\ 
         \midrule
         Inter-Coder Agreements & 0.43 & 0.49 & 0.51 \\
         Classification Accuracy & 0.63 & 0.69 & 0.67 \\
         \bottomrule
    \end{tabular}
    \caption{Machine‑learning classification accuracy and inter‑coder agreement for the tone‑classification task in \cite{card_computational_2022}. Inter‑coder agreement is evaluated on 3,643 human‑annotated items using Krippendorff’s $\alpha$, which ranges from 0 to 1, where 0 indicates chance‑level agreement. The human annotations are then aggregated with a Bayesian item‑response model that infers a probability distribution over labels for each item. The classification accuracy of the RoBERTa models is assessed against these aggregated human annotations.}
    \label{accuracy}
\end{table}

This example demonstrates how both challenges manifest in practice. Although the authors conducted careful robustness checks, relying solely on machine-learning predictions with 65\% classification accuracy may yield biased results and incorrect standard errors. Since only a small portion of the corpus is annotated, using annotations alone would discard most of the data. Recent studies, including \cite{angelopoulos2023prediction}, \cite{egami_neulips_2023, egami2024using}, and \cite{carlson_unifying_2025}, address this issue by combining human annotations with machine-learning predictions. However, as \cite{angelopoulos2023prediction} and \cite{egami2024using} acknowledge, these methods often produce large standard errors, particularly when classification accuracy is moderate. 
In this application, for example, while the direct use of machine learning for difference-in-means produce a standard error of 0.003, whereas the bias-correction approach with RoBERTa predictions yields a much larger standard error of 0.020.
This challenge motivates the SRI framework introduced in Section \ref{sec::method}.

Similarly, the modest inter-coder agreement illustrates why treating human annotations as perfect ground truth is problematic. When multiple annotations are available, some papers have proposed inferring true labels using parametric models (e.g. \citealt{dawid_maximum_1979, raykar_learning_2010, luo_detecting_2020}) or majority voting (e.g., \citealt{whitehill_whose_2009}). However, these methods typically lack theoretical guarantees for downstream analysis or rely on strong but untestable assumptions. I address this issue in Section \ref{sec:proximal} by extending the estimator in Section \ref{sec::method} to account for non-differential measurement error. The assumptions underlying this correction are evaluated via a formal testing procedure described in Appendix~\ref{sec::guidance}.

\section{The Proposed Methodology with Perfect Annotations} \label{sec::method}
This section outlines the proposed SRI framework. For the sake of simplicity, I begin with the case that human annotations contain no measurement errors, and relax this assumption in the next section. Throughout the paper, I focus on the case in which texts serve as outcomes, motivated by the examples in Section \ref{sec::example}. However, the proposed SRI methodology is broadly applicable to other settings where unstructured data play a role in statistical inference. I proceed by defining the quantities of interest, specifying the necessary assumptions, establishing identification, and discussing estimation strategies.

\subsection{Assumptions, quantity of interest, and nonparametric identification}
Consider a simple random sample of $n$ observations drawn from a population of interest, indexed by $i = 1, 2, \dots, n$, where for each unit $i$, the observed data consist of $(T_i, \bY_i, \bZ_i)$. 
The binary predictor of interest is denoted by $T_i \in \{0,1\}$. 
The set of control variables and texts are also observed for all units, represented by $\bZ_i \in \cZ$ and $\bY_i \in \cY$, respectively, where $\cZ$ and $\cY$ are the support of each variable. Suppose that the analysis focuses on a specific feature of texts (e.g., the tone of the speeches) as outcomes, for which the human annotations are available only for a subset of the data. 
Let $L_i \in \mathcal{L} \subseteq \mathbb{R}$ denote the unobserved true concept of interest and let $\tilde L_i \in \mathcal{L}$ represent the human-annotated label, both supported on $\cL$.
Human annotated labels $\tilde L_i$ are observed only for units where $S_i = 1$, with $S_i \in \{0, 1\}$ as the sampling indicator. 

The goal is to identify the average value of the text feature $L_i$ for each level of the predictor $T_i$, after adjusting for the set of control variables $\bZ_i$. Formally, the target parameter is the covariate-adjusted outcome mean under $T_i = t \in \{0,1\}$, which is denoted as
\begin{align*}
    \Psi_t := \E\bigl[\E[L_i \mid T_i = t, \bZ_i]\bigr].
\end{align*}
I say that the target parameter is identified when it is uniquely determined by the joint distribution of the observed data.

To identify this parameter of interest, the text feature of interest $L_i$ is assumed to be a deterministic function of the text $\bY_i$. This assumption rules out cases in which the same text corresponds to multiple true annotations. It is consistent with prior work in the literature on text-based inference (e.g. \citealt{egami_how_2022, imai_causal_2024}) and is formalized as follows.

\begin{assumption}[True Coding Rule] \label{coding} There exists a deterministic function $g_L: \cY \to \cL$ that maps outcome texts to a set of concepts $L_i$, i.e., 
$L_i  =  \ g_{L}(\bY_i).$
\end{assumption}

I then assume that human annotations $\tilde L_i$ do not contain classification errors, so that they always capture the true concept of interest $L_i$. This is formalized as follows.

\begin{assumption}[Perfect Annotations]\label{perfect} The human annotations $\tilde L_i$ always captures the concept of interest $L_i$ so that
$\tilde L_i = L_i$
for all $i = 1, \cdots, n$.
\end{assumption}
This assumption is often made implicitly in the literature (e.g., \citealt{fong_machine_2021, angelopoulos2023prediction, egami_neulips_2023, egami2024using}). Note that this assumption is very restrictive as it does not allow any single mistake by human coders. I relax this assumption in Section \ref{sec:proximal}.

Figure \ref{diagram_assumption} graphically illustrates the setup in this section. In the graph, an arrow with double red line represents a deterministic relationship, while an arrow with a single black line indicates a possibly stochastic relationship. The set of assumptions indicates that the human coding $\tilde L_i$ is a deterministic function of texts, which yields the conditional independence $T_i \ \indep \ \tilde L_i \mid \bY_i, \ \bZ_i$. In practice, this assumption is obtained by design as well, since researchers only need to make sure that human coders are blind to structured variables. Under this setup, coders may still infer the information of structured variables (e.g., the party affiliation of a speaker), but such inference must rely solely on the text they observe.

\begin{figure}[ht]
    \centering
\begin{tikzpicture}[>=Stealth,thick,node distance=2.5cm]

  \node[draw,circle] (T) {$T$};
  \node[draw,circle,right=of T] (Lhat) {$\bm{Y}$};
  \node[draw,circle,right=of Lhat] (L) {$\tilde L$};
  \node[draw, circle, below=1.25cm of $(T)!0.5!(Lhat)$](C) {$\bm{Z}$};

  \draw[->] (T) -- (Lhat);
  \draw[->, double, color = red] (Lhat) -- (L);
  \draw[->] (C) -- (T);
  \draw[->] (C) -- (Lhat);
\end{tikzpicture}
\caption{The diagram illustrating the assumptions for the proposed method. $T$ denotes the predictors of interest, $\bm{Y}$ denotes the texts (outcomes), $\bZ$ denotes the set of all control variables, 
and $\tilde L$ denotes the human annotated labels. An arrow with red double lines represents a deterministic relation while an arrow with a single line indicates a possibly stochastic relationship. It is assumed that when coders annotate $\tilde L$, they only look at texts $\bm{Y}$ and do not look at the predictors of interest $T$ directly. Under this setup, texts $\bm{Y}$ can be used as a surrogate variable that fully mediates $T$'s effect on $\tilde L$.
}
\label{diagram_assumption}
\end{figure}
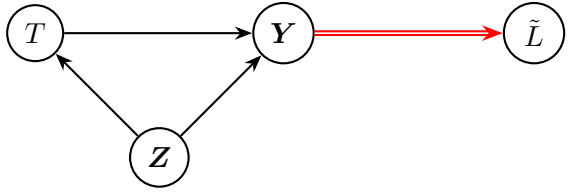

Next, human annotations are assumed to be applied to a random subset of texts, which is also a standard assumption in the literature (e.g., \citealt{fong_machine_2021, egami_neulips_2023, egami2024using}). This means that researchers should avoid the arbitrary selection of texts for annotations, such as coding only short or obvious ones.

\begin{assumption}[Human Annotation Sampling] \label{labeling} The human coders annotate the part of texts that are randomly sampled from the entire corpus such that
\begin{align}
    \{T_i, \bY_i, \bZ_i\} \; \indep \; S_i
\end{align}
with $\P(S_i = 1) > 0$.
\end{assumption}

Together, these assumptions provide the foundation for the identification strategy. The identification strategy of the SRI framework uses the fact that texts $\bY_i$ satisfy the conditional independence $T_i \ \indep \ \tilde L_i \mid \bY_i, \bZ_i$. Since the full texts $\bY_i$ are inherently high-dimensional and challenging to incorporate directly into statistical inference, I propose the use of a low-dimensional representation that fully mediates the relationship between the predictors of interest $T_i$ and human annotations $\tilde L_i$. I refer to this as \emph{surrogate representation}, a term that reflects its role in satisfying the statistical surrogacy condition \citep{prentice1989surrogate, athey_surrogate_2019}. This surrogate representation reduces the dimensionality of $\bY_i$ and $\bZ_i$ while retaining the essential information needed to make a statistical inference. Formally, surrogate representation is defined as follows:

\begin{definition}[Surrogate Representation of Texts]\label{def_surrep} I define surrogate representation as a low-dimensional representation $\bm{W}_i = \bm{f}(\bY_i, \bZ_i)$ with some deterministic function $\bm{f}$ that fully explains the relationship between $\{L_i, T_i\}$ and text and covariates $\{\bm{Y}_i, \bZ_i\}$ such that
\begin{align}
\{\bY_i, \bZ_i\} \; \indep \; \{L_i, T_i \}    \mid \bm{W}_i \label{indep_surrogate}
\end{align}
\end{definition}

Under Assumptions \ref{coding} and \ref{perfect}, Equation~\eqref{indep_surrogate} implies that the surrogate representation $\bW_i = \boldf(\bY_i, \bZ_i)$ fully mediates the relationship between the predictor $T_i$ and the human annotations $\tilde L_i = g_L(\bY_i)$, yielding $T_i ,\indep\ \tilde L_i \mid \bW_i$. The existence of such surrogate representation is trivial since the conditional probability $\eta(\bY_i, \bZ_i) =  \P(T_i, \tilde L_i\mid \bY_i, \bZ_i)$ is a natural, low-dimensional sufficient statistic that satisfies Equation \eqref{indep_surrogate}. Furthermore, it can be shown that any function of $\eta(\bY_i, \bZ_i)$ that retains this information is a valid surrogate representation, implying that surrogate representations can exist for arbitrary dimensions.

\begin{lemma}[Existence of Surrogate Representation]\label{existence}
Let $\boldf(\bY_i, \bZ_i)$ be a deterministic function of texts $\bY_i$ and covariates $\bZ_i$. Then, $\boldf(\bY_i, \bZ_i)$ is a surrogate representation satisfying equation \eqref{indep_surrogate} if and only if $\boldf(\bY_i, \bZ_i)$ is finer than $\eta(\bY_i, \bZ_i)=  \P(T_i, \tilde L_i\mid \bY_i, \bZ_i)$ in the sense that $\eta(\bY_i, \bZ_i) = b(\boldf(\bY_i, \bZ_i))$ for some function $b$.
\end{lemma}
The proof is in Appendix \ref{proof_existence}. This result is analogous to the property of balancing scores established by \cite{rosenbaum_central_1983}.


Under this setup, the target parameter can be nonparametrically identified using the surrogate representation $\bW_i$.

\begin{proposition}[Identification under Perfect Annotations]\label{identification}
Under Assumptions \ref{coding} to \ref{labeling}, the target parameter under $T_i = t$ can be identified as
\begin{align}
    \Psi_t := \E\biggl[\E[L_i \mid T_i = t, \bZ_i]\biggr] = \int_{\mathcal{\mathcal{Z}}}\int_{\mathcal{Y}} \E[\tilde{L}_{i} \mid \bm{W}_i,   S_i = 1]   dF(\bm{Y}_i \mid \bZ_i,  T_i = t) dF(\bZ_i) \label{identified}
\end{align}
where $\bm{W} = \boldf(\bY_i, \bZ_i)$ is the latent representation that satisfies Equation \eqref{indep_surrogate}. Importantly, the latent representation $\bW_i$ needs not to be unique and any latent representation with the same independence constraint leads to the same identification formula. 
\end{proposition}
The proof is in Appendix \ref{proof_iden}. When researchers additionally assume the conditional ignorability based on the set of covariates $\bZ_i$, positivity, and consistency, then the target parameter $\Psi_t$ can be interpreted as a causal quantity (an expectation of the potential outcome under $T_i = t$).

\subsection{Estimation and inference}
Given the identification result, I next consider estimation and inference. As derived in the previous section, SRI uses the surrogate representation $\bW_i= \boldf(\bY_i, \bZ_i)$ that satisfies Equation \eqref{indep_surrogate}. Lemma \ref{existence} suggests that any representation capable of reconstructing the conditional joint probabilities of $T_i$ and $\tilde L_i$ constitutes a valid surrogate representation. Moreover, since Equation \eqref{indep_surrogate} implies $T_i \ \indep \ \tilde L_i \mid \bW_i = \boldf(\bY_i, \bZ_i)$ under Assumptions \ref{coding} and \ref{perfect},
\begin{align*}
    \P(T_i, \tilde L_i \mid \bY_i, \bZ_i) = 
    \P(T_i, \tilde L_i \mid \boldf(\bY_i, \bZ_i)) = 
    \P(T_i \mid \boldf(\bY_i, \bZ_i)) \ \P(\tilde L_i \mid \boldf(\bY_i, \bZ_i)),
\end{align*}
where I use Equation \eqref{indep_surrogate} for the first equality. Therefore, a natural estimation strategy is to learn a representation that jointly predicts both human annotations and the predictor of interest. 

To achieve this, I propose estimating $\bW_i$ using a neural network architecture, following a similar approach to DragonNet \citep{shi_adapting_2019}. Specifically, the loss function is
\begin{align}
    \arg\min_{\xi, \lambda, \zeta} \frac{1}{n} \sum_{i = 1}^n \biggl\{
    S_i \cdot
    \underbrace{\biggl(
    \overbrace{\mu(\bm{f}(\bY_i, \bZ_i; \xi); \lambda)}^{\text{Outcome Model}}  - \tilde L_{i} \biggr)^2}_{\text{Outcome Loss}} \ + \ \alpha \cdot \underbrace{\mathrm{CrossEntropy}\biggl(T_i, \overbrace{\rho_1(\bm{f}(\bY_i,  \bZ_i; \xi); \zeta)}^{\text{Surrogacy Score}} \biggr)}_{\text{Prediction Loss for Predictor $T_i$}}
    \biggr\}, \label{loss_main}
\end{align}
where $\boldf(\cdot ; \xi)$ is the surrogate representation parametrized by $\xi$, $\mu(\cdot, \lambda)$ is the outcome model for the human annotation $\tilde L_i$ given $\bm{f}(\bY_i,\bZ_i; \xi)$ parametrized by $\lambda$, $\rho_t(\cdot; \zeta)$ is the model predicting the value of predictor of interest $T_i = t$ given $\bm{f}(\bY_i, \bZ_i; \xi)$ parametrized by $\zeta$, and $\alpha$ is the hyperparameter regulating the influence of both loss components. While the model $\rho_t$ appears analogous to propensity score, since $\bY_i$ is influenced by $T_i$, the predicted probability from $\rho_t$ is referred to as the surrogacy score in the literature \citep{athey_surrogate_2019}. The architecture of this neural network is illustrated in Figure \ref{dragonnet}.

\begin{figure}[ht]
    \centering
\begin{tikzpicture}[
  node distance = 1cm and 2.0cm,
  box/.style = {draw, rectangle, minimum height=1cm, minimum width=0.5cm, align=center},
]

\node[box, fill=lightgray] (z1) {$\bm{Y}_i$};
\node[box, fill=lightgray, above=0.5 of z1] (z2) {$\bZ_i$};

\node[box, right=2.5cm of z1, yshift=0.75cm,
      label={[align=center]below:{\small 
      \shortstack{Surrogate\\[-2pt] \small Representation $\bW_i$}
      }}] 
      (z3) {$\boldf(\bm{Y}_i, \bm{Z}_i;\xi)$};

\node[box, right=2.0 of z3, yshift=-1.0cm,
      label={[align=center]below:{\shortstack{Surrogacy Score\\[-2pt] 
        \scriptsize$\mathbb P(T_i=t\mid\bW_i)$}}}] 
      (g) {$\rho_t(\bW_i;\zeta)$};

\node[box, right=2.0 of z3, yshift=1.0cm,
      label={[align=center]above:{\shortstack{Outcome Model\\[-2pt] 
        \scriptsize$\mathbb E[\tilde L_i\mid\bW_i]$}}}] 
      (t11) {$\mu(\bW_i;\lambda)$};

\node[box, fill=lightgray, right=1.0 of t11] (yhat) {$\tilde L_i$};
\node[box, fill=lightgray, right=1.0 of g]   (that) {$T_i$};

\draw[->, line width=1] (z1) -- (z3);
\draw[->, line width=1] (z2) -- (z3);

\draw[->, line width=1] (z3) -- (t11);
\draw[->, line width=1] (z3) -- (g);

\draw[->, line width=1] (t11) -- (yhat);
\draw[->, line width=1] (g)   -- (that);

\end{tikzpicture}

    \caption{The architecture of neural network for the nuisance function estimation. Given the inputs $\bY_i$ (texts) and $\bZ_i$ (covariates),  the network maps them onto the shared internal representation $\bW_i = \boldf(\bY_i, \bZ_i; \xi)$ parametrized by $\xi$. This representation $\bW_i$ is used to simultaneously predict the human annotation $\tilde L_i$ (outcome model $\E[\tilde L_i \mid \bW_i]$) and the predictor of interest $T_i$ (surrogacy score $\P(T_i = t \mid \bW_i)$).}
    \label{dragonnet}
\end{figure}

To characterize the behavior of the estimator at the limit, I first derive its efficient influence function. The efficient influence function provides both a framework for semiparametric estimation and a theoretical foundation to characterize asymptotic behavior with optimal variance properties \citep{hampel_influence_1974, bickel1993efficient, vaart_asymptotic_1998, tsiatis_semiparametric_2006}.

\begin{theorem}[Efficient Influence Function] \label{eif} The efficient influence function for the target parameter under $T_i = t$ is given by
\begin{equation}
    \begin{aligned}
    &\psi_t(T_i, \tilde L_i, S_i, \bY_i, \bZ_i; \bm{f}, \rho_t, \pi_t, \mu, \Bar{m}_t, \Psi_t) \\
    &= \frac{\mathbbm{1}\{S_i = 1\}}{\P(S_i = 1)} \cdot  \frac{\rho_t(\boldf( \bY_i,  \bZ_i))}{\pi_t(\bZ_i)} \biggl( \tilde L_i - \mu( \boldf(\bY_i, \bZ_i))\biggr) \\
    &\qquad\qquad + 
    \frac{\mathbbm{1}\{T_i = t\} }{\pi_t(\bZ_i)} \biggl( \mu( \boldf(\bY_i, \bZ_i)) - \Bar{m}_t(\bZ_i) \biggr) + \Bar{m}_t(\bZ_i) - \Psi_t\label{efficientinfluencefunction}
    \end{aligned}
\end{equation}
where $\mu( \boldf(\bY_i, \bZ_i) ) = \E[\tilde L_i \mid  \bW_i]$ is the outcome model given the low-dimensional representation $\bW_i = \bm{f}(\bY_i, \bZ_i)$ and the set of control variables $\bZ_i$, $\rho_t(\bW_i) = \P(T_i = t \mid \bW_i)$ is the surrogacy score model given the low-dimensional representation $\bW_i = \bm{f}(\bY_i, \bZ_i)$, $\pi_t(\bZ_i) = \P(T_i = t \mid \bZ_i)$ is the propensity score model given the set of control variables $\bZ_i$, and $\Bar{m}_t(\bZ_i) = \int \mu( \boldf(\bY_i, \bZ_i)) dF(\bY_i \mid \bZ_i, T_i = t)$.
\end{theorem}
The proof is given in Appendix \ref{proof_eif}.

Once the efficient influence function is obtained, the asymptotically valid estimator can be constructed. Here, I propose to use a $K$-fold cross-fitting procedure \citep{van_der_laan_targeted_2010, chernozhukov_doubledebiased_2018}. Suppose that $n$ is divisible by $K$. The entire estimation procedure with $K$-fold cross-fitting is summarized as follows.
\begin{enumerate}
    \item Randomly partition the data into $K$ folds of equal size. Let $I(i) = k$ be the indicator of the observation $i$ belonging to the fold $k$.
    \item For each fold $k \in \{1, \cdots, K\}$, use $I(i) \neq k$ as training data:
    \begin{enumerate}
        \item Estimate the propensity score model $\pi_t$ using the whole sample with $I(i) \neq k$.
        \item Split the training data with $S_i = 0$ into two halves, say $I_{-1s}^{(-k)}$ and $I_{-2s}^{(-k)}$. Let the remaining training data (with $S_i = 1$) be $I_s^{(k)}$.
        \item Solve the optimization problem in Equation~\eqref{loss_main} using the data $I_s^{(k)}$ and $I_{-1s}^{(-k)}$ and obtain $\hat\mu^{(-k)}$ and $\hat\rho_t^{(-k)}$.
        \item Regress the obtained $\hat{\mu}^{(-k)}$ on $T_i = t$ and $\bZ_i$ using the data $I_{-2s}^{(-k)}$ and obtain $\hat{\Bar{m}}^{(-k)}$.
    \end{enumerate}
    \item For the level of predictors of interest $t$, solve the estimating equation 
    \begin{align}
    \frac{1}{n} \sum_{k = 1}^K \sum_{I(i) = k}\psi_t(T_i, \tilde L_i, S_i, \bY_i, \bZ_i; \hat{\bm{f}}^{(-k)}, \hat\rho_t^{(-k)}, \hat\pi_t^{(-k)}, \hat\mu^{(-k)}, \hat{\Bar{m}}_t^{(-k)}, \hat\Psi_t)  = 0 \label{est_eq}
    \end{align}
    and obtain the estimate of the target parameter $\hat \Psi_{t}$.
\end{enumerate}

Finally, I derive the asymptotic normality of the proposed estimator. For this, I assume the following regularity conditions.

\begin{assumption}[Regularity Conditions]
\label{reg} Let $c_1$ be a positive constant and $\delta_n$ be a sequence of positive constants approaching zero as the sample size $n$ increases.
\begin{enumerate}[label=(\alph*)]
    \item (Primitive Condition) The labeled outcome is bounded such that $\E\bigl[|\tilde L_i|^2\bigr]^\frac{1}{2} \leq c_1$.
    \item (Positivity) There exists a fixed constant $c_2 \in (0, 1/2]$ such that for all $t \in \{0,1\}$
    \begin{align*}
        \P( c_2 \leq \pi_t(\bZ_i) \leq 1 - c_2) = 1, \quad
        \P( c_2 \leq \hat \pi_t(\bZ_i) \leq 1 - c_2) = 1.
    \end{align*}
    In addition, the estimated surrogacy score is also almost surely bounded so that 
    \begin{align*}
        \P(\hat\rho_t(\hat \boldf(\bY_i, \bZ_i)) \leq 1) = 1
    \end{align*}
    for all $t \in \{0,1\}$.
    \item (Nuisance Functions Estimation) Nuisance functions are consistently estimated so that
    \begin{align*}
        \E\biggl[
        \biggl|\hat{\rho}_t(\hat{\boldf}(\bY_i, \bZ_i)) - {\rho}_t({\boldf}(\bY_i, \bZ_i)) \biggr|^2
        \biggr]^{\frac{1}{2}} \leq c_1, &\quad  \E\biggl[
        \biggl|\hat{\rho}_t(\hat{\boldf}(\bY_i, \bZ_i)) - {\rho}_t({\boldf}(\bY_i, \bZ_i)) \biggr|^2
        \biggr]^{\frac{1}{2}} \leq \delta_n\\
        \E\biggl[
        \biggl|\hat{\mu}(\hat{\boldf}(\bY_i, \bZ_i) - \mu({\boldf}(\bY_i, \bZ_i))\biggr|^2
        \biggr]^{\frac{1}{2}} \leq c_1, &\quad \E\biggl[
        \biggl|\hat{\mu}(\hat{\boldf}(\bY_i, \bZ_i) - \mu({\boldf}(\bY_i), \bZ_i)\biggr|^2
        \biggr]^{\frac{1}{2}} \leq \delta_n\\
        \E\biggl[
        \biggl|\hat{\pi}_t(\bZ_i) - \pi_t(\bZ_i) \biggr|^2
        \biggr]^{\frac{1}{2}} \leq c_1, &\quad \E\biggl[
        \biggl|\hat{\pi}_t(\bZ_i) - \pi_t(\bZ_i) \biggr|^2
        \biggr]^{\frac{1}{2}} \leq \delta_n\\
        \E\biggl[
        \biggl|\hat{\bar{m}}_t(\bZ_i) - \bar{m}_t(\bZ_i) \biggr|^2
        \biggr]^{\frac{1}{2}} \leq c_1, &\quad \E\biggl[
        \biggl|\hat{\bar{m}}_t(\bZ_i) - \bar{m}_t(\bZ_i) \biggr|^2
        \biggr]^{\frac{1}{2}}\leq \delta_n
    \end{align*}
    for all $t \in \{0,1\}$.
    \item (Convergence Rate of Cross Product)
    \begin{align*}
        \E\biggl[
        \biggl|\hat{\rho}_t(\hat{\boldf}(\bY_i, \bZ_i)) - {\rho}_t({\boldf}(\bY_i, \bZ_i)) \biggr|^2
        \biggr]^{\frac{1}{2}} &\cdot 
        \E\biggl[
        \biggl|\hat{\mu}(\hat{\boldf}(\bY_i, \bZ_i) - \mu({\boldf}(\bY_i, \bZ_i))\biggr|^2
        \biggr]^{\frac{1}{2}}
        \leq n^{-1/2}\delta_n\\
        \E\biggl[
        \biggl|\hat{\pi}_t(\bZ_i) - \pi_t(\bZ_i) \biggr|^2
        \biggr]^{\frac{1}{2}} &\cdot 
        \E\biggl[
        \biggl|\hat{\bar{m}}_t(\bZ_i) - \bar{m}_t(\bZ_i) \biggr|^2
        \biggr]^{\frac{1}{2}}
        \leq n^{-1/2}\delta_n
    \end{align*}
    for all $t \in \{0,1\}$.
\end{enumerate}
\end{assumption}
These are the standard conditions in the literature on semiparametric inference \citep{newey1994large, van_der_laan_targeted_2010, chernozhukov_doubledebiased_2018}, but they become difficult to verify when high-dimensional unstructured data are used directly as input. A common strategy is to project texts onto low-dimensional representations via pretrained embeddings, such as Bidirectional Encoder Representations from Transformers (BERT; \citealt{devlin_bert_2019}) or ModernBERT \citep{warner_smarter_2024}. Although these embeddings capture much of the textual content, their statistical properties remain poorly understood, and estimation errors may arise in practice.
If users care about the estimation loss from the embedding, a more principled alternative is to use the internal representation of the generative models, such as LLMs \citep{imai_causal_2024}. When a model can faithfully reconstruct an input $\bY_i$, the internal representation that actually generates the data can be extracted without any estimation error. Thus, if users can extract the internal representations of the generative models, these error‐free latent vectors can then be fed into the standard feed‐forward network with ReLU activations with sufficient depth and width, for which there is a convergence rate guarantee \citep{bauer_deep_2019, schmidt-hieber_nonparametric_2020, farrell_deep_2021}.

Given the stated regularity condition, the asymptotic normality of the proposed estimator can be derived as follows.

\begin{theorem}[Asymptotic Normality] \label{asymp_normal}
Under Assumptions~\ref{coding}--\ref{reg}, the estimator for the covariate-adjusted outcome mean $\hat\Psi_t$ obtained by solving the estimating equations in Equation~\eqref{est_eq} for $t \in \{0,1\}$ satisfies asymptotic normality:
$$
\frac{\sqrt{n}(\hat{\Psi}_t - \Psi_t)}{\sigma} \xrightarrow[]{d} \mathcal{n}(0, 1)
$$
where $\sigma^2 = \E\bigl[\bigl(\psi_t(T_i, \tilde L_i, S_i, \bY_i, \bZ_i)\bigr)^2 \bigr].$    
\end{theorem}
The proof is provided in Appendix \ref{proof_asymp_normal}. 
One drawback of this approach is that it disregards the machine learning predictions for the text feature of interest, even though these are often relatively easy to obtain. Accordingly, Appendix \ref{use_ml} considers a setting where such predictions are available, and I propose an estimation strategy that leverages the full set of observations by incorporating the machine learning predictions.

\section{Extension: The Proposed Methodology with Noisy Human Annotations}\label{sec:proximal}
The methodology described in the previous section enables the estimation under the assumption that human annotations contain no errors (Assumption \ref{perfect}). Although this assumption is common in the literature, human coders may still make random mistakes even when the coding rule is clearly defined.
Given that some amount of perfectly labeled data is required to obtain the sufficient statistical power, this limitation can be particularly problematic for real-world applications.

In this section, I extend the SRI framework in the previous section to relax the assumption of the perfect human annotations using the recent literature of double negative control (e.g., \citealt{hu_instrumental_2008, kuroki_measurement_2014, miao_identifying_2018, tchetgen_introduction_2020, zhou2024causal}). Specifically, I propose a methodology that requires multiple, but potentially incorrect, human annotations and still performs the efficient statistical inference. As before, I state assumptions, identification, and estimation strategies. I also write about the practical guidance for applied researchers about how to diagnose the required assumptions that are not guaranteed by design.

\subsection{Assumptions and nonparametric identification}
Consider the same setting as in the previous section, except that the true label $L_i \in \cL$ is assumed to be discrete. For simplicity, let $\cL = \{0, \cdots, C\}$ with some constant $C<\infty$. Instead of assuming the perfect annotations, suppose that there are $J$ coders who independently label texts $\bm{Y}_i$. Let $\tilde L^{(j)}_i  \in \cL$ be a label coded by the coder $j \in \{1, \cdots, J\}$. As before, I use $S_i = 1$ if $\tilde L^{(j)}_i$ is observed for all coders. Importantly, I allow for the possibility that $\tilde L^{(j)}_i \neq L_i$ for any $j$ and $i$, meaning that coders can make mistakes. For simplicity, I assume $J = 2$ (i.e., two human coders) for the remainder of this section, but the method can be generalized to the case of multiple coders.

The goal is to identify the same target parameter from the previous section using only the human-annotated labels $\tilde L_i^{(j)}$ for $j = \{1,2\}$, which are observed only for the part of the corpus and can be incorrect (i.e., $\tilde L_i^{(j)} \neq L_i = g_L(\bY_i)$). As the annotations may contain mistakes, I assume the following for the sake of formalization.

\begin{assumption}[Human Annotation Sampling] \label{labeling2} Human coders annotate a subset of texts that are randomly sampled from the entire corpus such that
\begin{align}
    \tilde L_i^{(j)} \; \indep \; S_i
\end{align}
for all coders $j$ with $\P(S_i = 1) > 0$.
\end{assumption}

\begin{assumption}[Noisy Human Annotations] \label{annotation} Each coder $j$ annotates $\bY_i$ solely based on the information within $\bY_i$ so that
\begin{align}
    \tilde L_i^{(j)} \ \indep \ \bZ_i \ \mid \ \bY_i \quad \text{and} \quad 
    \tilde L_i^{(j)} \ \indep \ T_i \ \mid \ \bY_i, \ \bZ_i 
\end{align}
for all $j \in \{1, \cdots, J\}$.
\end{assumption}
These assumptions do not substantially increase the practical burden compared to the previous setting. In particular, they can be satisfied if (1) researchers randomly select the texts to be annotated, (2) human coders do not directly observe the predictors of interest, and (3) the structured variable of interest $T_i$ is not coded by annotators. Unlike the earlier case, the third condition is essential here because annotated variables may otherwise remain correlated even after conditioning on texts. For instance, if coders classify both the speaker’s party affiliation and the tone of a speech, the classification errors in these two variables can be correlated even after conditioning on the text itself.

I then impose restrictions on the human annotations and the unobserved true label. While the human annotations may differ from the true label for some observations, I assume that (1) the mistakes made by each coder is independent and (2) the only way for texts to influence the human coding is through the true label. This is formally stated as follows:

\begin{assumption}[Conditional Independence of Human Annotations]\label{proximal} The human-annotated labels $\{ \tilde L^{(j)}_i: j = 1, \cdots J \}$ and the true label $L_i$ satisfy the followings:
\begin{enumerate}[label=(\alph*)]
    \item (Independent Coding) For all pair of $j,j'$ with $j\neq j'$,
    \begin{align}
        \tilde L_i^{(j)} \ \indep \ \tilde L_i^{(j')} \mid L_i
    \end{align}
    \item (Exclusion Restriction) For all human coders $j$, 
    \begin{align}
        \bY_i \ \indep \ \tilde L_i^{(j)} \ \mid \ L_i 
    \end{align}
\end{enumerate}
\end{assumption}
Figure \ref{diagram_assumption2} illustrates the required independence assumptions graphically. These two independence constraints in Assumption \ref{proximal} fundamentally require that errors in observed human annotations be completely non‑differential (i.e., the measurement error does not depend on true annotations $L_i$ or the contents of texts $\bY_i$). This essentially requires that the errors arise from pure mistakes or inattention rather than systematic biases related to the content of the texts. When mistakes are non‑differential, the coding errors from each coder are independent, satisfying the independent‑coding assumption. The exclusion restriction is also plausible when only the concept of interest in the text influences human annotation, which is reasonable when mistakes are random. Note that Assumption \ref{proximal} cannot be relaxed by conditioning on covariates $\bZ_i$, as they provide no additional information under Assumption \ref{annotation}.

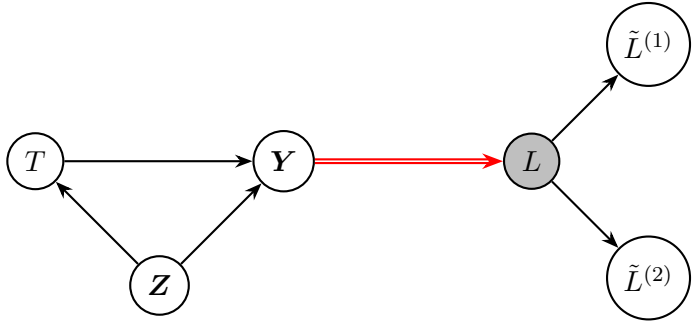
\begin{figure}[ht]
    \centering
\begin{tikzpicture}[>=Stealth, thick, node distance=2.5cm]
        \node[draw, circle] (T) {$T$};
        \node[draw, circle, right=of T] (Y) {$\bm{Y}$};
        \node[draw, circle, right=of Y, fill = lightgray] (Ltilde) {$L$};
        \node[draw, circle, above right=1.25of Ltilde] (L1) {$\tilde L^{(1)}$};
        \node[draw, circle, below right=1.25cm of Ltilde] (L2) {$\tilde L^{(2)}$};
        \node[draw, circle, below=1.25cm of $(T)!0.5!(Lhat)$](C) {$\bm{Z}$};
        
        \draw[->] (T) -- (Y);
        \draw[->, double, color = red] (Y) -- (Ltilde);
        \draw[->] (C) -- (T);
        \draw[->] (C) -- (Lhat);
        \draw[->] (Ltilde) -- (L1);
        \draw[->] (Ltilde) -- (L2);
    \end{tikzpicture}
\caption{The diagram illustrates the assumptions for the proposed method when human annotations may differ from the unobserved true label of interest. $T$ denotes the predictors of interest, $\bm{Y}$ denotes the texts (outcomes), $L$ represents the unobserved true concept of interest, and $\bZ$ denotes the set of control variables. $\tilde L^{(j)}$ indicates the human-annotated labels provided by coder $j$ ($j = 1,2$). An arrow with red double lines represents a deterministic relationship, while an arrow with a single line indicates a potentially stochastic relationship. White nodes correspond to observed variables, whereas gray nodes indicate unobserved variables.}
\label{diagram_assumption2}
\end{figure}

To better understand how Assumption \ref{proximal} may be violated, consider the example from Section \ref{sec::example} of classifying the tones of congressional speeches on immigration. Suppose researchers classify speeches as pro‑immigration as long as any part of the speech values immigrants or favors less restrictive policies. However, if coders interpret speeches as pro‑immigration only when they both value immigrants and favor less restrictive policies, then the text content directly influences the pattern of human coders’ annotations without being fully mediated by the researchers’ definition of the label concept, thus violating the exclusion restriction assumption.

In most cases, violations of both assumptions stem from vaguely specified coding rules. When coding guidelines cannot provide specific instructions for classifying each text, coders may attempt to fill these gaps using their own reasoning and procedures. However, these well‑intentioned efforts lead to violations of Assumption \ref{proximal}. Conversely, when coding rules clearly specify classification criteria for each text type so that texts can be classified almost deterministically, most errors should result from random mistakes  rather than systematic biases. Therefore, researchers must craft clear, unambiguous coding guidelines to maintain this assumption. I discuss procedures for testing this assumption and refining coding rules in Appendix \ref{sec::guidance}.

Lastly, I assume that human labels are accurate and humans agree on their annotations most of the time. This is formalized as follows.
\begin{assumption}[Accuracy of Human Annotations]\label{monotonicity} The human-annotated labels $\{ \tilde L^{(j)}_i: j = 1, \cdots J \}$ and the true label $L_i$ satisfy the followings:
\begin{enumerate}
    \item (Accurate Human Labeling) Human coders are more likely to correctly annotate the true label so that
    \begin{align}
        \mathbb{P}(\tilde L_i^{(j)} = l \mid L_i = l) &> 0.5 \label{unobservable} 
    \end{align}
    for all $l \in \cL$ and all coders $j$.
    \item (Human Labeling Agreements) Human coders generally agree on their labeling so that
    \begin{align}
    &\P(\tilde L_i^{(j)} = \tilde L_i^{(j')} = l, T_i = t\mid \bZ_i , S_i = 1) \geq \sum_{l\neq l'}\P(\tilde L^{(j)} = l', \tilde L^{(j')} = l, T_i = t\mid S_i = 1) \label{agree} 
\end{align}
for all $l \in \cL$, $t \in \{0,1\}$ and pair of coders $(j,j')$ with $j \neq j'$.
\end{enumerate}
\end{assumption}
Although these assumptions may appear restrictive, Equation~\eqref{agree} involves only the observed quantities and is thus directly testable. The only untestable component is Equation~\eqref{unobservable}. This assumption is satisfied as long as humans correctly label more than half of the texts. The assumption does not require perfection and humans can still make mistakes, but it requires that their labels be biased toward the truth, which is critical for correcting errors when combining annotations from multiple human coders. 

Under this setup, the distribution of the covariate-adjusted outcome mean can be nonparametrically identified by leveraging two or more human coding as proxies, which can then be used to identify the covariate-adjusted mean of the unobserved true label. 

\begin{proposition}[Identification without Perfect Human Annotations]\label{identification2} Under Assumptions \ref{coding}, \ref{labeling}, \ref{labeling2}--\ref{monotonicity}, I can uniquely identify the distribution of the covariate-adjusted outcome mean under $T_i = t$, i.e.,
\begin{align*}
    \Psi_{t,c} := \int \mathbb{P}(L_i = c \mid T_i = t, \bZ_i) dF(\bZ_i)
\end{align*}
from the observed data.
\end{proposition}
The proof is given in Appendix \ref{proof_identification2}. 

The essence of identification is as follows. Under Assumption \ref{proximal}, for all $l,m \in \cL$ and $\bz \in \cZ$,
\begin{align*}
&\mathbb{P}(\tilde L^{(1)}_i = l, \tilde L^{(2)}_i = m, T_i = t \mid \bZ_i = \bz, S_i = 1) \\
&= \sum_{c = 0}^C
\mathbb{P}(\tilde L^{(1)}_i = l \mid L_i = c)
\P(T_i = t \mid L_i = c, \bZ_i = \bz)
\P(L_i = c \mid \bZ_i = \bz)
\mathbb{P}(\tilde L^{(2)}_i = m \mid L_i = c),
\end{align*}
and
\begin{align*}
&\mathbb{P}(\tilde L^{(1)}_i = l, \tilde L^{(2)}_i = m \mid \bZ_i = \bz, S_i = 1) \\
&= \sum_{c = 0}^C
\mathbb{P}(\tilde L^{(1)}_i = l \mid L_i = c)
\P(L_i = c \mid \bZ_i = \bz)
\mathbb{P}(\tilde L^{(2)}_i = m \mid L_i = c).
\end{align*}
These two factorizations construct the system of equations. In this system, the left-hand sides of both equations are observable and the three terms on the right-hand sides
(i.e., $\mathbb{P}(\tilde L^{(1)}_i = l \mid L_i = c)$,
$\mathbb{P}(\tilde L^{(2)}_i = m \mid L_i = c)$, and
$\P(L_i = c \mid \bZ_i = \bz)$)
are shared between equations. Since the number of equations exceeds the number of unknown parameters, the system is usually identifiable only up to scaling and permutation. However, because the coder classification probabilities $\mathbb{P}(\tilde L^{(j)}_i = l \mid L_i = c)$ sum to one across all categories (by definition of probability) and attain their maximum at $\mathbb{P}(\tilde L^{(j)}_i = c \mid L_i = c)$ (Assumption \ref{monotonicity}), all unknown parameters are in fact identified. Once the coder-specific error rates are recovered for all coders, the target quantity can be obtained by another factorization
\begin{align*}
&\mathbb{P}(\tilde L^{(1)}_i = l, \tilde L^{(2)}_i = m \mid T_i = t, \bZ_i = \bz,  S_i = 1) \\
&= \sum_{c = 0}^C
\mathbb{P}(\tilde L^{(1)}_i = l \mid L_i = c)
\mathbb{P}(L_i = c \mid T_i = t, \bZ_i = \bz)
\mathbb{P}(\tilde L^{(2)}_i = m \mid L_i = c),
\end{align*}
where the only unknown is $\mathbb{P}(L_i = c \mid T_i = t,  \bZ_i = \bz)$, which is used to calculate the parameter of interest.

\subsection{Estimation and inference}
Given the identification result, I next turn to estimation and inference. The main challenge here is that the true label $L_i$ is not observed, so it cannot be used to learn the low-dimensional surrogate representation in the same way as in the previous section. To address this, I introduce the surrogate outcome
\begin{equation*}
    \begin{aligned}
        &M_c(\tilde L^{(1)}_i, \tilde L^{(2)}_i) \\
        &= \biggl( \frac{\mathbbm{1}\{\tilde L^{(1)}_i = c\} - \P(\tilde L^{(1)}_i = c \mid L_i \neq c) }{
    \P(\tilde L^{(1)}_i = c \mid L_i = c) - \P(\tilde L_i^{(1)} = c \mid L_i \neq c)
    }\biggl)
    \biggl( \frac{\mathbbm{1}\{\tilde L_i^{(2)} = c\} - \P(\tilde L^{(2)}_i = c \mid L_i \neq c) }{
    \P(\tilde L_i^{(2)} = c \mid L_i = c) - \P(\tilde L_i^{(2)} = c \mid L_i \neq c)
    }\biggl)
    \end{aligned}
\end{equation*}
for all levels of categories $c \in \cL$.  

Crucially, this surrogate outcome is identifiable and can be estimated from the data. Specifically, each coder $j$’s classification accuracy, $\P(\tilde L^{(j)}_i = c \mid L_i)$, is recovered using the system of equations in Proposition \ref{identification2}. Although the identification formula holds for any value of the predictor $t \in \{0,1\}$, precision improves when estimates are averaged across all levels. Once $\widehat{\P(\tilde L^{(j)}_i = c \mid L_i)}$ is obtained, these estimates can be substituted into the formula above to construct $\hat{M}_c(\tilde L^{(1)}_i, \tilde L^{(2)}_i)$.

This surrogate outcome serves as a substitute for the true label and $M_c(\tilde L^{(1)}_i, \tilde L^{(2)}_i) = \mathbbm{1}\{L_i = c\}$ if both coders classify texts accurately (i.e., $\P(\tilde L^{(1)}_i = c \mid L_i = c) = \P(\tilde L^{(2)}_i = c \mid L_i = c) = 1$). Formally, it can be proven that this surrogate outcome is an unbiased estimate of the true label given either the true label or the texts: i.e.,
\begin{align}
    \E[M_c(\tilde L^{(1)}_i, \tilde L^{(2)}_i) \mid L_i] = \E[M_c(\tilde L^{(1)}_i, \tilde L_i^{(2)}) \mid \bY_i] = \mathbbm{1}\{L_i = c\}. \label{unbiased_surroagte}
\end{align}
Furthermore, it can be shown that the outcome model predicting this surrogate outcome is mathematically equivalent to the one predicting the true, unobserved label $ L_i $. See Appendix \ref{surrogate_proof} for the theoretical properties of the surrogate outcome.

Based on this surrogate outcome, I can extend the estimation strategies in the previous section.
Since I can learn the outcome model for the true annotation $L_i$ using the surrogate outcome $M_c(\tilde L^{(1)}_i, \tilde L^{(2)}_i)$, once the surrogate outcome $\hat{M_c}(\tilde L^{(1)}_i, \tilde L_i^{(2)})$ is estimated, I need to solve the following optimization problem
\begin{equation}
    \begin{aligned}
        \arg&\min_{\xi, \tilde\lambda_c, \zeta} \frac{1}{n} \sum_{i = 1}^n \biggl\{
        S_i \cdot
    \sum_{c \in \cL} \biggl(
    \underbrace{
    \overbrace{\tilde\mu_c(\bm{f}(\bm{Y}_i, \bZ_i; \xi); \tilde\lambda_c)}^{\text{Outcome Model}}  - \hat{M_c}(\tilde L^{(1)}_i, \tilde L_i^{(2)}) }_{\text{Outcome Loss for Category } c} \biggr)^2 \\
    &\qquad\qquad\qquad\qquad\qquad\qquad\qquad+ \beta \cdot \underbrace{\text{CrossEntropy}\biggl(T_i, \overbrace{\rho_t(\bm{f}(\bm{Y}_i, \bZ_i; \xi); \zeta)}^{\text{Surrogacy Score}} \biggr)}_{\text{Prediction Loss for Predictor $T_i$}}
    \biggr\} \label{loss_proximal},
    \end{aligned}
\end{equation}
where \( \tilde\mu_c(\cdot; \tilde \lambda_c) \) is the outcome model for category \( L_i = c \) with the estimated surrogate outcome \( \hat{M_c}(\tilde L^{(1)}_i, \tilde L_i^{(2)}) \) given \( \bm{W}_i = \bm{f}(\bm{Y}_i, \bZ_i; \xi) \) parametrized by \( \tilde\lambda_c \), 
and $\beta$ are the hyperparameters regulating the influence of the two loss components. 
Importantly, the surrogate representation $\boldf(\bY_i, \bZ_i)$ learned under the loss function in Equation \eqref{loss_proximal} encodes the same conditional independence relationship as in Definition \ref{def_surrep}, because the outcome model based on the surrogate outcome $\tilde\mu$ is mathematically equivalent, in the population limit, to the outcome model based on the true outcome.

To characterize the behavior of the estimator at the limit, as in the previous section, I first derive its influence function.
\begin{theorem}[Influence Function] \label{theorem_eif2} The influence function for the distribution of the covariate-adjusted outcome under $T_i = t$ is given by
\begin{equation}
\begin{aligned}
    \tilde{\psi}_{t,c}(T_i, \tilde L^{(1)}_i, \tilde L^{(2)}_i, S_i, \bY_i, \bZ_i&; \boldf, \rho_t, \tilde \mu, \pi_t,  M_c, \Bar{m}_{tc}, \Psi_{t,c})\\
    &= \frac{\mathbbm{1}\{S_i = 1\}}{\P(S_i = 1)} \cdot  \frac{\rho_t(\boldf( \bY_i, \bZ_i))}{\pi_t(\bZ_i)} \biggl(M_c(\tilde L^{(1)}_i, \tilde L_i^{(2)}) - \tilde\mu_c( \boldf(\bY_i, \bZ_i))\biggr)\\
    &\quad\quad\quad + 
    \frac{\mathbbm{1}\{T_i = t\} }{\pi_t(\bZ_i)} \biggl( \tilde\mu_c( \boldf(\bY_i, \bZ_i)) - \Bar{m}_{t,c}(\bZ_i) \biggr) + \Bar{m}_{t,c}(\bZ_i) - \Psi_{t,c}
\end{aligned}
\label{eif2}
\end{equation}
where $\tilde \mu_c(\boldf(\bY_i, \bZ_i)) = \E[M_c(\tilde L_i^{(1)}, \tilde L_i^{(2)}) \mid \boldf(\bY_i, \bZ_i), S_i = 1]$, and $\Bar{m}_{t,c}(\bZ_i) = \int \tilde \mu_c(\boldf(\bY_i, \bZ_i)) dF(\bY_i \mid T_i = t, \bZ_i)$.
\end{theorem}
The proof is provided in Appendix \ref{proof_eif2}. 

Since the estimation of the target parameter requires both the estimation of the surrogate outcome $M_c(\tilde L^{(1)}_i, \tilde L^{(2)}_i)$ and the optimization of the loss function in Equation~\eqref{loss_proximal}, I need to modify the standard cross-fitting procedure accordingly. Suppose that $n$ is divisible by $K$. The entire estimation procedure with $K$-fold cross-fitting is summarized as follows.
\begin{enumerate}
    \item Randomly partition the data into $K$ folds of equal size. Let $I(i) = k$ be the indicator of the observation $i$ belonging to the fold $k$.
    \item For each fold $k \in \{1, \cdots, K\}$, use $I(i) \neq k$ as training data:
    \begin{enumerate}
        \item Estimate the propensity score model $\hat\pi_t$ using all the sample with $I(i) \neq k$.
        \item Split the training data with $S_i = 1$ into two halves, say $I^{(-k)}_{1s}$ and  $I^{(-k)}_{2s}$. Also, split the remaining training data (with $S_i = 0$) into two halves, say $I_{-1s}^{(-k)}$ and $I_{-2s}^{(-k)}$.
        \item Use $I^{(-k)}_{1s}$ to obtain the estimated $\P(\tilde L_i^{(j)} = l \mid L_i = m)$ for all categories $(l,m)$ and all coders $j = 1,2$. By inserting these estimates,  obtain the estimated surrogate outcome $\hat M_c(\{\tilde L_i^{(1)}, \tilde L_i^{(2)}\}_{i \in I^{(-k)}_{2s}})$ on $I^{(-k)}_{2s}$.
        \item Solve the optimization problem in Equation~\eqref{loss_proximal} using the data $I^{(-k)}_{2s}$ and $I_{-1s}^{(-k)}$ and obtain the estimated outcome models $\hat{\tilde\mu}_c^{(-k)}$ for all category levels and the surrogacy score $\hat\rho_t^{(-k)}$.
        \item Regress the obtained $\hat{\tilde \mu}_c^{(-k)}$ onto $T_i = t$ and $\bZ_i$ using the data $I_{-2s}^{(-k)}$ and obtain $\hat{\Bar{m}}_{t,c}^{(-k)}$.
    \end{enumerate}
    \item For each category level $c \in \cL$ and at the level of predictors of interest $t$, solve the estimating equation 
    \begin{align}
        \frac{1}{n} \sum_{k = 1}^K \sum_{I(i) = k}\tilde{\psi}_{t,c}(T_i, \tilde L^{(1)}_i, \tilde L^{(2)}_i, S_i, \bY_i, \bZ_i; \hat\boldf^{(-k)}, \rho_t^{-(k)}, \hat{\tilde \mu}^{(-k)}, \hat\pi_t^{(-k)},  \hat{M}_c^{(-k)}, \hat{\Bar{m}}^{(-k)}_{t,c}, \hat\Psi_{t,c}) = 0 \label{est_eq_proximal}
    \end{align}
    and obtain the estimate of the target parameter $\hat \Psi_{t,c}$.
\end{enumerate}
As the procedure described above is nested, I practically recommend using large values of $K$ to have a sufficient amount of training data for obtaining the coder classification estimates and the training of the neural network.

Under this setup, I formally establish the asymptotic normality of the proposed estimator. As in Section \ref{sec::method}, I assume standard regularity conditions, but the required convergence rates can be satisfied by the proposed neural network architecture. The formal statement of these regularity conditions is provided in Assumption~\ref{reg2} of Appendix~\ref{proof_asympnormal2}.

\begin{theorem}[Asymptotic Normality]\label{asympnormal2}
Under Assumptions \ref{coding},\ref{labeling},  \ref{labeling2}--\ref{reg2}, the estimator for the covariate-adjusted outcome mean $\hat{\Psi}_t = \sum_{c \in \cL} c \ \hat \Psi_{t,c}$ obtained by solving the estimating equations in Equation~\eqref{est_eq_proximal} satisfies asymptotic normality:
$$
\frac{\sqrt{n}(\hat{\Psi}_t - \Psi_t)}{\tilde\sigma} \xrightarrow[]{d} \mathcal{N}(0, 1)
$$
where $\tilde\sigma^2 = \E\bigl[\bigl(\sum_{c \in \cL} c \  \tilde{\psi}_{t,c}(T_i, \tilde L^{(1)}_i, \tilde L^{(2)}_i, S_i, \bY_i, \bZ_i)\bigr)^2 \bigr].$
\end{theorem}
The proof is in Appendix \ref{proof_asympnormal2}.

\section{Empirical Application}\label{sec::application}
I now apply the proposed SRI methodology to the existing study introduced in Section \ref{sec::example}. Following the original analysis, I use the party affiliation of the speaker as the predictor of interest. The outcome of interest is the framing of the immigrants in the U.S. Congressional speeches, and the original analysis used predictions from a pre-trained language model (RoBERTa; see Table \ref{accuracy} for classification accuracy). I use all observations with machine-learning predictions ($n = 248{,}616$), of which $3{,}643$ are annotated.

For the sake of efficiency comparison, I implement the proposed methodology under the assumption that human annotations are perfect (Section \ref{sec::method}). To do so, I first obtain vector representations of the texts using BERT embedding \citep{devlin_bert_2019} and feed them into the proposed neural-network architecture. I perform 2-fold cross-fitting with the architecture illustrated in Figure \ref{dragonnet}. Specifically, the network has three hidden layers with dimensions 100, 100, and 50 leading to the shared layer $\boldf(\bR_i)$. The shared layer is then connected to prediction heads for outcomes and the predictors of interest, each consisting of two layers with dimensions 50 and 1 (i.e., ultimately predicting a scalar). I use ReLU activations and solve the optimization problem in Equation \ref{loss_main} with $\alpha = 1$, employing the Adam optimizer (learning rate $=2 \times 10^{-5}$), a maximum of 200 epochs, and a batch size of 256 \citep{kingma_adam_2017}. I set aside 20 \% of the data for validation and implement early stopping with a patience of five epochs (i.e., if the validation loss fails to decrease for five consecutive epochs, training stops). Once the outcome model and surrogacy score are estimated, I use them to calculate the covariate-adjusted outcome mean for each party affiliation via the estimating equation in Equation \eqref{est_eq} and calculate the difference between Democrats and Republicans.

I compare the proposed SRI estimator with two alternatives. The first is a naïve estimator that directly uses RoBERTa’s predictions to compute the difference in means. Because RoBERTa’s predictions contain classification errors, this estimator is expected to be biased. The second is an existing bias-correction approach that combines machine-learning predictions with human annotations, often referred to as Prediction-Powered Inference (PPI; \citealt{angelopoulos2023prediction}) or, with only minor differences, Design-based Supervised Learning (DSL; \citealt{egami_neulips_2023, egami2024using}). With no control variables, these estimators for the outcome mean under $T_i = t$ is defined as\footnote{Specifically, I follow the DSL approach \citep{egami_neulips_2023, egami2024using} for $\hat\Psi_{t, \mathrm{Existing}}$ because of imbalances between the labeled and unlabeled data. The only difference between DSL and PPI is that PPI uses the average over the unbalanced subset of data for the first term, whereas DSL uses the average over the entire dataset. Formally, the PPI estimator are defined as
\begin{align*}
    \hat \Psi_{t, \mathrm{PPI}} = \frac{1}{N_t} \sum_{i = 1}^n \mathbbm{1}\{T_i = t, S_i = 0\} \hat{L}_i - \frac{1}{n_t} \sum_{i = 1}^n \mathbbm{1}\{T_i = t, S_i = 1\}  \biggl(\hat{L}_i - \tilde L_i\biggr) \quad \text{for all} \ t,
\end{align*}
where $N_t$ is the number of unlabeled observations (i.e., $N_t = \sum_{i = 1}^n \mathbbm{1}\{T_i = t, S_i = 0\}$). As long as texts are randomly assigned to annotators, these two estimators should, on average, yield the same estimates at the population level. In this application, however, although random assignment was implemented, the average tone differs between labeled and unlabeled data, making PPI somewhat sensitive. For this reason, I use the DSL estimator here, since it relies on the overall average for the first term.
} 
\begin{align*}
    \hat\Psi_{t, \mathrm{Existing}} := \frac{1}{n} \sum_{i = 1}^n \mathbbm{1}\{T_i = t\} \hat{L}_i - \frac{1}{n_t} \sum_{i = 1}^n \mathbbm{1}\{T_i = t, S_i = 1\}  \biggl(\hat{L}_i - \tilde L_i\biggr) \quad \text{for all} \ t,
\end{align*}
where $n_t$ is the number of labeled units in group $T_i = t$ (i.e., $n_t = \sum_{i = 1}^n \mathbbm{1}\{T_i = t, S_i = 1\}$), and $\hat L_i$ denotes the predictions from RoBERTa. Once the outcome mean is estimated for each value of the predictor of interest, the difference between $\hat\Psi_{1, \mathrm{Existing}}$ (Democrats) and $\hat\Psi_{0, \mathrm{Existing}}$ (Republicans) is calculated. Unlike SRI, these existing estimators do not impose independence constraints between the machine-learning predictions $\hat{L}_i$ and the human annotations $\tilde L_i$. Although such constraints are not guaranteed by design for machine learning predictions $\hat{L}_i$, their absence is expected to reduce efficiency, especially when $\hat{L}_i$ and $\tilde L_i$ diverge. See Appendix \ref{difference} for further theoretical discussion.

\begin{figure}[h]
    \centering
    \includegraphics[width=1.0\linewidth]{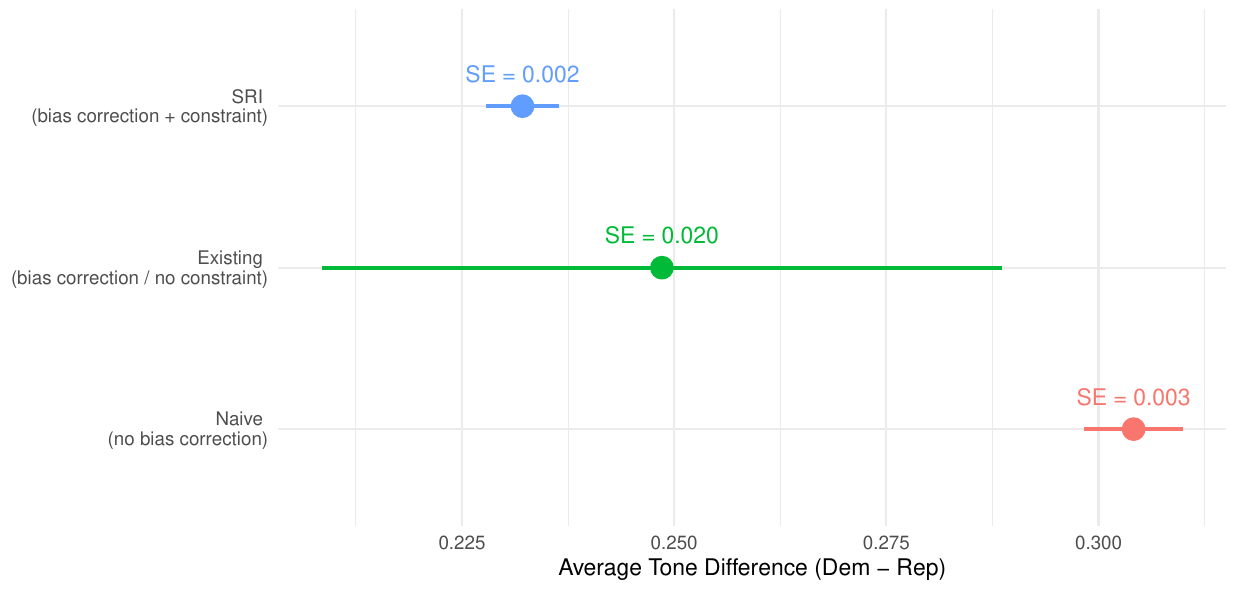}
    \caption{The estimated partisan difference in tone towards immigrants based on the data of \cite{card_computational_2022} for (1) the SRI estimator (proposed), (2) the existing bias-correction approach using both human annotations and machine learning predictions; \citealt{angelopoulos2023prediction, egami_neulips_2023, egami2024using}),
    and (3) the na\"ive estimator (directly using the machine learning predictions).}
    \label{emp_result}
\end{figure}

Figure \ref{emp_result} presents the results. The na\"ive estimator based on machine-learning predictions shows a larger partisan difference in the framing of immigrants than both the proposed SRI estimator and the existing bias-correction approaches. While both SRI and the existing bias-correction approach are consistent, the results highlight a substantial efficiency gain: by exploiting independence constraints, the SRI estimator significantly reduces the standard error compared with the existing bias-correction approaches.

\section{Conclusion}\label{sec::conclusion}
In this paper, I introduce \emph{Surrogate Representation Inference} (SRI), a framework that addresses the challenges of incorporating high-dimensional unstructured data into statistical inference with a limited number of human annotations that may contain measurement errors. The SRI framework rests on a key assumption that can be guaranteed by design: human coders rely only on unstructured data for annotations and do not directly look at the structured variables of interest. Under this assumption, unstructured data fully mediate the structured variables and human annotations, satisfying the statistical surrogacy condition. As unstructured data are high-dimensional, I propose learning their low-dimensional representations such that the surrogacy conditions continue to hold. The paper examines the setting in which a text-based variable functions as an outcome and develops identification conditions, a neural network architecture for representation learning, and semiparametric efficient estimation strategies based on the learned representations. When multiple annotations are available, SRI also corrects for non-differential measurement error in human annotations. The empirical application demonstrates that SRI significantly reduces the standard errors compared to the existing methodology. Systematic simulation studies in Appendix \ref{sec::simulation} also confirm that when the quality of the machine learning proxy is moderate, the proposed SRI methodology can reduce the standard error approximately by half. These findings highlight the potential of SRI to enhance the precision and reliability of statistical inference with unstructured data, paving the way for future applications in diverse text or image-based research settings.

While this work has primarily focused on settings where text-based variable serves as an outcome, the proposed SRI procedure can be applied to diverse settings and modalities, such as those where unstructured data serves as control or predictors. 
Another exciting direction for future work is the discovery of text features influenced by predictor of interest. Although this paper currently assumes researchers know which features to measure, unstructured data offers the potential to identify these features without prior specification, an advantage particularly useful for open-ended questions. Looking ahead, future research should develop a unified framework that supports data-driven discovery of outcome concepts and perform the valid and efficient statistical inference with minimal amount of labeling. Finally, it would be valuable to explore how LLM‑generated annotations can be incorporated into the SRI methodology. Although the current framework primarily addresses human annotations with non‑differential error, it remains valid when annotations are produced by LLMs, provided that the required assumptions hold. Some assumptions, such as independent coding, may be difficult to satisfy with LLMs because of overlap in their training data. Future work should therefore investigate ways to relax these assumptions or devise prompting strategies that ensure LLM annotations meet the necessary conditions.

\newpage 
\bibliography{references, ref2}

\newpage 
\appendix
\setcounter{equation}{0}
\setcounter{figure}{0}
\setcounter{table}{0}
\setcounter{section}{0}
\renewcommand {\theequation} {S\arabic{equation}}
\renewcommand {\thefigure} {S\arabic{figure}}
\renewcommand {\thetable} {S\arabic{table}}
\renewcommand {\thesection} {S\arabic{section}}

\begin{center}
  {\LARGE \bf Appendix} 
\end{center}

\section{Theoretical Explanation of Efficiency Gains}\label{difference}
I now further investigate why the SRI framework can yield efficiency gains compared to existing methods in general. To compare Surrogate Representation Inference (SRI) with the existing bias-correction approach, I first define each semiparametric model and the corresponding tangent space. Let \(\mathcal{M}_{\text{SRI}}\) be a semiparametric model of Surrogate Representation Inference (SRI) and \(\mathcal{M}_{\text{PPI}}\) be that of the existing bias-correction methodology. Then, the observed data for each semiparametric model, denoted as \(O_{\text{SRI}}\) and \(O_{\text{PPI}}\) respectively, are written as
\begin{align*}
    \text{Surrogate Representation Inference:} \quad &O_{\text{SRI}} = \{T_i, \bY_i, \bZ_i, S_i, G_{S_i}(\tilde L_i)\}_{i = 1}^n  \\
    \text{Existing Methods:} \quad &O_{\text{PPI}} = \{T_i, \bm{h}(\mathbf{Y}_i), \mathbf{Z}_i, S_i, G_{S_i}(\tilde L_i)\}_{i = 1}^n,
\end{align*}
where \(T_i\) is the predictor of interest, \(\bm{Z}_i\) is the set of control variables, \(\tilde L_i\) is the human annotation, \(S_i\) is the indicator variable for whether the human annotation \(\tilde L_i\) is observed, \(G_{S_i}(\cdot)\) is the coarsening operator that outputs the input when \(S_i = 1\) and outputs the empty set when \(S_i = 0\), and \(\bm{h}(\bm{Y}_i)\) is the machine learning-based proxy. Importantly, texts $\bY_i$ satisfy the independence $\tilde{L}_i \ \indep \ T_i \mid \bY_i, \bZ_i$, but this is not necessarily true for \(\bm{h}(\bm{Y}_i)\). As a result, without any constraints on machine learning predictions, $\mathcal{M}_{\mathrm{SRI}} \subseteq \mathcal{M}_{\mathrm{PPI}}$
since any factorization under $\mathcal{M}_{\mathrm{SRI}}$ belongs to $\mathcal{M}_{\mathrm{PPI}}$.
As a result, the tangent space of Surrogate Representation Inference, denoted
as $\cT_{\mathrm{SRI}}$, is a linear subspace of that of the existing method $\cT_{\mathrm{PPI}}$.

In semiparametric statistics, the model's efficiency is often evaluated in terms of the \emph{semiparametric efficiency bound}, which is the analogue of the Cramér-Rao lower bound for the semiparametric problem and no semiparametric estimators can have an asymptotic variance smaller than this bound (e.g., \citealt{bickel1993efficient, newey_semiparametric_1990, vaart_asymptotic_1998, tsiatis_semiparametric_2006}). To see the efficiency gains, I compare the efficiency bounds for the two different methods.  Mathematically, they are given by
\begin{align*}
    \text{Surrogate Representation Inference:} \quad & \mathbb{E}\biggl[
    \biggl(S(O_i) - \Pi( S(O_i) \mid \mathcal{T}_{\text{SRI}}) \biggr)
    \biggl(S(O_i) - \Pi( S(O_i) \mid \mathcal{T}_{\text{SRI}}) \biggr)^\top \biggr]^{-1} \\
     \text{Existing Methods:} \quad & \mathbb{E}\biggl[
    \biggl(S(O_i) - \Pi( S(O_i) \mid \mathcal{T}_{\text{PPI}}) \biggr)
    \biggl(S(O_i) - \Pi( S(O_i) \mid \mathcal{T}_{\text{PPI}}) \biggr)^\top \biggr]^{-1}
\end{align*}
where \(O_i = \{T_i, \bm{Y}_i, \bm{Z}_i, S_i\}\) is the set of variables required to consider the score, \(S(O_i)\) is the score function, and \(\Pi(\cdot \mid \cdot)\) is the projection operator (note that I can remove \(\tilde L_i\) from \(O_i\) when considering the score since it is a deterministic function of \(\bm{Y}_i\)). As \(\mathcal{T}_{\text{SRI}} \subseteq \mathcal{T}_{\text{PPI}}\),
\begin{align*}
    \underbrace{S(O_i) - \Pi( S(O_i) \mid \mathcal{T}_{\text{SRI}})}_{\text{Efficient Score for SRI}} = \underbrace{S(O_i) - \Pi( S(O_i) \mid \mathcal{T}_{\text{PPI}})}_{\text{Efficient Score for PPI}} + \underbrace{\Pi( S(O_i) \mid \mathcal{T}_{\text{PPI}})  - \Pi( S(O_i) \mid \mathcal{T}_{\text{SRI}})}_{:= \Delta \ (\text{residual term)}}
\end{align*}
where the residual term \(\Delta\) is orthogonal to the efficient score for PPI. As a result, since $\Delta \in \cT_{\text{PPI}}$ and the efficient score belongs to the orthogonal complement of the tangent space,
\begin{align*}
    &\mathbb{E}\biggl[
    \biggl(S(O_i) - \Pi( S(O_i) \mid \mathcal{T}_{\text{SRI}}) \biggr)
    \biggl(S(O_i) - \Pi( S(O_i) \mid \mathcal{T}_{\text{SRI}}) \biggr)^\top \biggr]\\
    &=\mathbb{E}\biggl[
    \biggl(S(O_i) - \Pi( S(O_i) \mid \mathcal{T}_{\text{PPI}}) + \Delta \biggr)
    \biggl(S(O_i) - \Pi( S(O_i) \mid \mathcal{T}_{\text{PPI}}) + \Delta \biggr)^\top \biggr]\\
    &= \mathbb{E}\biggl[
    \biggl(S(O_i) - \Pi( S(O_i) \mid \mathcal{T}_{\text{PPI}}) \biggr)
    \biggl(S(O_i) - \Pi( S(O_i) \mid \mathcal{T}_{\text{PPI}}) \biggr)^\top \biggr] + \mathbb{E}[\Delta\Delta^\top]\\
    &\succeq \mathbb{E}\biggl[
    \biggl(S(O_i) - \Pi( S(O_i) \mid \mathcal{T}_{\text{PPI}}) \biggr)
    \biggl(S(O_i) - \Pi( S(O_i) \mid \mathcal{T}_{\text{PPI}}) \biggr)^\top \biggr]
\end{align*}
where the third equality is by the orthogonality of the residual term to the efficient score that lies in the orthocomplement of the tangent space. Therefore, as most of the existing literature on the bias-correction estimator works on the semiparametric model \(\mathcal{M}_{\text{PPI}}\), regardless of the form of the estimators, without imposing any additional restrictions, such as parametric models, SRI always has a lower semiparametric efficiency bound and the proposed estimator in Section \ref{sec::method} attains this bound.

From this theoretical analysis, it is expected that when the machine learning predictions already achieve high classification accuracy, there should not be much efficiency gain in reality. This is because the efficiency gains come from the fact that the machine learning predictions do not necessarily satisfy the independence constraint, but if they are almost accurate, there is no reason to expect that the residual term \(\Delta\) will become large. I confirm this conjecture in the application section (Section \ref{sec::application}) and the simulation section (Section \ref{sec::simulation}).

\newpage
\section{How to test assumptions in Section \ref{sec:proximal} that are not guaranteed by design} \label{sec::guidance}
While Section \ref{sec:proximal} has developed a methodology to account for human annotations with non-differential measurement errors, the proposed SRI approach relies on several additional assumptions. Three specific assumptions warrant particular attention: (1) the independent coding (Assumption \ref{proximal} (a)), which states that the coding errors are independent across coders, (2) the exclusion restriction (Assumption \ref{proximal} (b)), which posits that texts influence the human annotations solely through the true concept, and (3) accurate human labels (Assumption \ref{monotonicity} (a)), which asserts that the accuracy of the human annotated labels is more than 0.5. None of these assumptions can be directly tested or is guaranteed by the study design, thus complicating the application of the proposed methodology.

\begin{figure}[ht]
    \centering
    \begin{tikzpicture}[
    node distance=1cm and 2cm,
    box/.style={
        draw,
        rounded corners,
        minimum width=5cm,
        minimum height=1.5cm,
        align=center,
        fill=blue!10
    },
    arrow/.style={
        -Stealth,
        thick
    }
]
\node[box] (step1) {\textbf{Step 1}:\\Create the rigid coding rule and the small\\ number of the true labels by yourself};
\node[box, below=of step1] (step2) { \textbf{Step 2}:\\Ask human coders to label texts for\\the part that you coded in Step 1};
\node[box, below=of step2] (step3) { \textbf{Step 3}:\\Check Assumptions\\ (Conditional Independence / Accuracy)};
\node[box, below=of step3] (step4) { \textbf{Step 4}:\\Ask human coders to code more texts\\ and perform statistical inference};
\node[box, right=of step4] (update) {Update the coding rule\\and go back to Step 1};

\draw[arrow] (step1) -- (step2);
\draw[arrow] (step2) -- (step3);
\draw[arrow] (step3) -- node[midway, fill=white, align=center] {If all assumptions\\are satisfied} (step4);
\draw[arrow] (step3) -| node[near start, above, fill=white, align=center] {If assumptions\\are violated} (update);

\draw[arrow] (update) -| ($(update.east) + (1,0)$) |- (step1);

\end{tikzpicture}
    \caption{Workflow for the applied researchers to create the coding rule and the labeled data with multiple human coders. To test the conditional independence assumption (Assumption \ref{proximal}) in Step 3, the conditional permutation test is used. See the main text for the specific procedure.}
    \label{workflow}
\end{figure}

To mitigate this limitation, I propose a procedure that uses a limited number of true, gold-standard labels $L_i$ and permutation tests to evaluate these assumptions. All procedures are illustrated in Figure \ref{workflow}. The proposed procedure comprises four steps. First, researchers develop their own coding rule and apply it to code a small number of texts themselves. I assume that researchers can generate a tiny set of gold-standard labels that perfectly capture $L_i$ for these texts. Given that researchers should have a clear understanding of the classification task, it is reasonable to posit that they can produce error-free labels for a small number of observations, such as 100 texts, although scaling this process without permitting errors is challenging. Second, researchers ask multiple human coders to annotate the texts.

Upon obtaining the researcher-annotated gold-standard labels $L_i$ and the human-annotated labels $\tilde{L}^{(j)}_i$ for the tiny subset of observations, researchers can then verify all the problematic assumptions. For Assumption \ref{monotonicity} (a) (human label accuracy), researchers can verify each conditional probability and check if the assumptions are valid or not. For Assumption \ref{proximal} (independent coding and exclusion restriction), I propose to use permutation tests to diagnose the independence constraints. Permutation tests are nonparametric statistical methods that are particularly advantageous and commonly utilized when dealing with a small number of observations, as conventional statistical tests often exhibit insufficient power in such scenarios (e.g., \citealt{good_permutation_1994, ojala_permutation_2009, pesarin2010permutation, berrett_conditional_2019}). By default, the null hypothesis of the permutation test is independence, and thus I use test of equivalence so that the independence is directly tested (e.g. \citealt{wellek2002testing, hartman_equivalence_2018}). The proposed procedure is based on \cite{arboretti_testing_2018}, which formalized the permutation test for test of equivalence.

Formally, to directly evaluate the conditional independence specified in Assumption \ref{proximal}, the objective of the test is to evaluate the following null and alternative hypotheses:
\begin{align*}
    &H_0: \tilde L^{(1)}_i \  \not\!\perp\!\!\!\perp  \tilde L^{(2)}_i \ \mid \ L_i
    \quad \text{or} \quad
    \bY_i \  \not\!\perp\!\!\!\perp  \ \{\tilde L^{(1)}_i, \tilde L^{(2)}_i\} \ \mid \  L_i\\
    &H_1: \tilde L^{(1)}_i \  \indep \ \  \tilde L^{(2)}_i \ \mid \ L_i
    \quad \text{and} \quad
    \bY_i \  \indep  \ \{\tilde L^{(1)}_i, \tilde L^{(2)}_i\} \ \mid \  L_i  .
\end{align*}
To use the permutation test for conditional independence, since $L_i$ is categorical, I can simply permute the observations within each stratum defined by $L_i$ to nonparametrically assess conditional independence \citep{berrett_conditional_2019}. For exclusion restriction, due to the high dimensionality of $\bY_i$, it is necessary to employ a scalar measure of statistical independence, such as the kernel-based measure proposed by \cite{gretton_kernel_2012} or the correlational distance proposed by \cite{szekely_measuring_2007}, to construct the test statistic. The entire testing procedure is as follows:
\begin{enumerate}
    \item \textbf{Defining Equivalence Range}: Define the equivalence margin $\delta \in \R^{+}$, which specifies the range within which the difference is considered inconsequential. 
    \item \textbf{Stratification}: Stratify the labeled data based on the values of $L_i$, creating strata for each possible value. Let $I_l = \{i: L_i = l\}$ denote the set of indices for each stratum.
    \item For each stratum $L_i = l$ with $l \in \cL$, 
    \begin{enumerate}
        \item \textbf{Computing Test Statistic}: Compute the test statistic by
    \begin{align*}
        T_{l1} = \max\{0,  d(\tilde L_i^{(1)}, \tilde L_i^{(2)}) - \delta\} \quad \text{and} \quad 
        T_{l2} = \max\{0, d(\bY_i, \{\tilde L_i^{(1)}, \tilde L_i^{(2)}\}) - \delta\}
    \end{align*}
    where $d(\cdot, \cdot)$ is any user-chosen measure for statistical independence.
        \item \textbf{Permutation}: For $b = 1, \cdots, B$ (where $B$ is the maximum number of permutations),
        \begin{itemize}
            \item Randomly permute the pair $\{\tilde L_i^{(1)}, \tilde L_i^{(2)}\}$ within each strata
            \item Compute test statistic with the permuted sample, denoted as $T_{l1}^{(b)}$ and $T_{l2}^{(b)}$
        \end{itemize}
    \end{enumerate}
    \item \textbf{Aggregation}: Aggregate the test statistic over all strata by
    \begin{align*}
        T = \max_{j \in \{1,2\}}\max_{l \in \cL} T_{jl}, \quad \text{and} \quad T^{(b)} =  \max_{j \in \{1,2\}}\max_{l \in \cL}T^{(b)}_{jl} \quad \text{for each} \ \ b \in \{1, \cdots, B\}
    \end{align*}
    \item \textbf{Calculating Monte Carlo p-value}: Calculate the Monte Carlo p-value by
    \begin{align*}
        p = \frac{1 + \sum_{b = 1}^B \mathbbm{1}\{T^{(b)} \leq T\}}{B + 1}
    \end{align*}
\end{enumerate}
Regarding the implementation of this test procedure, the following two points need to be considered. First, the standard representation of texts $\bY_i$ is typically still too high‑dimensional, so the permutation test can be underpowered unless researchers have a large number of true, gold-standard labels. In practice, researchers may therefore benefit from applying dimension‑reduction techniques that preserve most of the variation in the data and then applying the proposed procedure. Second, the testing procedure is sensitive to the choice of equivalence margin: the larger the margin, the more likely researchers are to reject the null hypothesis. If researchers cannot choose the equivalence margin ex ante on the basis of substantive knowledge, they may instead focus on the equivalence interval, the smallest margin that would lead to rejection of the null hypothesis at the pre‑specified significance level, which is insensitive to the researcher's choice of equivalence range \citep{hartman_equivalence_2018, hartman_equivalence_2021}.

\newpage 
\section{Extension: Use of Machine Learning Predictions}\label{use_ml}
In many applications, researchers already have access to machine-learning predictions of the text feature of interest $L_i$. With the advent of large language models, such predictions are increasingly inexpensive to obtain at scale. However, the baseline SRI procedure described in the main text relies exclusively on human annotations and therefore does not directly exploit these auxiliary predictions.

This section extends the estimation procedure to incorporate all available information, including machine-learning predictions. The proposed extension does not alter the estimand or the identification strategy, but is designed to more effectively exploit the full set of observations using these predictions.

Suppose that, in addition to the setup described in the main text, I observe machine-learning predictions of $L_i$ for all units, denoted as $\hat L_i$, for the entire observations. I propose estimating the surrogate representations $\bW_i = \boldf(\bY_i)$ using the same neural network architecture as before but with the following modified loss function:
\begin{equation}
    \begin{aligned}
        \arg\min_{\xi, \lambda, \zeta} \frac{1}{n} &\sum_{i = 1}^n \biggl\{
    \underbrace{\biggl(
    \overbrace{\mu(\bm{f}(\bY_i; \xi), \bZ_i; \lambda)}^{\text{Outcome Model}}  -  \bigl( \hat{L}_{i} + \frac{S_i}{\P(S_i = 1)}(\tilde L_i - \hat L_i)  \bigr) \biggr)^2}_{\text{Outcome Loss}} \ \\
    &\qquad\qquad\qquad\qquad+ \ \alpha \cdot \underbrace{\mathrm{CrossEntropy}\biggl(T_i, \overbrace{\rho_1(\bm{f}(\bY_i; \xi), \bZ_i; \zeta)}^{\text{Surrogacy Score}} \biggr)}_{\text{Prediction Loss for Predictor $T_i$}}
    \biggr\} \label{loss_ml}
    \end{aligned}
\end{equation}
where $\alpha$ is the hyperparameter. 
The key departure from the original loss in Equation~\eqref{loss_main} is the replacement of the human annotation $\tilde L_i$ with the augmented outcome $\hat{L}_{i} + \frac{S_i}{\P(S_i = 1)}(\tilde L_i - \hat L_i)$, which has been widely used in the missing-data  literature \citep{robins_estimation_1994, angelopoulos2023prediction, egami_neulips_2023}.
Importantly, the surrogate representation $\boldf(\bY_i, \bZ_i)$ is learned under the same conditional independence constraints as in Section \ref{sec::method}.

The principal advantage of the proposed loss function is improved efficiency. Under the original loss in Equation~\eqref{loss_main}, the outcome model is trained only on the subset of units with human annotations (i.e., $S_i = 1$), which may be insufficient to achieve the desired convergence rate in Assumption \ref{reg} when the annotated sample is small. In contrast, the augmented outcome is available for all units, allowing the modified loss in Equation~\eqref{loss_ml} to leverage the full sample. Consequently, as long as the machine-learning predictions $\hat L_i$ are informative about the human annotations $\tilde L_i$, the proposed procedure yields efficiency gains relative to the baseline approach.

Because the outcome loss is modified, the influence function must be updated accordingly. The proposed influence function is
\begin{equation}
    \begin{aligned}
         &\nu_t(T_i, \tilde L_i, \hat L_i,  S_i, \bY_i, \bZ_i; \boldf, \mu, \pi_t, \rho_t, \Bar{m}_t, \Psi_t)\\
         &= \psi_t(T_i, \hat L_i,  S_i, \bY_i, \bZ_i) - \frac{S_i}{\P(S_i = 1)}\biggl(\psi_t(T_i, \hat L_i,  S_i, \bY_i, \bZ_i)
    - \psi_t(T_i, \tilde L_i,  S_i, \bY_i, \bZ_i)\biggr),
    \end{aligned}
\end{equation}
where
\begin{align*}
    \psi_t(T_i, \tilde L_i, \bY_i, \bZ_i) &=   \frac{\rho_t(\boldf( \bY_i,  \bZ_i))}{\pi_t(\bZ_i)} \biggl( \tilde L_i - \mu( \boldf(\bY_i, \bZ_i))\biggr) + 
    \frac{\mathbbm{1}\{T_i = t\} }{\pi_t(\bZ_i)} \biggl( \mu( \boldf(\bY_i, \bZ_i)) - \Bar{m}_t(\bZ_i) \biggr) + \Bar{m}_t(\bZ_i) - \Psi_t\\
    \psi_t(T_i, \hat L_i, \bY_i, \bZ_i) &=   \frac{\rho_t(\boldf( \bY_i,  \bZ_i))}{\pi_t(\bZ_i)} \biggl( \hat L_i - \mu( \boldf(\bY_i, \bZ_i))\biggr) + 
    \frac{\mathbbm{1}\{T_i = t\} }{\pi_t(\bZ_i)} \biggl( \mu( \boldf(\bY_i, \bZ_i)) - \Bar{m}_t(\bZ_i) \biggr) + \Bar{m}_t(\bZ_i) - \Psi_t.
\end{align*}

Estimation proceeds exactly as in Section~3, except that the outcome loss is constructed using the augmented outcome in Equation~\eqref{loss_ml}. In particular, the surrogate representation $\boldf(\bY_i, \bZ_i)$, the outcome model $\mu(\cdot)$, and the surrogacy score $\rho_t(\cdot)$ are jointly estimated using the same neural network architecture and cross-fitting scheme in Section \ref{sec::method}. The target parameter $\Psi_t$ is then obtained by solving the estimating equation
\begin{align}
    \frac{1}{n} \sum_{k = 1}^K \sum_{I(i) = k}\nu_t(T_i, \tilde L_i, \hat L_i,  S_i, \bY_i, \bZ_i; \hat{\bm{f}}^{(-k)}, \hat\rho_t^{(-k)}, \hat\pi_t^{(-k)}, \hat\mu^{(-k)}, \hat{\Bar{m}}_t^{(-k)}, \hat\Psi_t)  = 0. \label{est_eq3}
\end{align}

Finally, the asymptotic normality of the proposed estimator is established under the following regularity conditions.

\begin{assumption}[Regularity conditions for augmented-outcome SRI]\label{reg_ml} Let $\eta = (\boldf, \mu, \pi_t, \rho_t, \bar{m}_t)$ be a collection of nuisance functions, $\eta^*$ be the probability limit of the estimated nuisance functions. Also, to differentiate the parameter $\Psi_t$ from the probability limit of the parameter of interest, let $\Psi_t^*$ be the probability limit of the estimated target parameter $\hat\Psi_t$. 

Assume that $\norm{\nu_t(T_i, \tilde L_i, \hat L_i,  S_i, \bY_i, \bZ_i; \hat\eta, \Psi_t^*) - \nu_t(T_i, \tilde L_i, \hat L_i,  S_i, \bY_i, \bZ_i; \eta, \Psi_t^*)}_2^2 = o_p(1)$ in each fold. Also, assume that $\E[\partial \nu_t(T_i, \tilde L_i, \hat L_i,  S_i, \bY_i, \bZ_i; \eta, \Psi_t^*)/\partial \Psi_t]$ exists and there is a neighborhood $\mathcal{N}$ of $\Psi_t$ such that
\begin{enumerate}
\item[(a)] for each fold $k$, $\norm{\hat{ \eta}^{(k)} - \eta^*}_2 = o_p(1)$
\item[(b)] for all $\norm{\hat{ \eta}^{(k)} - \eta^*}_2$ small enough, $\nu_t(T_i, \tilde L_i, \hat L_i,  S_i, \bY_i, \bZ_i; \eta, \Psi_t)$ is differentiable in $\Psi_t$ on $\mathcal{N}$ with probability approaching one, and there are $C > 0$ and $\delta(T_i, \tilde L_i, \hat L_i,  S_i, \bY_i, \bZ_i; \eta)$ such that, for $\Psi_t \in \mathcal{N}$ and $\norm{\hat{ \eta}^{(k)} - \eta^*}_2$ small enough,
\begin{align*}
    &\biggl|\biggl| \frac{\partial \nu_t(T_i, \tilde L_i, \hat L_i,  S_i, \bY_i, \bZ_i; \eta, \Psi_t) }{\partial \Psi_t} - \frac{\partial \nu_t(T_i, \tilde L_i, \hat L_i,  S_i, \bY_i, \bZ_i; \eta, \Psi_t^*) }{\partial \Psi_t}\biggr|\biggr|_2 \\
    &\qquad\qquad\qquad\qquad\qquad\qquad\qquad\qquad\qquad\qquad\leq \delta(T_i, \tilde L_i, \hat L_i,  S_i, \bY_i, \bZ_i, \eta) \norm{\Psi_t - \Psi_t^*}^{\frac{1}{C}}
\end{align*}
\item[(c)] For each $k$, 
\begin{align*}
    \E[\partial  \nu_t(T_i, \tilde L_i, \hat L_i,  S_i, \bY_i, \bZ_i; \hat\eta^{(k)}, \Psi_t) / \partial \Psi_t - \nu_t(T_i, \tilde L_i, \hat L_i,  S_i, \bY_i, \bZ_i; \eta^*, \Psi_t^*) / \partial \Psi_t] = o_p(1).
\end{align*}
\end{enumerate}
\end{assumption}

\begin{theorem}[Asymptotic Normality with Augmented Outcome]\label{asymp_ml}
Under Assumptions \ref{coding} to \ref{annotation} and \ref{reg_ml}, the estimator for the covariate-adjusted mean $\hat\Psi_t$ obtained by solving the estimating equation in Equation \ref{est_eq3} for $t \in \{0,1\}$ satisfies the asymptotic normality:
\begin{align*}
    \frac{\sqrt{N}(\hat\Psi_t - \Psi_t^*) }{\breve\sigma} \xrightarrow[]{d} \mathcal{N}(0,1)
\end{align*}
where $\breve\sigma^2 = \E[ \nu_t(T_i, \tilde L_i, \hat L_i,  S_i, \bY_i, \bZ_i)^2]$.
\end{theorem}
The proof is omitted, as the result follows directly from Proposition~1 of \citet{egami2024using}. 

\newpage
\section{Theoretical Results}
\subsection{Proof of Lemma \ref{existence}}\label{proof_existence}
\begin{proof}
I first prove that $\eta(\bY_i, \bZ_i) = \P(T_i, \tilde L_i \mid \bY_i, \bZ_i)$ satisfies $\{\bY_i, \bZ_i\} \ \indep \ \{L_i, T_i\} \mid \eta(\bY_i, \bZ_i)$. Now,
\begin{align*}
    \P(T_i, \tilde L_i \mid \eta(\bY_i, \bZ_i)) &= \E[\P(T_i, \tilde L_i \mid \eta(\bY_i, \bZ_i), \bY_i, \bZ_i) \mid \eta(\bY_i, \bZ_i)]\\
    &= \E[\P(T_i, \tilde L_i \mid \bY_i, \bZ_i) \mid \eta(\bY_i, \bZ_i)]\\
    &= \E[\eta(\bY_i, \bZ_i) \mid \eta(\bY_i, \bZ_i)] \\
    &= \eta(\bY_i, \bZ_i) =  \P(T_i, \tilde L_i \mid \bY_i, \bZ_i)\\
    &= \P(T_i, \tilde L_i \mid \bY_i, \bZ_i, \eta(\bY_i, \bZ_i))
\end{align*}
which means that 
\begin{align*}
    \{T_i, \tilde L_i\} \ \indep \{\bY_i, \bZ_i\} \mid \eta(\bY_i, \bZ_i)
\end{align*}
Therefore, $\eta(\bY_i, \bZ_i)$ is a trivial example of surrogate representations.

Then, I prove the main statement: $\{T_i, \tilde L_i\} \ \indep \{\bY_i, \bZ_i\} \mid \boldf(\bY_i, \bZ_i)$ holds if and only if $\eta(\bY_i, \bZ_i) = b(\boldf(\bY_i, \bZ_i))$ for some function $b$. Suppose that there exists some function $b$ such that $\eta(\bY_i, \bZ_i) = b(\boldf(\bY_i, \bZ_i))$. Then,
\begin{align*}
    \P(T_i, \tilde L_i \mid \boldf(\bY_i, \bZ_i)) &= \E[\P(T_i, \tilde L_i \mid \boldf(\bY_i, \bZ_i), \bY_i, \bZ_i) \mid  \boldf(\bY_i, \bZ_i) ]\\
    &= \E[\P(T_i, \tilde L_i \mid \bY_i, \bZ_i) \mid  \boldf(\bY_i, \bZ_i) ]\\
    &= \E[\eta(\bY_i, \bZ_i) \mid  \boldf(\bY_i, \bZ_i) ]\\
    &= \E[b(\boldf(\bY_i, \bZ_i)) \mid  \boldf(\bY_i, \bZ_i) ]\\
    &= b(\boldf(\bY_i, \bZ_i))\\
    &= \eta(\bY_i, \bZ_i) =  \P(T_i, \tilde L_i \mid \bY_i, \bZ_i)\\
    &= \P(T_i, \tilde L_i \mid \bY_i, \bZ_i, \boldf(\bY_i, \bZ_i))
\end{align*}
where the first line is by the law of iterated expectation.
Therefore,
\begin{align*}
    \{T_i, \tilde L_i\} \ \indep \  \bY_i, \bZ_i \ \mid \boldf(\bY_i, \bZ_i)
\end{align*}
and thus $\boldf(\bY_i, \bZ_i)$ is a surrogate representation.

For the converse, suppose that $\boldf(\bY_i, \bZ_i)$ is a surrogate representation, but $\boldf(\bY_i, \bZ_i)$ is not finer than $\eta(\bY_i, \bZ_i)$. This means that there exist $(\bY_j, \bZ_j)$ and $(\bY_k, \bZ_k)$ with $j \neq k$ such that $\boldf(\bY_j, \bZ_j) = \boldf(\bY_k, \bZ_k)$ but $\eta(\bY_j, \bZ_j) \neq \eta(\bY_k, \bZ_k)$. However, then
\begin{align*}
    \eta(\bY_j, \bZ_j) = \P(T_j, \tilde L_j \mid \bY_j, \bZ_j)  \neq \P(T_k, \tilde L_k \mid \bY_k, \bZ_k) = \eta(\bY_k, \bZ_k)
\end{align*}
Since $\boldf(\bY_j, \bZ_j) = \boldf(\bY_k, \bZ_k)$, this means that within each strata of $\boldf(\bY_i, \bZ_i)$, the independence does not hold; i.e., $\{T_i, \tilde L_i\} \not\!\perp\!\!\!\perp \{\bY_i, \bZ_i\} \mid \boldf(\bY_i, \bZ_i)$. However, this contradicts with the statement that $\boldf(\bY_i, \bZ_i)$ is a surrogate representation. Therefore, to be a surrogate representation, $\boldf(\bY_i, \bZ_i)$ must be finer than $\eta(\bY_i, \bZ_i)$.
\end{proof}

\newpage 

\subsection{Proof of Proposition \ref{identification} (Identification of Target Parameter)} \label{proof_iden}
\begin{proof} 
By Definition \ref{def_surrep}, $\bW_i$ satisfies $T_i \; \indep \; \{\bY_i, \bZ_i\} \mid \bW_i$. Thus, the use of Assumption~\ref{coding} yields
\begin{align}
    T_i \; \indep \; L_i \mid \bW_i, \bZ_i. \label{sufficiency_L}
\end{align}
Now, notice that Assumption~\ref{labeling} indicates the marginal independence $\bW_i \; \indep \; S_i$ because $\bW_i$ is the deterministic function of $\bY_i$ and $\bZ_i$. Under Assumption~\ref{labeling}, I can claim
\begin{align}
    \bY_i \; \indep \; S_i \; \mid \bW_i, \bZ_i,
\end{align}
because
\begin{align*}
    \P(\bY_i = \bm{y}, S_i = s \mid \bW_i \in A, \bZ_i \in B) &= \frac{\P(\bY_i = \bm{y}, S_i = s, \bW_i \in A, \bZ_i \in B )}{\P(\bW_i \in A, \bZ_i \in B)}\\
    &= \frac{\P(\bY_i = \bm{y}, \bW_i \in A, \bZ_i \in B) \; \P(S_i = s)}{\P(\bW_i \in A, \bZ_i \in B)}\\
    &= \P(\bY_i = \bm{y} \mid \bW_i \in A, \bZ_i \in B) \; \P(S_i = s)\\
    &=  \P(\bY_i = \bm{y} \mid \bW_i \in A, \bZ_i \in B) \; \P(S_i = s  \mid \bW_i \in A, \bZ_i \in B)
\end{align*}
for any $\bm{y} \in \cY$ and any pair of Borel sets $A, B$. In the above transformation, I apply Assumption~\ref{labeling} and the fact that $\bW_i$ is the deterministic function of $\bY_i$ in the second line, and the final line is because $\{\bW_i, \bZ_i\} \; \indep \; S_i$. By applying Assumption \ref{coding}, I obtain
\begin{align}
    L_i \; \indep \; S_i \; \mid \bW_i, \bZ_i. \label{randomlabel}
\end{align}
Therefore, 
\begin{align*}
    \Psi_t := \E\biggl[ \E[L_i \mid T_i = t, \bZ_i] \biggr] &= \int_{\cZ} \int_{\mathcal{W}} \E[L_i \mid T_i = t, \bm{W}_i, \bZ_i]  dF(\bm{W}_i \mid \bZ_i,  T_i = t) dF(\bZ_i)\\
    &= \int_{\cZ} \int_{\mathcal{W}} \E[L_i \mid \bm{W}_i, \bZ_i]  dF(\bm{W}_i \mid \bZ_i,  T_i = t) dF(\bZ_i)\\
    &= \int_{\cZ} \int_{\mathcal{W}} \E[L_i \mid \bm{W}_i]  dF(\bm{W}_i \mid \bZ_i,  T_i = t) dF(\bZ_i)\\
    &= \int_{\cZ} \int_{\mathcal{W}} \E[ L_i \mid \bm{W}_i, S_i = 1]  dF(\bm{W}_i \mid \bZ_i,  T_i = t) dF(\bZ_i)\\
    &= \int_{\cZ} \int_{\mathcal{W}} \E[\tilde L_i \mid \bm{W}_i, S_i = 1]  dF(\bm{W}_i \mid \bZ_i,  T_i = t) dF(\bZ_i)
\end{align*}
where $\mathcal{W}$ is the support of the latent representation $\bW_i = \boldf(\bY_i, \bZ_i)$. In the transformations above, I use the law of iterated expectation in the first line, Equation \ref{sufficiency_L}
in the second line, $L_i \ \indep \ \bZ_i \mid \bW_i$ (Equation \eqref{indep_surrogate}) in the third line, Equation \ref{randomlabel} in the fourth line, and Assumption \ref{perfect} in the last line. As $\tilde L_i$ is observed under $S_i = 1$ by definition, I can estimate both $\E[\tilde L_i \mid \bm{W}_i, \bZ_i, S_i = 1] $ and $\P(\bm{W}_i \mid \bZ_i,T_i = t)$ from the data. 

Finally, I show that $\bm{W}$ does not need to be unique. Notice that the formula above is equivalent to
\begin{align*}
    \Psi_t = \int_{\cZ} \int_{\mathcal{Y}} \E[\tilde L_i \mid \boldf(\bY_i,  \bZ_i), S_i = 1]  dF(\bY_i \mid \bZ_i,  T_i = t) dF(\bZ_i)
\end{align*}
where $g_{L}$ is uniquely specified by researchers but $\bm{f}$ is estimated from data and might not be unique. Suppose that there exists another $\bW' = \bm{f}'(\bY_i, \bZ_i)$ that satisfies $\{T_i, L_i\} \; \indep \; \{\bm{Y}_i, \bZ_i\} \mid \bW'$. Then,
\begin{align*}
     \Psi_t &= \int_{\cZ} \int_{\mathcal{Y}} \E[ \tilde  L_i \mid \boldf'(\bY_i, \bZ_i), S_i = 1]  dF(\bY_i \mid \bZ_i,  T_i = t) dF(\bZ_i)\\
&= \int_{\cZ} \int_{\mathcal{Y}} \E[\tilde  L_i \mid \boldf'(\bY_i, \bZ_i), \bY_i, \bZ_i, S_i = 1]  dF(\bY_i \mid \bZ_i,  T_i = t) dF(\bZ_i)\\
&= \int_{\cZ} \int_{\mathcal{Y}} \E[\tilde  L_i \mid \bY_i, \bZ_i, S_i = 1]  dF(\bY_i \mid \bZ_i,  T_i = t) dF(\bZ_i)\\
&= \int_{\cZ} \int_{\mathcal{Y}} \E[\tilde L_i \mid \boldf(\bY_i, \bZ_i), S_i = 1]  dF(\bY_i \mid \bZ_i,  T_i = t) dF(\bZ_i).
\end{align*}
Hence, I can say that the different $\bm{f}$ still leads to the same identification formula.

\end{proof}

\newpage
\subsection{Proof of Theorem \ref{eif} (Efficient Influence Function)} \label{proof_eif}
\begin{proof}
I split the proof to the two parts.

\medskip

\noindent \textbf{STEP 1: Deriving Influence function}: Firstly, I derive the influence function. Here, I use the technique called point-mass contamination approach \citep{kennedy_semiparametric_2023, hines_demystifying_2022}. Following \cite{hines_demystifying_2022}, I assume that $\bZ$, $L$, and $\bY$ are continuous only for the notational convenience, but as \cite{hines_demystifying_2022} discussed, the entire proof works with discrete variables or a mix of discrete and continuous variables.

Let $\cP$ be a true data distribution and $\Psi_t(\cP)$ be an estimand under $T = t \in \{0,1\}$ and true data distribution $\mathcal{P}$. That is,
\begin{align*}
\Psi_{t}(\cP) &= \int_{\cZ}  \int_{\cY} \mu_{\cP}(\boldf_\cP(\bm{y}, \bm{z}))   f_{\cP}(\bm{y} \mid \bm{z},  T_i = t) f_{\cP}(\bm{z}) d\bm{y}d \bm{z}
\end{align*}
where $\mu_{\cP}(\boldf_\cP(\bY_i, \bZ_i)) = \mathbb{E}_\cP[\tilde L_i \mid \boldf_\cP(\bY_i, \bZ_i),  S_i = 1] = \mathbb{E}_\cP[\tilde L_i \mid \bW_i, S_i = 1]$, $\E_\cP[\cdot]$ is the expectation operator under the distribution $\cP$, and $f_\cP(\cdot)$ is the PDF under the true distribution $\cP$. Then, consider the parametric submodel $\cP_{\epsilon} = (1 - \epsilon) \cP + \epsilon \tilde{\cP}$, where $\tilde{\cP}$ is the point mass at $(\tilde{t}, \tilde{\bm{y}}, \tilde{l}, \tilde{\bz})$ and $\epsilon \in [0,1]$. For $t \in \{0,1\}$, consider the influence of perturbation in the direction of $\Tilde{\cP}$:
\begin{align*}
    \frac{d}{d \epsilon}\Psi_{t}(\cP_\epsilon) \biggl|_{\epsilon = 0} = \frac{d}{d \epsilon} \int_{\cZ}  \int_{\cY}  \mu_{\cP_\epsilon}(\boldf_{\cP_\epsilon}(\bm{y}, \bm{z}) )   f_{\cP_\epsilon}(\bm{y} \mid \bm{z},  T_i = t) f_{\cP_\epsilon}(\bm{z}) d\bm{y} d\bm{z} \biggl|_{\epsilon = 0} 
\end{align*}
where $\mu_{\cP_\epsilon}(\boldf_{\cP_\epsilon}(\bY_i, \bZ_i)) = \mathbb{E}_{\cP_\epsilon}[\tilde L_i \mid \boldf_{\cP_\epsilon}(\bY_i, \bZ_i), S_i = 1]$ is the outcome model under $\cP_\epsilon$, $\E_{\cP_\epsilon}[\cdot]$ is the expectation operator under $\cP_\epsilon$, and $F_{\cP_\epsilon}(\cdot)$ is the CDF under $\cP_\epsilon$. As $\mu_{\cP_\epsilon}(\boldf_{\cP_\epsilon}(\bY_i, \bZ_i))$ can be regarded as a composite function, for the sake of simplicity, let $m_{\cP}(\bY_i, \bZ_i) = \mu_{\cP}(\boldf_{\cP}(\bY_i), \bZ_i)$ and $m_{\cP_\epsilon}(\bY_i, \bZ_i) = \mu_{\cP_\epsilon}(\boldf_{\cP_\epsilon}(\bY_i, \bZ_i))$. Then, under the regularity condition, I can switch the derivative and integrals, and thus the chain rule yields
\begin{align*}
    &\frac{d}{d \epsilon}  \int_{\cZ}  \int_{\cY} m_{\cP_\epsilon}(\bm{y}, \bm{z}) \,  f_{\cP_\epsilon}(\bm{y} \mid \bm{z},  T_i = t) f_{\cP_\epsilon}(\bm{z}) d\bm{y} d\bm{z}\biggl|_{\epsilon = 0} \\
    &\quad\quad = \int_{\cZ}  \int_{\cY} \frac{d}{d \epsilon} m_{\cP_\epsilon}(\bm{y}, \bm{z}) \biggl|_{\epsilon = 0} \,  f_{\cP}(\bm{y} \mid \bm{z},  T_i = t) f_{\cP}(\bm{z}) d\bm{y} d\bm{z} \\
    &\quad\quad\quad\quad\quad\quad +  \int_{\cZ} \int_{\cY}   m_\cP(\bm{y}, \bm{z})  \frac{d}{d \epsilon} \,
    f_{\cP_\epsilon}(\bm{y} \mid \bm{z},  T_i = t)
    \biggl|_{\epsilon = 0} f_{\cP}(\bm{z}) d\bm{y} d\bm{z}\\
    &\quad\quad\quad\quad\quad\quad + \int_{\cZ}  \int_{\cY}  m_\cP(\bm{y}, \bm{z})  f_{\cP}(\bm{y} \mid \bm{z},  T_i = t)  \frac{d}{d \epsilon} \,f_{\cP_\epsilon}(\bm{z})\biggl|_{\epsilon = 0} d\bm{y} d\bm{z}
\end{align*}
For the first component,
\begin{align*}
    &\int_{\cZ}  \int_{\cY} \frac{d}{d \epsilon} m_{\cP_\epsilon}(\bm{y}, \bm{z}) \biggl|_{\epsilon = 0} \,  f_{\cP}(\bm{y} \mid \bm{z},  T_i = t) f_{\cP}(\bm{z}) d\bm{y} d\bm{z}\\
    &= \int_{\cZ}  \int_{ \mathcal{Y}} \biggl\{ \frac{d}{d \epsilon} \int_{\cL} l f_{\cP_{\epsilon}}(l \mid \boldf(\bm{y}), \bm{z}, S_i = 1) d l  \; \biggl|_{\epsilon = 0} \biggr\} \,\,f_{\cP}(\bm{y} \mid \bm{z},  T_i = t) f_{\cP}(\bm{z}) d\bm{y}d\bm{z} \\
    &= \int_{\cZ}  \int_{ \mathcal{Y}} \biggl\{ \frac{d}{d \epsilon} \int_{\cL} l f_{\cP_{\epsilon}}(l \mid \bm{y}, \bm{z}, S_i = 1) d l  \; \biggl|_{\epsilon = 0} \biggr\} \,\,f_{\cP}(\bm{y} \mid \bm{z},  T_i = t) f_{\cP}(\bm{z}) d\bm{y}d\bm{z} \\
    &= \int_{\cZ} \int_{ \mathcal{Y}} \biggl\{ \frac{ \mathbbm{1}\{\bm{y} = \tilde{\bm{y}}, \bm{z} = \tilde{\bz}, S_i = 1\} }{f_\cP(\bm{y}, \bm{z}, S_i = 1)}  \biggl( \Tilde{l} - \E_{\cP}[\tilde L_i \mid \bm{y}, \bm{z}, S_i = 1] \biggr) \biggr\} \,f_{\cP}(\bm{y} \mid \bm{z},  T_i = t) f_{\cP}(\bm{z}) d\bm{y}d\bm{z}\\
    &= \frac{\mathbbm{1}\{ S_i = 1 \} }{\mathbb{P}(S = 1)} \int_{\cZ}  \int_{ \mathcal{Y}} \biggl\{ \frac{ \mathbbm{1}\{\bm{y} = \tilde{\bm{y}}, \bm{z} = \tilde{\bz}\} }{f_\cP(\bm{y}, \bm{z})}  \biggl( \Tilde{l} - \E_{\cP}[\tilde L_i \mid \bm{y}, \bm{z}, S_i = 1] \biggr) \biggr\} \,f_{\cP}(\bm{y} \mid \bm{z},  T_i = t) f_{\cP}(\bm{z}) d\bm{y}d\bm{z}\\
    &= \frac{\mathbbm{1}\{ S_i = 1 \} }{\mathbb{P}(S = 1)} \int_{\cZ}  \int_{ \mathcal{Y}}  \mathbbm{1}\{\bm{y} = \tilde{\bm{y}}, \bm{z} = \tilde{\bz} \}  \biggl( \Tilde{l} - \E_{\cP}[\tilde L_i \mid \bm{y}, \bm{z}, S_i = 1] \biggr) \,\frac{\,f_{\cP}(\bm{y} \mid \bm{z},  T_i = t) f_{\cP}(\bm{z})}{f_{\cP}(\bm{y}, \bm{z})} d\bm{y}d\bm{z}\\
    &= \frac{\mathbbm{1}\{ S_i = 1 \} }{\mathbb{P}(S = 1)} \int_{\cZ} \int_{ \mathcal{Y}}  \mathbbm{1}\{\bm{y} = \tilde{\bm{y}}, \bm{z} = \tilde{\bz} \}  \biggl( \Tilde{l} - \E_{\cP}[\tilde L_i \mid \bm{y}, \bm{z}, S_i = 1] \biggr) \,\frac{\,f_{\cP}(\bm{y} \mid \bm{z},  T_i = t)}{f_{\cP}(\bm{y} \mid \bm{z})} d\bm{y}d\bm{z}\\
    &= \frac{\mathbbm{1}\{S_i = 1\}}{\P(S_i = 1)} \cdot  \frac{\P(T_i = t \mid \bY_i = \tilde{\bm{y}}, \bZ_i = \tilde\bz)}{\P(T_i = t \mid \bZ_i = \tilde\bz)} \biggl(\Tilde{l} - \E_{\cP}[\tilde L_i \mid \bm{y}, \bm{z}, S_i = 1] \biggr)\\
    &= \frac{\mathbbm{1}\{S_i = 1\}}{\P(S_i = 1)} \cdot  \frac{\P(T_i = t \mid \boldf(\bY_i, \bZ_i) = \boldf(\tilde{\bm{y}}, \tilde{\bm{z}}))}{\P(T_i = t \mid \bZ_i = \tilde\bz)} \biggl(\Tilde{l} - \E_{\cP}[\tilde L_i \mid \boldf(\bm{y},  \bm{z}), S_i = 1] \biggr)\\
    &= \frac{\mathbbm{1}\{S_i = 1\}}{\P(S_i = 1)} \cdot  \frac{\P(T_i = t \mid \boldf(\bY_i, \bZ_i) = \boldf(\tilde{\bm{y}}, \tilde{\bm{z}}))}{\P(T_i = t \mid \bZ_i = \tilde\bz)} \biggl(\Tilde{l} - m_\cP(\tilde{\bm{y}}, \tilde\bz)\biggr)
\end{align*}
where the third equality is because of the influence function for the conditional outcome mean function (see Example 6 of \citealt{hines_demystifying_2022}), the fourth equality is by $S_i \; \indep \; \bY_i$, and the second and eighth equalities are by $\bY_i \ \indep \ \{T_i, L_i\} \mid \boldf(\bY_i, \bZ_i)$.

For the second component,
\begin{align*}
    &\int_{\cZ} \int_{\cY}    m_\cP(\bm{y}, \bm{z}) \frac{d}{d \epsilon} f_{\cP_\epsilon}(\bm{y} \mid \bm{z},  T_i = t) \biggl|_{\epsilon = 0} f_{\cP}(\bm{z}) d\bm{y}d\bm{z}\\
    &=  \int_{\cZ}  \int_{\cY}   m_\cP(\bm{y}, \bm{z}) \biggl( \frac{d}{d \epsilon} \frac{f_{\cP_\epsilon}(\bm{y}, \bm{z},  T_i=t)}{f_{\cP_\epsilon}( \bm{z}, T_i=t)}\biggl|_{\epsilon = 0} \biggr) f_{\cP}(\bm{z}) d\bm{y}d\bm{z}\\
    &= \int_{\cZ}  \int_{ \mathcal{Y}}  m_\cP(\bm{y}, \bm{z}) \biggl(
    \frac{
    \{\frac{d}{d\epsilon}
    f_{\cP_\epsilon}(\bm{y}, \bm{z}, T_i=t)
    |_{\epsilon = 0}
    \}
    f_{\cP}(\bm{z}, T_i = t)
    }{\{f_{\cP}(\bm{z}, T_i = t)\}^2}
    - \frac{
    \{f_{\cP}(\bm{y}, \bm{z}, T_i=t)
    \{
    \frac{d}{d\epsilon}f_{\cP_\epsilon}(\bm{z}, T_i=t)|_{\epsilon = 0}
    \}
    }{\{f_{\cP}(\bm{z}, T_i = t)\}^2} \biggr) f_{\cP}(\bm{z})d\bm{y}d\bm{z}
    \\
    &= \int_{\cZ}  \int_{ \mathcal{Y}}  m_\cP(\bm{y}, \bm{z}) \biggl(
    \frac{
    \{
     \mathbbm{1}\{\bm{y} = \bm{\Tilde{y}}, \bm{z} = \Tilde{\bz}, T_i = t\}
     - f_{\cP}(\bm{y}, \bm{z}, T_i = t)
    \} f_{\cP}(\bm{z}, T_i  = t)
    }{\{f_{\cP}(\bm{z}, T_i = t)\}^2}\\
    &\quad\quad\quad\quad\quad\quad\quad\quad - \frac{f_{\cP}(\bm{y}, \bm{z}, T_i=t)
    \{
    \mathbbm{1}\{T_i = t, \bm{z} = \Tilde{\bz}\} - f_{\cP}(\bm{z} ,T_i = t)
    \}}{\{f_{\cP}(\bm{z}, T_i = t)\}^2}
    \biggr)f_{\cP}(\bm{z}) d\bm{y}d\bm{z}\\
    &= \int_{\cZ}  \int_{ \mathcal{Y}}  m_\cP(\bm{y}, \bm{z}) \biggl(
    \frac{
    \{
    \mathbbm{1}\{\bm{y} = \bm{\Tilde{y}}, \bm{z} = \Tilde{\bz}, T_i = t\}
    f_{\cP}(\bm{z}, T_i = t) - f_{\cP}(\bm{y}, \bm{z}, T_i=t)\mathbbm{1}\{T_i = t, \bm{z} = \Tilde{\bz}\}
    }{\{f_{\cP}(\bm{z}, T_i = t)\}^2}
    \biggr)f_{\cP}(\bm{z}) d\bm{y}d\bm{z}\\
    &= \mathbbm{1}\{T_i = t\} \int_{\cZ}  \int_{ \mathcal{Y}}  m_\cP(\bm{y}, \bm{z}) \biggl(
    \mathbbm{1}\{\bm{y} = \bm{\Tilde{y}}\} - f_{\cP}(\bm{y} \mid \bm{z}, T_i=t)\}
    \biggr) \mathbbm{1}\{\bm{z} = \Tilde{\bz}\} \frac{f_{\cP}(\bm{z})}{f_{\cP}(\bm{z}, T_i = t)} d\bm{y}d\bm{z}\\
    &= \frac{\mathbbm{1}\{T_i = t\} }{\P(T_i = t \mid \bZ_i = \tilde{\bm{z}} )} \biggl( m_\cP(\Tilde{\bm{y}}, \Tilde{\bz}) - \int_{ \mathcal{Y}} m_\cP(\bm{y}, \tilde{\bz})f_{\cP}(\bm{y} \mid \bm{Z}_i = \Tilde{\bm{z}}, T_i = t) \biggr)
\end{align*}
where the fourth line is by Example 2 of \cite{hines_demystifying_2022}
\begin{align*}
    \frac{d}{d\epsilon}
    f_{\cP_\epsilon}(\bm{y}, \bm{z}, T_i=t)
    |_{\epsilon = 0} &= 
    \frac{d}{d\epsilon} \biggl(\epsilon \mathbbm{1}\{ \bm{y} = \Tilde{\bm{y}}, 
    \bm{z} = \tilde{\bm{z}},
    T_i = t \} + (1 - \epsilon) f_{\cP}(\bm{y}, \bm{z}, T_i=t)\bigg)\\
    &= \mathbbm{1}\{\bm{y} = \Tilde{\bm{y}}, \bm{z} = \tilde{\bm{z}}, T_i = t\} - f_{\cP}(\bm{y}, \bm{z}, T_i=t)
\end{align*}
Finally, for the last component,
\begin{align*}
     &\int_{\cZ}  \int_{\cY}  m_\cP(\bm{y}, \bm{z}) f_{\cP}(\bm{y} \mid \bm{z},  T_i = t)  \frac{d}{d \epsilon} \,f_{\cP_\epsilon}(\bm{z})\biggl|_{\epsilon = 0} d\bm{y}d\bm{z}\\
     &= \int_{\cZ}  \int_{\cY}  m_\cP(\bm{y}, \bm{z}) dF_{\cP}(\bm{y} \mid \bm{z},  T_i = t) \biggl(\mathbbm{1}\{\bm{z} = \Tilde{\bz}\} - f_{\cP}(\bm{z})\biggr)d\bm{y}d\bm{z}\\
     &= \int_{\cY}  m_\cP(\bm{y}, \bm{z}) f_{\cP}(\bm{y} \mid \bZ_i = \Tilde{\bz},  T_i = t)d\bm{y}  - \Psi_t
\end{align*}
Therefore, 
\begin{align*}
    &\frac{d}{d \epsilon}\Psi_{t}(\cP_\epsilon) \biggl|_{\epsilon = 0} = \frac{\mathbbm{1}\{S_i = 1\}}{\P(S_i = 1)} \cdot  \frac{\rho_t(\boldf( \Tilde{\bm{y}}), \Tilde{\bm{z}})}{\pi_t(\Tilde{\bm{z}})} \biggl(\Tilde{l} - m(\tilde{\bm{y}}, \tilde\bz)\biggr) + 
    \frac{\mathbbm{1}\{T_i = t\} }{\pi_t(\Tilde{\bm{z}})} \biggl( m(\Tilde{\bm{y}}, \Tilde{\bz}) - \Bar{m}_t(\tilde{z}) \biggr) + \Bar{m}_t(\tilde{z}) - \Psi_t
\end{align*}
where
\begin{align*}
    \pi_t(\Tilde{\bm{z}}) &= \P(T_i = t \mid \bZ_i = \tilde\bz)\\
    \rho_t(\boldf( \Tilde{\bm{y}}, \Tilde{\bm{z}})) &= \P(T_i = t \mid \boldf(\bY_i) = \boldf(\tilde{\bm{y}},  \tilde\bz))\\
    \Bar{m}_t(\tilde{\bm{z}}) &= \int_{ \mathcal{Y}} m(\bm{y}, \tilde{\bz})f(\bm{y} \mid \bZ_i = \Tilde{\bm{z}}, T_i = t)d\bm{y}.
\end{align*}
Because $S_i$ and $T_i$ are assumed to be binary, the variance of indicator function is bounded and thus its variance is finite, meaning that the estimand is pathwise differentiable. Hence, the influence function is given by
\begin{equation}
\begin{aligned}
    &\psi_t(T_i, \tilde L_i, S_i, \bY_i; \boldf, \mu, \pi_t, \rho_t, \Bar{m}_t, \Psi_t) \\
    &\qquad\qquad= \frac{\mathbbm{1}\{S_i = 1\}}{\P(S_i = 1)} \cdot  \frac{\rho_t(\boldf( \bY_i, \bZ_i))}{\pi_t(\bZ_i)} \biggl(\tilde L_i - \mu( \boldf(\bY_i, \bZ_i))\biggr)\\
    &\qquad\qquad\qquad\qquad\qquad\qquad+ 
    \frac{\mathbbm{1}\{T_i = t\} }{\pi_t(\bZ_i)} \biggl( \mu( \boldf(\bY_i, \bZ_i)) - \Bar{m}_t(\bZ_i) \biggr) + \Bar{m}_t(\bZ_i) - \Psi_t. \label{eif_formula}
\end{aligned}
\end{equation}

\medskip

\noindent \textbf{STEP 2: Deriving Efficient Influence Function}: Since the model is semiparametric due to the independence constraint, the derivation of efficient influence function requires the projection onto the tangent space. Therefore, based on the previous step, I then confirm if the influence function derived in the previous step is an efficient influence function or not. Let $O_i = \{S_i, T_i, \bY_i, \bZ_i\}$ be the observed data. Now, the full data distribution is factorized into
\begin{align*}
    &f(T_i, \bY_i, \bZ_i, S_i)\\
    &= f_S(S_i)
    f_{\bZ}(\bZ_i) 
    f_{\bY \mid \bZ}(\bY_i \mid \bZ_i) f_{T\mid \bY, \bZ_i}(T_i \mid \bY_i, \bZ_i)f(\tilde L_i \mid S_i, T_i, \bY_i, \bZ_i)^{S_i}\\
    &= f_S(S_i)f_{\bZ}(\bZ_i) 
    f_{\bY \mid \bZ}(\bY_i \mid \bZ_i) f_{T\mid \bY, \bZ_i}(T_i \mid \bY_i, \bZ_i)f(\tilde L_i \mid T_i, \bY_i, \bZ_i)^{S_i}\\
    &= f_S(S_i)f_{\bZ}(\bZ_i) 
    f_{\bY \mid \bZ}(\bY_i \mid \bZ_i) f_{T\mid \bY, \bZ_i}(T_i \mid \bY_i, \bZ_i)f(\tilde L_i \mid \bY_i, \bZ_i, S_i = 1).
\end{align*}
By Theorem 4.5 of \cite{tsiatis_semiparametric_2006}, the tangent space of the nonparametric model $\Lambda$ is expressed as a direct sum of the orthogonal subspaces: i.e., 
\begin{align}
    \Lambda = \biggl\{ 
    \Lambda_{S} \ \oplus \ \Lambda_{\bZ}  \oplus  \Lambda_{\bY \mid \bZ} \ \oplus \ \Lambda_{T\mid \bY, \bZ} \ \oplus \ \Lambda_{\tilde L \mid \bY, \bZ, S_i = 1}\biggr\} \label{factorization}
\end{align}
where
\begin{align*}
    &\Lambda_{S} = \biggl\{ \alpha_0(S_i) \in \mathcal{H}  : \mathbb{E}[\alpha_0(S_i)] = 0 \biggr\}\\
    &\Lambda_{\bZ} = \biggl\{ \alpha_1(\bZ_i) \in \mathcal{H}  : \mathbb{E}[\alpha_1(\bZ_i)] = 0 \biggr\}\\
    &\Lambda_{\bY \mid \bZ} = \biggl\{ \alpha_2(\bY_i, \bZ_i) \in \mathcal{H}  : \mathbb{E}[\alpha_2(\bY_i, \bZ_i) \mid \bZ_i] = 0 \biggr\}\\
    &\Lambda_{T \mid \bY, \bZ} = \biggl\{ \alpha_3(T_i, \bY_i, \bZ_i) \in \mathcal{H}  : \mathbb{E}[\alpha_3(T_i, \bY_i, \bZ_i) \mid \bY_i, \bZ_i] = 0 \biggr\}\\
    &\Lambda_{\tilde L \mid \bY, \bZ} = \biggl\{ \mathbbm{1}\{S_i = 1\}\alpha_4(\tilde L_i, \bY_i, \bZ_i) \in \mathcal{H}  : \mathbb{E}[\alpha_4(\tilde L_i, \bY_i, \bZ_i) \mid \bY_i, \bZ_i, S_i = 1] = 0 \biggr\}
\end{align*}
where $\mathcal{H}$ is the Hilbert space of the full data with mean 0 and finite variance equipped with the covariance inner product.

Now, recall that efficient influence function is the influence function that resides in the tangent space (Theorem 4.3 of \citealt{tsiatis_semiparametric_2006}). Then, notice that Equation \eqref{eif_formula} is written as
\begin{align*}
    \psi_t(T_i, \tilde L_i, S_i, \bY_i; \boldf, \mu, \pi_t, \rho_t, \Bar{m}_t, \Psi_t) = \psi_{1t}(\bZ_i) + \psi_{2t}(T_i, \bY_i, \bZ_i) +
    \psi_{3t}(\bY_i, \bZ_i) + 
    \psi_{4t}(S, \tilde L_i, \bZ_i, \bY_i)
\end{align*}
where
\begin{align*}
    &\psi_{1t}(\bZ_i) =  \Bar{m}_t(\bZ_i) - \Psi_t\\
    & \psi_{2t}(T_i, \bY_i, \bZ_i)  = \frac{\mathbbm{1}\{T_i = t\} - \rho_t(\boldf(\bY_i, \bZ_i)) }{\pi_t(\bZ_i)} \biggl( \mu( \boldf(\bY_i, \bZ_i)) - \Bar{m}_t(\bZ_i) \biggr)\\
    &\psi_{3t}(\bY_i, \bZ_i) = \frac{\rho_t(\boldf(\bY_i, \bZ_i))}{\pi_t(\bZ_i)}\biggl( \mu( \boldf(\bY_i, \bZ_i)) - \Bar{m}_t(\bZ_i) \biggr) \\
    &\psi_{4t}(S, \tilde L_i, \bY_i, \bZ_i) = \frac{\mathbbm{1}\{S_i = 1\}}{\P(S_i = 1)} \cdot  \frac{\rho_t(\boldf( \bY_i, \bZ_i))}{\pi_t(\bZ_i)} \biggl(\tilde L_i - \mu( \boldf(\bY_i, \bZ_i))\biggr)
\end{align*}
and thus I need to check if $\psi_{1t}(\bZ_i) \in \Lambda_{\bZ_i}$, $\psi_{2t}(T_i, \bY_i, \bZ_i) \in \Lambda_{T \mid \bY, \bZ}$, 
$\psi_{3t}(\bY_i, \bZ_i) \in \Lambda_{\bY \mid \bZ}$
and $\psi_{4t}(S, \tilde L_i, \bZ_i, \bY_i) \in \Lambda_{\tilde L \mid \bY, \bZ}$. 
\begin{enumerate}
    \item $\psi_{1t}(\bZ_i) \in \Lambda_{\bZ_i}$: To show $\psi_{1t}(\bZ_i) \in \Lambda_{\bZ_i}$, I need to show $\E[\psi_{1t}(\bZ_i)] = 0$. Then,
    \begin{align*}
        \E[\psi_{1t}(\bZ_i)] &= \E\biggl[\Bar{m}_t(\bZ_i) - \Psi_t \biggr]\\
        &= \E\biggl[\int \mu( \boldf(\bY_i, \bZ_i)) dF(\bY_i \mid \bZ_i, T_i = t) - \Psi_t \biggr] = 0.
    \end{align*}
    \item $\psi_{2t}(T_i, \bY_i, \bZ_i) \in \Lambda_{T \mid \bY, \bZ}$: Similarly, to show the statement, I need to show $\E[\psi_{2t}(T_i, \bY_i, \bZ_i) \mid \bY_i, \bZ_i] = 0$. Then,
    \begin{align*}
        \E[\psi_{2t}(T_i, \bY_i, \bZ_i) \mid \bY_i, \bZ_i] &= \E \biggl[\frac{\mathbbm{1}\{T_i = t\} - \rho_t(\boldf(\bY_i, \bZ_i)) }{\pi_t(\bZ_i)} \biggl( \mu( \boldf(\bY_i, \bZ_i)) - \Bar{m}_t(\bZ_i) \biggr) \mid \bY_i, \bZ_i \biggr]\\
        &= \frac{\P(T_i = t \mid \bY_i, \bZ_i) - \rho_t(\boldf(\bY_i, \bZ_i)) }{\pi_t(\bZ_i)} \biggl( \mu( \boldf(\bY_i, \bZ_i)) - \Bar{m}_t(\bZ_i) \biggr)\\
        &= \frac{\P(T_i = t \mid \boldf(\bY_i, \bZ_i)) - \rho_t(\boldf(\bY_i, \bZ_i)) }{\pi_t(\bZ_i)} \biggl( \mu( \boldf(\bY_i, \bZ_i)) - \Bar{m}_t(\bZ_i) \biggr)\\
        &= 0
    \end{align*}
    \item $\psi_{3t}(\bY_i, \bZ_i) \in \Lambda_{\bY_i \mid \bZ_i}$: Similarly, to show the statement, I need to show $\E[\psi_{3t}( \bY_i, \bZ_i) \mid \bZ_i] = 0$. Now,
    \begin{align*}
        \E[\psi_{3t}(\bY_i, \bZ_i)\mid \bZ_i] 
        &= \frac{1}{\pi_t(\bZ_i)}\biggl\{\E\biggl[ \rho_t(\boldf(\bY_i, \bZ_i))\mu( \boldf(\bY_i, \bZ_i)) \mid \bZ_i \biggr] - 
        \Bar{m}_t(\bZ_i) \E\biggl[ \rho_t(\boldf(\bY_i, \bZ_i)) \mid \bZ_i \biggr]
        \biggr\}.
    \end{align*}
    Then,
    \begin{align*}
        \E\biggl[ \rho_t(\boldf(\bY_i, \bZ_i)) \mid \bZ_i \biggr] &=
        \E\biggl[ \P(T_i = t \mid \boldf(\bY_i, \bZ_i)) \mid \bZ_i \biggr]\\
        &= \E\biggl[ \P(T_i = t \mid \bY_i, \bZ_i) \mid \bZ_i \biggr] = \P(T_i = t \mid \bZ_i) = \pi_t(\bZ_i)
    \end{align*}
    and
    \begin{align*}
         \E\biggl[ \rho_t(\boldf(\bY_i, \bZ_i))\mu( \boldf(\bY_i, \bZ_i)) \mid \bZ_i \biggr]
         &= 
         \E\biggl[ 
         \E[ \mathbbm{1}\{T_i = t\} \mid \boldf(\bY_i, \bZ_i)]
         \mu( \boldf(\bY_i, \bZ_i)) \mid \bZ_i \biggr]\\
         &= 
         \E\biggl[ 
         \E[ \mathbbm{1}\{T_i = t\} \mid \bY_i, \bZ_i]
         \mu( \boldf(\bY_i, \bZ_i)) \mid \bZ_i \biggr]\\
         &= 
         \E\biggl[ 
         \E[ \mathbbm{1}\{T_i = t\} \mu( \boldf(\bY_i, \bZ_i)) \mid \bY_i, \bZ_i]
          \mid \bZ_i \biggr]\\
          &= 
         \E\biggl[ 
         \mathbbm{1}\{T_i = t\} \mu( \boldf(\bY_i, \bZ_i))
          \mid \bZ_i \biggr]\\
          &= \E\biggl[ 
         \E[\mathbbm{1}\{T_i = t\} \mu( \boldf(\bY_i, \bZ_i))
         \mid T_i, \bZ_i]
          \mid \bZ_i \biggr]\\
          &= \E\biggl[ 
         \mathbbm{1}\{T_i = t\} \E[\mu( \boldf(\bY_i, \bZ_i))
         \mid T_i, \bZ_i]
          \mid \bZ_i \biggr]\\
          &= \E[\mu( \boldf(\bY_i, \bZ_i))
         \mid T_i = t, \bZ_i] \ \P(T_i = t \mid \bZ_i)\\
         &= \Bar{m}_t(\bZ_i) \pi(\bZ_i).
    \end{align*}
    Therefore, $\E[\psi_{3t}(\bY_i, \bZ_i)\mid \bZ_i] = 0$.
    \item $\psi_{4t}(S, \tilde L_i, \bY_i, \bZ_i)$: Finally, I show $\E[\psi_{4t}(S, \tilde L_i, \bY_i, \bZ_i) \mid \bY_i, \bZ_i, S_i = 1] = 0$.
    \begin{align*}
        &\E[\psi_{4t}(S, \tilde L_i, \bY_i, \bZ_i) \mid \bY_i, \bZ_i, S_i = 1]\\
        &= 
        \frac{1}{\P(S_i = 1)} \cdot  
        \E\biggl[
        \frac{\rho_t(\boldf( \bY_i, \bZ_i))}{\pi_t(\bZ_i)} \biggl(\tilde L_i - \mu( \boldf(\bY_i, \bZ_i))\biggr) \mid \bY_i, \bZ_i, S_i = 1 \biggr]\\
        &= 
        \frac{1}{\P(S_i = 1)} \cdot  
        \biggl\{
        \frac{\rho_t(\boldf( \bY_i, \bZ_i))}{\pi_t(\bZ_i)} \biggl(\E[\tilde L_i \mid \bY_i, \bZ_i, S_i = 1] - \mu( \boldf(\bY_i, \bZ_i))\biggr) \mid \bY_i, \bZ_i, S_i = 1 \biggr\}\\
        &= 
        \frac{1}{\P(S_i = 1)} \cdot  
        \biggl\{
        \frac{\rho_t(\boldf( \bY_i, \bZ_i))}{\pi_t(\bZ_i)} \biggl(\E[\tilde L_i \mid \bY_i, \bZ_i] - \mu( \boldf(\bY_i, \bZ_i))\biggr) \mid \bY_i, \bZ_i, S_i = 1 \biggr\}\\
        &= 
        \frac{1}{\P(S_i = 1)} \cdot  
        \biggl\{
        \frac{\rho_t(\boldf( \bY_i, \bZ_i))}{\pi_t(\bZ_i)} \biggl(\E[\tilde L_i \mid \boldf(\bY_i, \bZ_i)] - \mu( \boldf(\bY_i, \bZ_i))\biggr) \mid \bY_i, \bZ_i, S_i = 1 \biggr\}\\
        &= 0.
    \end{align*}
\end{enumerate}
Therefore, since the derived influence function resides in the tangent space, I can conclude that it is the efficient influence function.
\end{proof}

\begin{remark}
The derived efficient influence function has the same algebraic form as the efficient influence function obtained when conditioning on $(\bY_i, \bZ_i)$ directly rather than surrogate representation $\boldf(\bY_i, \bZ_i)$. This is because under Definition \ref{def_surrep}, I have
\begin{align*}
    &\mu(\boldf(\bY_i, \bZ_i)) = \E[\tilde L_i \mid \bY_i, \bZ_i, S_i = 1]\\
    &\rho_t(\boldf(\bY_i, \bZ_i)) = \E[T_i \mid \bY_i, \bZ_i].
\end{align*}
Consequently, the perturbations of $\boldf$ is already absorbed into perturbations of the induced nuisance functions, and the effect of estimating $\boldf$ is handled via the convergence rate conditions stated for the composite functions $\mu  \circ \boldf$ and $\rho_t  \circ \ \boldf$ (Assumptions \ref{reg} (c)).
\end{remark}

\newpage
\subsection{Proof of Theorem \ref{asymp_normal} (Asymptotic Normality)}\label{proof_asymp_normal}
For this proof, I use the following lemma to deal with the so-called empirical process term. See \cite{kennedy_semiparametric_2023} for the detailed explanation.

\begin{lemma}[Lemma 2 of \citealt{kennedy_sharp_2020}] \label{empiricalprocess_lemma} Let $\hat{f}(z)$ be a function estimated from a sample $(Z_{n+1}, \cdots, Z_N)$, and let $\hat{\cP}_n$ be the empirical measure over $(Z_1, \cdots, Z_n)$ (i.e., $\hat\cP_n = \frac{1}{n}\sum_{i=1}^n \delta_{Z_i}$), which is independent of $(Z_{n+1}, \cdots, Z_N)$. Then,
\begin{align*}
    \{\hat{\cP}_n - \cP\}(\hat{f} - f) = O_p\biggl(\frac{\norm{\hat{f} - f}_2}{\sqrt{n}}\biggr)
\end{align*}
where $\norm{\cdot}_2$ is $L_2$ norm and $\cP(f) = \cP\{f(Z_i)\} = \int f(z)d\cP(z)$.
\end{lemma}

\begin{proof}
Let $\hat\cP_n$ be a probability distribution consistent with the estimated nuisance functions $\hat\pi$, $\hat\mu$, and $\hat\boldf$, $\Psi_t(\cP)$ be an estimator constructed from $n$ i.i.d. sample from $\cP$, $O_i$ be a specific observed data point indexed by $i$, and $\psi_t(O_i, \cP)$ be an efficient influence function of the observation $O_i$ under the distribution $\cP$.  

Now, consider the parametric submodel $\cP_{\epsilon} = (1 - \epsilon) \cP + \epsilon \hat{\cP}_n$. Recall that the von Mises expansion yields
\begin{align}
    \Psi_t(\hat{\mathcal{P}}_n) = \Psi_t(\mathcal{P}) + \frac{d \Psi(\mathcal{P}_\epsilon)}{d\epsilon}\biggl|_{\epsilon = 1} (0-1) + R(\hat{\mathcal{P}}_n, \mathcal{P}) \label{vonmises}
\end{align}
where $R(\hat{\mathcal{P}}_n, \mathcal{P})$ denotes the remainder term of the expansion. Then, by Equation (14) of \cite{hines_demystifying_2022},
\begin{align*}
    &\sqrt{n}\biggl(\Psi_t(\hat{\mathcal{P}}_n) - \Psi_t(\mathcal{P})\biggr) = \frac{1}{\sqrt{n}}\sum_{i = 1}^n \psi_t(O_i, \cP) - H_1 + H_2 - \sqrt{n}R(\hat{\mathcal{P}}_n, \mathcal{P})
\end{align*}
where
\begin{align*}
    H_1 = \frac{1}{\sqrt{n}}\sum_{i = 1}^n \psi_t(O_i, \hat{\cP}_n), \quad H_2 = \sqrt{n}\{\hat{\cP}_n - \cP\}(\hat{\psi}_t - \psi_t)
\end{align*}
and $\hat\psi_t = \psi_t(\hat\cP_n)$ is the influence function constructed from the empirical distribution $\hat\cP_n$. 

Now, I analyze each term. The term $H_1$ is called bias-term, and as I solve the estimating equation so that this term becomes zero, $H_1 = 0$ by construction.  The second term $H_2$ is called the empirical process term, and I use Lemma \ref{empiricalprocess_lemma} to deal with this term. To do this, I mush show that $\hat\psi_t$ is consistent for $\psi_t$ in $L_2$ norm. Now,
\begin{align*}
    H_2 &= \sqrt{n}\{\hat{\cP}_n - \cP\}(\hat{\psi}_t - \psi_t) =  \sqrt{n}\{\hat{\cP}_n - \cP\}\biggl(
    \frac{S_i}{\P(S_i = 1)} (H_{21} - H_{22}) + 
    \mathbbm{1}\{T_i = t\} H_{23}
    + H_{24}\biggr)
\end{align*}
where
\begin{align*}
    &H_{21} = \biggl\{\frac{\rho_t(\boldf( \bY_i, \bZ_i))}{\pi_t(\bZ_i)} 
    - 
    \frac{\hat\rho_t(\hat\boldf( \bY_i, \bZ_i)) }{ \hat\pi_t(\bZ_i) } \biggr\}\tilde L_i\\
    &H_{22} = \frac{\rho_t(\boldf( \bY_i, \bZ_i))}{\pi_t(\bZ_i)} \mu( \boldf(\bY_i, \bZ_i))
    - 
    \frac{ \hat\rho_t(\hat\boldf( \bY_i, \bZ_i)) }{\hat\pi_t(\bZ_i) } \hat\mu( \boldf(\bY_i, \bZ_i))\\
    &H_{23} = \frac{ \mu( \boldf(\bY_i, \bZ_i)) - \Bar{m}_t(\bZ_i) }{\pi_t(\bZ_i)} 
- \frac{ \hat\mu( \hat\boldf(\bY_i, \bZ_i)) - \hat{\Bar{m}}_t(\bZ_i) }{ \hat{\pi}_t(\bZ_i) } \\
    &H_{24} = \Bar{m}_t(\bZ_i) - \hat{\Bar{m}}_t(\bZ_i).
\end{align*}
Then, I analyze each term. For $H_{21}$,
\begin{align*}
    &\E \biggl [ |H_{21}|^2\biggr]^{\frac{1}{2}} =
    \E \biggl [ \biggl|
     \biggl\{\frac{\rho_t(\boldf( \bY_i, \bZ_i))}{\pi_t(\bZ_i)} 
    - 
    \frac{\hat\rho_t(\hat\boldf( \bY_i, \bZ_i)) }{ \hat\pi_t(\bZ_i) } \biggr\}\tilde  L_i
    \biggr|^2\biggr]^{\frac{1}{2}} \\
    &\leq \E \biggl [ \biggl|\biggl\{\frac{\rho_t(\boldf( \bY_i, \bZ_i))}{\pi_t(\bZ_i)} 
    - 
    \frac{\hat\rho_t(\hat\boldf( \bY_i, \bZ_i)) }{ \hat\pi_t(\bZ_i) }\biggr\}\biggr|^2\biggr]^{\frac{1}{2}} \E[|\tilde L_i|^2]^{\frac{1}{2}}\\
    &\leq c_1 \E \biggl [ \biggl|\frac{\rho_t(\boldf( \bY_i, \bZ_i))}{\pi_t(\bZ_i)} 
    - 
    \frac{\hat\rho_t(\hat\boldf( \bY_i, \bZ_i)) }{ \hat\pi_t(\bZ_i) } \biggr|^2\biggr]^{\frac{1}{2}}\\
    &=  c_1 \E \biggl [ \biggl|
    \frac{{\hat\pi_t(\bZ_i)}\rho_t(\boldf( \bY_i, \bZ_i)) - 
    \pi_t(\bZ_i)  {\hat\rho_t(\hat\boldf( \bY_i, \bZ_i))}
    }{\pi_t(\bZ_i){\hat\pi_t(\bZ_i)} }
    \biggr|^2\biggr]^{\frac{1}{2}}\\
    &\leq  
    c_1 \E \biggl [ \biggl| \frac{1}{\pi_t(\bZ_i){\hat\pi_t(\bZ_i)} } \biggr|^2\biggr]^\frac{1}{2} \E\biggl[ \biggl|  {\hat\pi_t(\bZ_i)}\rho_t(\boldf( \bY_i, \bZ_i)) - 
    \pi_t(\bZ_i)  {\hat\rho_t(\hat\boldf( \bY_i, \bZ_i))}
    \biggr|^2\biggr]^{\frac{1}{2}}
    \\
    &\leq  \frac{c_1}{c_2^2} \E \biggl [ \biggl|
    \rho_t(\boldf( \bY_i, \bZ_i))
    \{{\hat\pi_t(\bZ_i)} - 
    \pi_t(\bZ_i)
    \} + \pi_t(\bZ_i)  \{\rho_t(\boldf( \bY_i, \bZ_i))
    - {\hat\rho_t(\hat\boldf( \bY_i, \bZ_i))}\}
    \biggr|^2\biggr]^{\frac{1}{2}}\\
    &\leq \frac{c_1}{c_2^2} \biggl\{ 
    \E \biggl [ \biggl|
    \rho_t(\boldf( \bY_i, \bZ_i))
    \{ {\hat\pi_t(\bZ_i)} - 
    \pi_t(\bZ_i)
    \}\biggr|^2\biggr]^{\frac{1}{2}} + \E\biggl[ \biggl| 
    \pi_t(\bZ_i)  \{\rho_t(\boldf( \bY_i, \bZ_i))
    - {\hat\rho_t(\hat\boldf( \bY_i, \bZ_i))}\}
    \biggr|^2\biggr]^{\frac{1}{2}}
    \biggr\}\\
    &\leq \frac{c_1}{c_2^2}
    \E\biggl[|\rho_t(\boldf( \bY_i, \bZ_i))|^2\biggr]^{\frac{1}{2}}
    \biggl\{ 
    \E \biggl [ \biggl|
    \{{\hat\pi_t(\bZ_i)} - 
    \pi_t(\bZ_i)
    \}\biggr|^2\biggr]^{\frac{1}{2}}+ \E\biggl[|\pi_t(\bZ_i)|^2\biggr]^{\frac{1}{2}}
    \E\biggl[ \biggl| 
      \rho_t(\boldf( \bY_i, \bZ_i))
    - {\hat \rho_t(\hat \boldf( \bY_i, \bZ_i))}
    \biggr|^2\biggr]^{\frac{1}{2}}
    \biggr\}\\
    &\leq c_1 (1/c_2^2 + 1)\delta_n
\end{align*}
where the third, sixth, and ninth lines are by Cauchy-Schwartz inequality and the eighth line is by triangular inequality.

For $H_{22}$,
\begin{align*}
    &\E \biggl [ |H_{22}|^2\biggr]^{\frac{1}{2}}
    = 
    \E\biggl[ 
    \biggl| 
    \frac{\rho_t(\boldf( \bY_i, \bZ_i))}{\pi_t(\bZ_i)} \mu( \boldf(\bY_i, \bZ_i))
    - 
    \frac{{\hat\rho_t(\hat\boldf( \bY_i, \bZ_i))} }{{\hat\pi_t(\bZ_i)} } {\hat\mu( \hat\boldf(\bY_i, \bZ_i))}
    \biggr|^{2}
    \biggr]^\frac{1}{2}\\
    &= \E\biggl[ 
    \biggl| 
    \biggl(
    \frac{\rho_t(\boldf( \bY_i, \bZ_i))}{\pi_t(\bZ_i) 
    }
    - 
    \frac{{\hat\rho_t(\hat\boldf( \bY_i, \bZ_i))} }{{\hat\pi_t(\bZ_i)} 
    }
    \biggr)
    \mu( \boldf(\bY_i, \bZ_i))
    -
    \frac{{\hat\rho_t(\hat\boldf( \bY_i, \bZ_i))} }{{\hat\pi_t(\bZ_i)} }
    \biggl(
    \mu( \boldf(\bY_i, \bZ_i)) -
    {\hat\mu( \hat\boldf(\bY_i, \bZ_i))} \biggr)
    \biggr|^{2}
    \biggr]^\frac{1}{2}\\
    &\leq 
    \E\biggl[ 
    \biggl|
    \biggl(\frac{\rho_t(\boldf( \bY_i, \bZ_i))}{\pi_t(\bZ_i) 
    }
    - 
    \frac{{\hat\rho_t(\hat\boldf( \bY_i, \bZ_i))} }{{\hat\pi_t(\bZ_i)} 
    }
    \biggr)
    \mu( \boldf(\bY_i, \bZ_i)) \biggr|^{2}
    \biggr]^\frac{1}{2}
    +
    \E\biggl[ 
    \biggl|\frac{{\hat\rho_t(\hat\boldf( \bY_i, \bZ_i))} }{{\hat\pi_t(\bZ_i)} }
    \biggl(
    \mu( \boldf(\bY_i, \bZ_i)) -
    {\hat\mu( \hat\boldf(\bY_i, \bZ_i))} \biggr) \biggr|^{2}
    \biggr]^\frac{1}{2}\\
    &\leq 
    \E\biggl[ 
    \biggl|
    \frac{\rho_t(\boldf( \bY_i, \bZ_i))}{\pi_t(\bZ_i) 
    }
    - 
    \frac{{\hat\rho_t(\hat\boldf( \bY_i, \bZ_i))} }{{\hat\pi_t(\bZ_i)} 
    } \biggr|^{2}
    \biggr]^\frac{1}{2}
    \E\biggl[|\mu( \boldf(\bY_i, \bZ_i))|^2\biggr]^\frac{1}{2}\\
    &\quad\quad\quad\quad\quad\quad\quad +
    \E\biggl[ 
    \biggl|\frac{{\hat\rho_t(\hat\boldf( \bY_i, \bZ_i))} }{{\hat\pi_t(\bZ_i)} } \biggr|^{2}
    \biggr]^\frac{1}{2}
    \E\biggl[ 
    \biggl|
    \mu( \boldf(\bY_i, \bZ_i)) -
    {\hat\mu( \hat\boldf(\bY_i, \bZ_i))} \biggr|^{2}
    \biggr]^\frac{1}{2}\\
    &\leq (1/c_2^2 + 1)\delta_n \cdot \E\biggl[|\mu( \boldf(\bY_i, \bZ_i))|^2\biggr]^\frac{1}{2} + \E\biggl[ 
    \biggl|\frac{{\hat\rho_t(\hat\boldf( \bY_i, \bZ_i))} }{{\hat\pi_t(\bZ_i)} } \biggr|^{2}
    \biggr]^\frac{1}{2} \cdot c_1\delta_n\\
    &\leq c_1(1 + 1/c_2^2 + 1/c_2)\delta_n
\end{align*}
where the third line is by triangular inequality, the fourth line is by Cauchy-Schwartz inequality, and the last line is by $\E[|\tilde  L_i|^2]^{\frac{1}{2}}$ (Assumption \ref{reg}).

For $H_{23}$,
\begin{align*}
    &\E \biggl [ |H_{23}|^2\biggr]^{\frac{1}{2}}
    = 
    \E\biggl[ 
    \biggl|
    \frac{ \mu( \boldf(\bY_i, \bZ_i)) - \Bar{m}_t(\bZ_i) }{\pi_t(\bZ_i)} 
- \frac{ {\hat\mu( \hat\boldf(\bY_i, \bZ_i))} - {\hat{\Bar{m}}_t(\bZ_i)} }{ {\hat\pi_t(\bZ_i)}}
    \biggr|^2
    \biggr]^\frac{1}{2}\\
    &= \E\biggl[ 
    \biggl|
    \frac{ \{ \mu( \boldf(\bY_i, \bZ_i)) - \Bar{m}_t(\bZ_i)\}\hat\pi_t(\bZ_i)
    - \{ \hat \mu( \hat\boldf(\bY_i, \bZ_i)) - \hat{\Bar{m}}_t(\bZ_i)\}\pi_t(\bZ_i)
    }{\pi_t(\bZ_i))\hat \pi_t(\bZ_i)}
    \biggr|^2
    \biggr]^\frac{1}{2}\\
    &= \E\biggl[ 
    \biggl|
    \frac{ \{ \mu( \boldf(\bY_i, \bZ_i)) - \Bar{m}_t(\bZ_i)\} \{\hat\pi_t(\bZ_i) - \pi_t(\bZ_i)\}
    }{\pi_t(\bZ_i))\hat \pi_t(\bZ_i)}\\
    &\quad\quad\quad\quad\quad\quad\quad - \frac{\pi_t(\bZ_i) \{\hat  \mu( \hat\boldf(\bY_i, \bZ_i)) -  \mu( \boldf(\bY_i, \bZ_i))\} + \pi_t(\bZ_i)\{ \hat{\Bar{m}}_t(\bZ_i) - \Bar{m}_t(\bZ_i)}{\pi_t(\bZ_i))\hat \pi_t(\bZ_i)}
    \biggr|^2
    \biggr]^\frac{1}{2}\\
    &\leq \frac{1}{c_2^2} \biggl\{ \E\biggl[ \biggl| \{ \mu( \boldf(\bY_i, \bZ_i)) - \Bar{m}_t(\bZ_i)\} \{\hat\pi_t(\bZ_i) - \pi_t(\bZ_i)
    \}  \\
    &\quad\quad\quad\quad\quad\quad\quad +\pi_t(\bZ_i) \{\hat  \mu( \hat\boldf(\bY_i, \bZ_i)) -  \mu( \boldf(\bY_i, \bZ_i))\} + \pi_t(\bZ_i)\{ \hat{\Bar{m}}_t(\bZ_i) - \Bar{m}_t(\bZ_i) 
    \biggr|^2 \biggr]^\frac{1}{2}
    \biggr\}\\
    &\leq \frac{1}{c_2^2} \biggl\{ \E\biggl[ \biggl| \{ \mu( \boldf(\bY_i, \bZ_i)) - \Bar{m}_t(\bZ_i)\} \{\hat\pi_t(\bZ_i) - \pi_t(\bZ_i)\} \biggr|^2 \biggr]^\frac{1}{2}  \\
    &\quad\quad\quad\quad\quad\quad\quad +\
    \E\biggl[ \biggl| \pi_t(\bZ_i) \{\hat  \mu( \hat\boldf(\bY_i, \bZ_i)) -  \mu( \boldf(\bY_i, \bZ_i))\} \biggr|^2 \biggr]^\frac{1}{2}
    + \E\biggl[ \biggl| \pi_t(\bZ_i)\{ \hat{\Bar{m}}_t(\bZ_i) - \Bar{m}_t(\bZ_i)\} \biggr|^2 \biggr]^\frac{1}{2}
    \biggr\}\\
    &\leq \frac{1}{c_2^2} \biggl\{ \E\biggl[ \biggl| \{ \mu( \boldf(\bY_i, \bZ_i)) - \Bar{m}_t(\bZ_i)\} \biggr|^2 \biggr]^\frac{1}{2}  \E\biggl[ \biggl| \{\hat\pi_t(\bZ_i) - \pi_t(\bZ_i)\} \biggr|^2 \biggr]^\frac{1}{2}  \\
    &\quad\quad\quad\quad + \E[|\pi_t(\bZ_i)|^2]^\frac{1}{2} \
    \E\biggl[ \biggl|  \{\hat  \mu( \hat\boldf(\bY_i, \bZ_i)) -  \mu( \boldf(\bY_i, \bZ_i))\} \biggr|^2 \biggr]^\frac{1}{2}
    + \E[|\pi_t(\bZ_i)|^2]^\frac{1}{2} \E\biggl[ \biggl|\{ \hat{\Bar{m}}_t(\bZ_i) - \Bar{m}_t(\bZ_i)\} \biggr|^2 \biggr]^\frac{1}{2}
    \biggr\}\\
    &\leq \frac{1}{c_2^2} \biggl\{ \E\biggl[ \biggl| \{ \mu( \boldf(\bY_i, \bZ_i)) - \Bar{m}_t(\bZ_i)\} \biggr|^2 \biggr]^\frac{1}{2} \delta_n + \delta_n + \delta_n \biggr\}\\
    &= \frac{1}{c_2^2}\delta_n \biggl\{ \E\biggl[ \biggl| \{ \E[\tilde L_i \mid \boldf(\bY_i, \bZ_i))] - 
    \E[\E[\tilde L_i \mid \boldf(\bY_i, \bZ_i))] \mid \bZ_i, T_i = t]
    \} \biggr|^2 \biggr]^\frac{1}{2} + 2 \biggr\}\\
    &\leq \frac{1}{c_2^2}\delta_n \biggl\{ \E\biggl[ \biggl| \{ \E[\tilde L_i \mid \boldf(\bY_i, \bZ_i)] 
    \} \biggr|^2 \biggr]^\frac{1}{2} + \E\biggl[ \biggl| \{ \E[\E[\tilde L_i \mid \boldf(\bY_i, \bZ_i)] \mid \bZ_i, T_i = t]
    \} \biggr|^2 \biggr]^\frac{1}{2} + 2 \biggr\}\\
    &\leq \frac{1}{c_2^2}\delta_n \biggl\{ \E\biggl[ \{ \E[ |\tilde L_i|^2 \mid \boldf(\bY_i, \bZ_i)] 
    \} \biggr]^\frac{1}{2} + \E\biggl[ \{ \E[\E[|\tilde L_i |^2 \mid \boldf(\bY_i, \bZ_i)] \mid \bZ_i, T_i = t]
    \}  \biggr]^\frac{1}{2} + 2 \biggr\}\\
    &\leq \frac{1}{c_2^2}\delta_n \biggl\{ \E\biggl[ \{ \E[ |\tilde L_i|^2 \mid \boldf(\bY_i, \bZ_i), T_i = t] 
    \} \biggr]^\frac{1}{2} + \E\biggl[ \{ \E[\E[|\tilde L_i |^2 \mid \boldf(\bY_i, \bZ_i), T_i = t] \mid \bZ_i, T_i = t]
    \}  \biggr]^\frac{1}{2} + 2 \biggr\}\\
    &\leq \frac{1}{c_2^2}\delta_n \biggl\{ \E\biggl[ \{ \E[ |\tilde L_i|^2 \mid \boldf(\bY_i, \bZ_i), T_i = t] 
    \} \biggr]^\frac{1}{2} + \E\biggl[ \{ \E[|\tilde L_i |^2 \mid \bZ_i, T_i = t]
    \}  \biggr]^\frac{1}{2} + 2 \biggr\}\\
    &\leq \frac{1}{c_2^2}\delta_n \biggl\{ \E[ |\tilde L_i|^2]^\frac{1}{2} + \E[ |\tilde L_i|^2]^\frac{1}{2} + 2\biggr\}\\
    &= \frac{1}{c_2^2}(2c_1 + 2)\delta_n
\end{align*}
where the fourth and sixth inequalities are by Cauchy-Schwartz inequality, the fifth and ninth lines are by triangular inequality, and the tenth line is by Jensen's inequality.

Finally, I assume that $\E[|H_{24}|^2]^\frac{1}{2} \leq \delta_n$ in Assumption \ref{reg}. Therefore,
\begin{align*}
    \E\biggl[ |\hat{\psi}_t - \psi_t|^2 \biggr]^{\frac{1}{2}} &= 
    \E\biggl[ \biggl|\frac{S_i}{\P(S_i = 1)} (H_{21} - H_{22}) + 
    \mathbbm{1}\{T_i = t\} H_{23}
    + H_{24} \biggr|^2 \biggr]^{\frac{1}{2}}\\
    &\leq \E\biggl[ \biggl| \frac{S_i}{\P(S_i = 1) } (H_{21} - H_{22} ) \biggr|^2 \biggr]^{\frac{1}{2}} + \E\biggl[ \biggl|\mathbbm{1}\{T_i = t\} H_{23} \biggr|^2 \biggr]^{\frac{1}{2}}
    + \E\biggl[ |H_{24}|^2 \biggr]^{\frac{1}{2}}
    \\
    &\leq \frac{1}{\sqrt{\P(S_i = 1)}} \E\biggl[ \biggl| H_{21} - H_{22} \biggr|^2 \biggr]^{\frac{1}{2}} + \E\biggl[ | H_{23} |^2 \biggr]^{\frac{1}{2}}
    + \E\biggl[ |H_{24} |^2 \biggr]^{\frac{1}{2}}
    \\
    &\leq \frac{1}{\sqrt{\P(S_i = 1)}} \biggl( \E\biggl[|H_{21}|^2 \biggr]^\frac{1}{2} + \E\biggl[|H_{22}|^2 \biggr]^\frac{1}{2} \biggr) + \E\biggl[ | H_{23} |^2 \biggr]^{\frac{1}{2}}
    + \E\biggl[ |H_{24} |^2 \biggr]^{\frac{1}{2}}\\
    &\leq \delta_n'
\end{align*}
where $\delta_n'$ is another sequence of random variables that converges to 0 as $n$ goes to infinity.  In the above transformation, the second and the fourth lines are by triangular inequality and the third line is by Cauchy-Schwartz inequality. Thus, by applying this Lemma \ref{empiricalprocess_lemma},
\begin{align*}
    H_2 = \sqrt{n}\{\hat{\cP}_n - \cP\}(\hat{\psi}_t - \psi_t) = O_p\biggl(\norm{\hat{\psi}_t - \psi_t}_2\biggr) = O_p(o_p(1)) = o_p(1).
\end{align*}
Thus, to derive the asymptotic normality, I only need to control the reminder term of the von Mises expansion. Rewriting Equation \ref{vonmises}, I can write the reminder term as
\begin{align*}
    &R(\hat{\mathcal{P}}_n, \mathcal{P}) = - \Psi_t(\hat{\mathcal{P}}_n) + \Psi_t(\mathcal{P}) - \int \psi_t(O_i, \hat{\cP}_n) d \mathcal{P}(O_i)\\
    &= - \Psi_t(\hat{\mathcal{P}}_n) + \Psi_t(\mathcal{P}) - 
    \E_{\cP}\biggl[\frac{\mathbbm{1}\{S_i = 1\}}{\P(S_i = 1)} \cdot  \frac{
    \E_{\hat\cP_n}[\mathbbm{1}\{T_i = t\} \mid \hat{\boldf}(\bY_i, \bZ_i) ]
    }{\E_{\hat\cP_n}[\mathbbm{1}\{T_i = t\} \mid \bZ_i ]} \biggl(\tilde L_i - \E_{\hat{\cP}_n}[\tilde L_i \mid \hat{\boldf}(\bY_i, \bZ_i)] \biggr)\\
    &\quad\quad\quad\quad\quad\quad +
    \frac{\mathbbm{1}\{T_i = t\} }{\E_{\hat\cP_n}[\mathbbm{1}\{T_i = t\} \mid \bZ_i ]} \biggl( \E_{\hat{\cP}_n}[\tilde L_i \mid \hat{\boldf}(\bY_i, \bZ_i)] - 
    \E_{\hat{\cP}_n}[\E_{\hat{\cP}_n}[\tilde L_i \mid \hat{\boldf}(\bY_i, \bZ_i)] \mid \bZ_i, T_i = t]
    \biggr)\\
    &\quad\quad\quad\quad\quad\quad\quad\quad\quad\quad\quad\quad+ \E_{\hat{\cP}_n}[\E_{\hat{\cP}_n}[ \tilde L_i \mid \hat{\boldf}(\bY_i, \bZ_i)] \mid \bZ_i, T_i = t] - \Psi_t(\hat\cP_n)\biggr]\\
    &=  - R_{1} - R_{2} - R_{3}
\end{align*}
where $\E_{\cP}[\cdot]$ is the expectation operator with respect to the population distribution $\cP$ and $\E_{\hat\cP_n}[\cdot]$ is the expectation operator with respect to the empirical distribution $\hat\cP_n$, and 
\begin{align*}
    R_1 &= \E_{\cP}\biggl[ \frac{\mathbbm{1}\{S_i = 1\}}{\P(S_i = 1)} \cdot  \frac{
    \E_{\hat\cP_n}[\mathbbm{1}\{T_i = t\} \mid \hat{\boldf}(\bY_i, \bZ_i) ]
    }{\E_{\hat\cP_n}[\mathbbm{1}\{T_i = t\} \mid \bZ_i ]} \biggl( \tilde L_i - \E_{\hat{\cP}_n}[ \tilde L_i \mid \hat{\boldf}(\bY_i, \bZ_i)] \biggr)\biggr],\\
    R_2 &= \E_{\cP}\biggl[\frac{\mathbbm{1}\{T_i = t\} }{\E_{\hat\cP_n}[\mathbbm{1}\{T_i = t\} \mid \bZ_i ]} \biggl( \E_{\hat{\cP}_n}[\tilde L_i \mid \hat{\boldf}(\bY_i, \bZ_i)] - 
    \E_{\hat{\cP}_n}[\E_{\hat{\cP}_n}[\tilde L_i \mid \hat{\boldf}(\bY_i, \bZ_i)] \mid \bZ_i, T_i = t]
    \biggr)\biggr],\\
    R_3 &= \E_\cP\biggl[\E_{\hat{\cP}_n}[\E_{\hat{\cP}_n}[ \tilde L_i \mid \hat{\boldf}(\bY_i, \bZ_i)] \mid \bZ_i, T_i = t] - \E_\cP [\E_{\cP}[\tilde L_i \mid \bZ_i, T_i = t]]\biggr].
\end{align*}

For $R_1$,
\begin{align*}
    R_1 &= \E_{\cP}\biggl[ \frac{\mathbbm{1}\{S_i = 1\}}{\P(S_i = 1)} \cdot  \frac{
    \E_{\hat\cP_n}[\mathbbm{1}\{T_i = t\} \mid \hat{\boldf}(\bY_i, \bZ_i) ]
    }{\E_{\hat\cP_n}[\mathbbm{1}\{T_i = t\} \mid \bZ_i ]} \biggl(\tilde L_i - \E_{\hat{\cP}_n}[\tilde L_i \mid \hat{\boldf}(\bY_i, \bZ_i)] \biggr)\biggr]\\
    &= \E_{\cP}\biggl[ \frac{
    \E_{\hat\cP_n}[\mathbbm{1}\{T_i = t\} \mid \hat{\boldf}(\bY_i, \bZ_i) ]
    }{\E_{\hat\cP_n}[\mathbbm{1}\{T_i = t\} \mid \bZ_i ]} \biggl(\tilde L_i - \E_{\hat{\cP}_n}[\tilde L_i \mid \hat{\boldf}(\bY_i, \bZ_i)] \biggr)\biggr]\\
    &= \E_{\cP}\biggl\{ \E_{\cP}\biggl[ \frac{
    \E_{\hat\cP_n}[\mathbbm{1}\{T_i = t\} \mid \hat{\boldf}(\bY_i, \bZ_i) ]
    }{\E_{\hat\cP_n}[\mathbbm{1}\{T_i = t\} \mid \bZ_i ]} \biggl(\tilde L_i - \E_{\hat{\cP}_n}[\tilde L_i \mid \hat{\boldf}(\bY_i, \bZ_i)] \biggr) \mid \bY_i, \bZ_i\biggr]\biggr\}\\
    &=  \E_{\cP}\biggl[ \frac{
    \E_{\hat\cP_n}[\mathbbm{1}\{T_i = t\} \mid \hat{\boldf}(\bY_i, \bZ_i) ]
    }{\E_{\hat\cP_n}[\mathbbm{1}\{T_i = t\} \mid \bZ_i ]} \biggl( \E_\cP[\tilde L_i \mid \bY_i, \bZ_i]- \E_{\hat{\cP}_n}[\tilde L_i \mid \hat{\boldf}(\bY_i, \bZ_i)] \biggr)\biggr]\\
    &=  \E_{\cP}\biggl[ \frac{
    \E_{\hat\cP_n}[\mathbbm{1}\{T_i = t\} \mid \hat{\boldf}(\bY_i, \bZ_i) ]
    }{\E_{\hat\cP_n}[\mathbbm{1}\{T_i = t\} \mid \bZ_i ]} \biggl( \E_\cP[\tilde L_i \mid \boldf(\bY_i, \bZ_i)]- \E_{\hat{\cP}_n}[ \tilde L_i \mid \hat{\boldf}(\bY_i, \bZ_i)] \biggr)\biggr]
\end{align*}
where the second line is by Assumption \ref{labeling}, the third equality is by the law of iterated expectation, and the last line is by the definition of surrogate representation (Definition \ref{def_surrep}).

For $R_2$,
\begin{align*}
    R_2 &=  \E_{\cP}\biggl[\frac{\mathbbm{1}\{T_i = t\} }{\E_{\hat\cP_n}[\mathbbm{1}\{T_i = t\} \mid \bZ_i ]} \biggl( \E_{\hat{\cP}_n}[\tilde L_i \mid \hat{\boldf}(\bY_i, \bZ_i)] - 
    \E_{\hat{\cP}_n}[\E_{\hat{\cP}_n}[\tilde L_i \mid \hat{\boldf}(\bY_i, \bZ_i)] \mid \bZ_i, T_i = t]
    \biggr)\biggr]\\
    &= \E_{\cP}\biggl[\frac{\mathbbm{1}\{T_i = t\} }{\E_{\hat\cP_n}[\mathbbm{1}\{T_i = t\} \mid \bZ_i ]} \biggl( \E_{\hat{\cP}_n}[\tilde L_i \mid \hat{\boldf}(\bY_i, \bZ_i)] -\E_{\cP}[\tilde L_i \mid \boldf(\bY_i, \bZ_i)]\biggr)\biggr]\\
    &\quad\quad + \E_{\cP}\biggl[\frac{\mathbbm{1}\{T_i = t\} }{\E_{\hat\cP_n}[\mathbbm{1}\{T_i = t\} \mid \bZ_i ]} \biggl( \E_{\cP}[\tilde L_i \mid \boldf(\bY_i, \bZ_i)]
    - 
    \E_{\hat{\cP}_n}[\E_{\hat{\cP}_n}[\tilde L_i \mid \hat{\boldf}(\bY_i, \bZ_i)] \mid \bZ_i, T_i = t] \biggr)\biggr]\\
    &= \E_{\cP}\biggl[\frac{\mathbbm{1}\{T_i = t\} }{\E_{\hat\cP_n}[\mathbbm{1}\{T_i = t\} \mid \bZ_i ]} \biggl( \E_{\hat{\cP}_n}[\tilde L_i \mid \hat{\boldf}(\bY_i, \bZ_i)] -\E_{\cP}[\tilde L_i \mid \boldf(\bY_i, \bZ_i)]\biggr)\biggr]\\
    &\quad\quad + \E_{\cP}\biggl\{ \E_{\cP}\biggl[\frac{\mathbbm{1}\{T_i = t\} }{\E_{\hat\cP_n}[\mathbbm{1}\{T_i = t\} \mid \bZ_i ]} \biggl( \E_{\cP}[\tilde L_i \mid \boldf(\bY_i, \bZ_i)]
    - 
     \E_{\hat{\cP}_n}[\E_{\hat{\cP}_n}[\tilde L_i \mid \hat{\boldf}(\bY_i, \bZ_i)] \mid \bZ_i, T_i = t] \biggr)\biggr] \mid T_i, \bZ_i \biggr\}\\
    &= \E_{\cP}\biggl[\frac{\mathbbm{1}\{T_i = t\} }{\E_{\hat\cP_n}[\mathbbm{1}\{T_i = t\} \mid \bZ_i ]} \biggl( \E_{\hat{\cP}_n}[\tilde L_i \mid \hat{\boldf}(\bY_i, \bZ_i)] -\E_{\cP}[\tilde L_i \mid \boldf(\bY_i, \bZ_i)]\biggr)\biggr]\\
    &\quad\quad + \E_{\cP}\biggl[\frac{\mathbbm{1}\{T_i = t\} }{\E_{\hat\cP_n}[\mathbbm{1}\{T_i = t\} \mid \bZ_i ]} \biggl( \E_{\cP}[ \E_{\cP}[\tilde L_i \mid \boldf(\bY_i, \bZ_i)] \mid \bZ_i, T_i = t]
    - \E_{\hat{\cP}_n}[\E_{\hat{\cP}_n}[\tilde L_i \mid \hat{\boldf}(\bY_i, \bZ_i)] \mid \bZ_i, T_i = t] \biggr)\biggr]
\end{align*}
where the third equality is by the law of iterated expectation.

For $R_3$,
\begin{align*}
    R_3 &= \E_\cP\biggl[\E_{\hat{\cP}_n}[\E_{\hat{\cP}_n}[\tilde L_i \mid \hat{\boldf}(\bY_i, \bZ_i)] \mid \bZ_i, T_i = t] - \E_\cP [\E_{\cP}[\tilde L_i \mid \bZ_i, T_i = t]]\biggr]\\
    &= \E_\cP\biggl[\E_{\hat{\cP}_n}[\E_{\hat{\cP}_n}[\tilde L_i \mid \hat{\boldf}(\bY_i, \bZ_i)] \mid \bZ_i, T_i = t] - \E_{\cP}[\E_{\cP}[ \tilde L_i \mid \boldf(\bY_i, \bZ_i), T_i = t] \mid \bZ_i, T_i = t]]\biggr]\\
    &= \E_\cP\biggl[\E_{\hat{\cP}_n}[\E_{\hat{\cP}_n}[\tilde L_i \mid \hat{\boldf}(\bY_i, \bZ_i)] \mid \bZ_i, T_i = t] - \E_{\cP}[\E_{\cP}[ \tilde L_i \mid \boldf(\bY_i, \bZ_i)] \mid \bZ_i, T_i = t]]\biggr]
\end{align*}
where the second line is by the law of iterated expectation and the third line is by the definition of surrogate representation (Definition \ref{def_surrep}).

Therefore, by combining all the terms, I obtain
\begin{align*}
    &R(\hat{\mathcal{P}}_n, \mathcal{P}) = -R_1 - R_2 - R_3\\
    &= - \E_{\cP}\biggl[ \biggl(\frac{
    \E_{\hat\cP_n}[\mathbbm{1}\{T_i = t\} \mid \hat{\boldf}(\bY_i, \bZ_i) ] - \mathbbm{1}\{T_i = t\}
    }{\E_{\hat\cP_n}[\mathbbm{1}\{T_i = t\} \mid \bZ_i ]}
    \biggr) \biggl( \E_\cP[\tilde L_i \mid \boldf(\bY_i, \bZ_i)]- \E_{\hat{\cP}_n}[\tilde L_i \mid \hat{\boldf}(\bY_i, \bZ_i)] \biggr)\biggr]\\
    &- 
    \E_{\cP}\biggl[\biggl(\frac{\mathbbm{1}\{T_i = t\} }{\E_{\hat\cP_n}[\mathbbm{1}\{T_i = t\} \mid \bZ_i ]} - 1\biggr) \biggl( \E_{\cP}[ \E_{\cP}[\tilde L_i \mid \boldf(\bY_i, \bZ_i)] \mid \bZ_i, T_i = t]
    - \E_{\hat{\cP}_n}[\E_{\hat{\cP}_n}[\tilde L_i \mid \hat{\boldf}(\bY_i, \bZ_i)] \mid \bZ_i, T_i = t] \biggr)\biggr]\\
    &= - \E_{\cP}\biggl[ \biggl(\frac{
    \E_{\hat\cP_n}[\mathbbm{1}\{T_i = t\} \mid \hat{\boldf}(\bY_i, \bZ_i) ] -  \E_{\cP}[\mathbbm{1}\{T_i = t\} \mid \boldf(\bY_i, \bZ_i) ] 
    }{\E_{\hat\cP_n}[\mathbbm{1}\{T_i = t\} \mid \bZ_i ]}
    \biggr) \biggl( \E_\cP[\tilde L_i \mid \boldf(\bY_i, \bZ_i)]- \E_{\hat{\cP}_n}[\tilde L_i \mid \hat{\boldf}(\bY_i, \bZ_i)] \biggr)\biggr]\\
    &- \E_{\cP}\biggl[\biggl(\frac{\E_{\cP}[\mathbbm{1}\{T_i = t\} \mid \bZ_i ] - \E_{\hat\cP_n}[\mathbbm{1}\{T_i = t\} \mid \bZ_i ] }{\E_{\hat\cP_n}[\mathbbm{1}\{T_i = t\} \mid \bZ_i ]}\biggr) \\
    &\quad\quad\quad\quad\quad\quad\quad\quad\quad \biggl( \E_{\cP}[ \E_{\cP}[\tilde L_i \mid \boldf(\bY_i, \bZ_i)] \mid \bZ_i, T_i = t]
    - \E_{\hat{\cP}_n}[\E_{\hat{\cP}_n}[\tilde L_i \mid \hat{\boldf}(\bY_i, \bZ_i)] \mid \bZ_i, T_i = t] \biggr)\biggr].
\end{align*}
Therefore,
\begin{align*}
    &|R(\hat{\mathcal{P}}_n, \mathcal{P})| \leq \frac{1}{\hat\pi(\bZ_i)} \biggl\{\E_\cP\biggl[  \biggl(\hat{\rho_t}(\hat{\boldf}(\bY_i, \bZ_i)) - \rho_t(\boldf(\bY_i, \bZ_i)) \biggr)^2
    \biggr]^{\frac{1}{2}}
    \E_\cP\biggl[
    \biggl(
    \mu(\boldf(\bY_i, \bZ_i))  - 
    \hat{\mu}(\hat{\boldf}(\bY_i, \bZ_i))
    \biggr)^2
    \biggr]^{\frac{1}{2}} \\
    &\quad\quad\quad\quad\quad\quad\quad\quad\quad\quad\quad\quad\quad\quad\quad + \E_\cP\biggl[  \biggl(\hat{\pi_t}(\bZ_i) - \pi_t(\bZ_i) \biggr)^2
    \biggr]^{\frac{1}{2}}
    \E_\cP\biggl[
    \biggl(
    \bar{m}_t(\bZ_i)  - 
    \hat{\bar{m}}_t(\bZ_i)
    \biggr)^2
    \biggr]^{\frac{1}{2}}\biggr\}\\
    &= o_p(n^{-1/2}),
\end{align*}
where the first inequality is by Cauchy-Schwartz inequality and the last equality is from the convergence rate of the cross-product of the nuisance functions in Assumption \ref{reg}. Now, I can also show $\V[\psi_t(O_i, \cP)] = \E\biggl[\psi_t(O_i, \cP)^2 \biggr] < \infty$, as each $H_{21}, H_{22}, H_{23}, H_{24}$ is bounded by constant (the proof is immediate since I only need to replace $\delta_n$ with constant $c_1$ in the earlier proof). Therefore, by applying the central limit theorem, I finally obtain the asymptotic normality:
\begin{align*}
    \sqrt{n}\biggl(\hat{\Psi_t} - \Psi_t \biggr) &= 
    \frac{1}{\sqrt{n}}\sum_{i = 1}^n \psi_t(O_i, \cP)  + o_p(1) \xrightarrow[]{d} \mathcal{N}\biggl(0, \E\biggl[\psi_t(O_i, \cP)^2 \biggr]\biggr).
\end{align*}
\end{proof}

\newpage
\subsection{Proof of Proposition \ref{identification2} (Identification with Noisy Human Annotations)} \label{proof_identification2}

\begin{proof}
Here, I prove the case with $J = 2$. Firstly, I prove the required independence for the identification. Notice that
\begin{align*}
    \P(\tilde L_i^{(j)} = l, \bZ_i = \bz, \mid L_i) &= 
    \int_{\cY} \P(\tilde L_i^{(j)} = l, \bZ_i = \bz \mid L_i, \bY_i) dF(\bY_i \mid L_i)\\
    &=\int_{\cY}  \P(\tilde L_i^{(j)} = l, \bZ_i = \bz \mid \bY_i ) dF(\bY_i  \mid L_i)\\
    &= \int_{\cY} \P(\tilde L_i^{(j)} = l \mid \bY_i) \P(\bZ_i = \bz \mid \bY_i) dF(\bY_i \mid L_i)\\
    &=\int_{\cY} \P(\tilde L_i^{(j)} = l \mid L_i) \P(\bZ_i = \bz \mid \bY_i) dF(\bY_i \mid L_i)\\
    &=  \P(\tilde L_i^{(j)} = l \mid L_i) \int_{\cY} \P(\bZ_i = \bz \mid \bY_i) dF(\bY_i \mid L_i)\\
    &=  \P(\tilde L_i^{(j)} = l \mid L_i) \int_{\cY}  \P(\bZ_i = \bz \mid \bY_i, L_i) dF(\bY_i \mid L_i)\\
    &= \P(\tilde L_i^{(j)} = l \mid L_i) \P(\bZ_i = \bz \mid L_i)
\end{align*}
where the second and sixth equalities are by Assumption \ref{coding}, the third equality is by Assumption \ref{annotation}, and the fourth equality is by Assumption \ref{coding} and \ref{proximal} ($\tilde L_i \ \indep \ \bY_i \mid L_i$).
Therefore, I have
\begin{align}
    \tilde L_i^{(j)} \ \indep \ \bZ_i \mid L_i. \label{tildeL_Z_given_L}
\end{align}
Based on this, I have
\begin{align*}
    \P(\tilde L_i = l \mid \bY_i, L_i, \bZ_i) = \P(\tilde L_i = l \mid \bY_i, L_i) = \P(\tilde L_i = l \mid L_i) = \P(\tilde L_i = l \mid L_i, \bZ_i)
\end{align*}
where the first equality is by Assumption \ref{annotation}, the second is by Assumption \ref{proximal} (exclusion restriction), and the last equality is by Equation~\eqref{tildeL_Z_given_L}. Thus, I have
\begin{align}
    \tilde L_i^{(j)} \ \indep \ \bY_i \mid L_i, \bZ_i. \label{tildeL_Y_given_LZ}
\end{align}
Also,
\begin{align*}
    \P(\tilde L_i^{(j)} = l, T_i = t \mid L_i, \bZ_i) &= 
    \int_{\cY} \P(\tilde L_i^{(j)} = l, T_i = t \mid L_i, \bZ_i, \bY_i) dF(\bY_i \mid L_i, \bZ_i)\\
    &= \int_{\cY} \P(\tilde L_i^{(j)} = l, T_i = t \mid \bZ_i, \bY_i) dF(\bY_i \mid L_i, \bZ_i)\\
    &= \int_{\cY} \P(\tilde L_i^{(j)} = l \mid \bZ_i, \bY_i)
    \P(T_i = t \mid \bZ_i, \bY_i)
    dF(\bY_i \mid L_i, \bZ_i)\\
    &= \int_{\cY} \P(\tilde L_i^{(j)} = l \mid \bZ_i, L_i, \bY_i)
    \P(T_i = t \mid \bZ_i, \bY_i)
    dF(\bY_i \mid L_i, \bZ_i)\\
    &= \int_{\cY} \P(\tilde L_i^{(j)} = l \mid \bZ_i, L_i)
    \P(T_i = t \mid \bZ_i, \bY_i)
    dF(\bY_i \mid L_i, \bZ_i)\\
    &= \P(\tilde L_i^{(j)} = l \mid \bZ_i, L_i) \int_{\cY}
    \P(T_i = t \mid \bZ_i, \bY_i)
    dF(\bY_i \mid L_i, \bZ_i)\\
    &= \P(\tilde L_i^{(j)} = l \mid \bZ_i, L_i) \int_{\cY} 
    \P(T_i = t \mid L_i, \bZ_i, \bY_i)
    dF(\bY_i \mid L_i, \bZ_i)\\
    &= \P(\tilde L_i^{(j)} = l \mid \bZ_i, L_i) \P(T_i = t \mid L_i, \bZ_i)
\end{align*}
where the second and fourth equalities are by Assumption \ref{coding}, the third equality is by Assumption \ref{annotation}, and the fifth equality is by Equation~\eqref{tildeL_Y_given_LZ}. Therefore, I have
\begin{align}
    \tilde L_i^{(j)} \ \indep \ T_i \ \mid \ L_i, \bZ_i. \label{tildeL_t_given_L_Z}
\end{align}
Based on these independence conditions, I prove the identification. Now,
\begin{align*}
    &\mathbb{P}(\tilde L^{(1)}_i = l, \tilde L^{(2)}_i = m, T_i = t \mid \bZ_i, S_i = 1)\\ &= 
    \sum_{c = 0}^C
    \mathbb{P}(\tilde L^{(1)}_i = l, \tilde L^{(2)}_i = m, T_i = t\mid L_i = c,  \bZ_i , S_i = 1) \ \P(L_i = c \mid \bZ_i , S_i = 1) \\
    &= \sum_{c = 0}^C
    \mathbb{P}(\tilde L^{(1)}_i = l, \tilde L^{(2)}_i = m, T_i = t\mid L_i = c,  \bZ_i ) \ \P(L_i = c \mid \bZ_i )\\
    &= \sum_{c = 0}^C
    \mathbb{P}(\tilde L^{(1)}_i = l, \tilde L^{(2)}_i = m \mid L_i = c,  \bZ_i ) \ \P(T_i = t \mid  L_i = c,  \bZ_i )  \ \P(L_i = c \mid \bZ_i )\\
    &= \sum_{c = 0}^C
    \mathbb{P}(\tilde L^{(1)}_i = l, \tilde L^{(2)}_i = m \mid L_i = c) \ \P(T_i = t \mid  L_i = c,  \bZ_i )  \ \P(L_i = c \mid \bZ_i )\\
    &= \sum_{c = 0}^C
    \mathbb{P}(\tilde L^{(1)}_i = l \mid L_i = c) \ 
    \P(T_i = t \mid  L_i = c,  \bZ_i )  \ \P(L_i = c \mid \bZ_i ) \
    \mathbb{P}(\tilde L^{(2)}_i = m \mid L_i = c)
\end{align*}
where the second equality is by Assumptions \ref{labeling} and \ref{labeling2}, the third equality is by Equation~\eqref{tildeL_t_given_L_Z}, the fourth equality is by Equation~\eqref{tildeL_Z_given_L}, and the last equality is by Assumption \ref{proximal} (independent coding).

I can also have another decomposition as follows:
\begin{align*}
    &\mathbb{P}(\tilde L^{(1)}_i = l, \tilde L^{(2)}_i = m \mid \bZ_i , S_i = 1)\\
    &= \sum_{c = 0}^C \mathbb{P}(\tilde L^{(1)}_i = l, \tilde L^{(2)}_i = m \mid L_i = c, \bZ_i , S_i = 1) \ \P(L_i = c \mid \bZ_i , S_i = 1)\\
    &= \sum_{c = 0}^C \mathbb{P}(\tilde L^{(1)}_i = l, \tilde L^{(2)}_i = m \mid L_i = c, \bZ_i ) \ \P(L_i = c \mid \bZ_i )\\
    &= \sum_{c = 0}^C \mathbb{P}(\tilde L^{(1)}_i = l, \tilde L^{(2)}_i = m \mid L_i = c) \ \P(L_i = c \mid \bZ_i )\\
    &= \sum_{c = 0}^C \mathbb{P}(\tilde L^{(1)}_i = l \mid L_i = c) \ \P(L_i = c \mid \bZ_i ) \  \mathbb{P}(\tilde L^{(2)}_i = m \mid L_i = c)
\end{align*}
where the second equality is by Assumptions \ref{labeling} and \ref{labeling2} (i.e., random sampling of human annotations), the third equality is by Equation~\eqref{tildeL_Z_given_L}, and the last equality is by Assumption \ref{proximal} (independent coding).

Using the matrix notation, the factorization above can be written as
\begin{align*}
    &\bm{M}(t, \bm{Z}_i) = \bm{A}^{(1)} \bm{D}(t, \bm{Z}_i) \bm{D}(\bm{Z}_i) (\bm{A}^{(2)})^\tp = \bm{A}^{(1)} \bm{D}(\bm{Z}_i) \bm{D}(t, \bm{Z}_i)  (\bm{A}^{(2)})^\tp\\
    &\bm{M}(\bZ_i) = \bm{A}^{(1)}
    \bm{D}(\bm{Z}_i) (\bm{A}^{(2)})^\tp
\end{align*}
where
\begin{align*}
    \bm{D}(t, \bZ_i) &= \mathrm{diag}( [\mathbb{P}(T_i = t \mid L_i = 0, \bZ_i ), \cdots, \mathbb{P}(T_i = t \mid L_i = C, \bZ_i )]^\tp )\\
    \bm{D}(\bm{Z}_i) &= \mathrm{diag}( [\mathbb{P}(L_i = 0 \mid \bZ_i ), \cdots, \mathbb{P}(L_i = C \mid \bZ_i )]^\tp )\\ 
    \bm{A}^{(j)} &=
    \begin{bmatrix}
        \mathbb{P}(\tilde L^{(j)}_i = 0 \mid L_i = 0) & \cdots &  \mathbb{P}(\tilde L^{(j)}_i = 0 \mid L_i = C)\\
        \vdots & \ddots & \vdots \\
        \mathbb{P}(\tilde L^{(j)}_i = C \mid L_i = 0) & \cdots & \mathbb{P}(\tilde L^{(j)}_i = C \mid L_i = C)
    \end{bmatrix} \quad j = 1,2\\
    \bm{M}(t, \bm{Z}_i) &= 
    \begin{bmatrix}
        \mathbb{P}(\tilde L^{(1)}_i = 0, \tilde L^{(2)}_i = 0, T_i = t \mid \bZ_i , S_i = 1) & \cdots & \mathbb{P}(\tilde L^{(1)}_i = 0, \tilde L^{(2)}_i = C, T_i = t \mid \bZ_i , S_i = 1)\\
        \vdots & \ddots & \vdots \\
        \mathbb{P}(\tilde L^{(1)}_i = C, \tilde L^{(2)}_i = 0, T_i = t \mid \bZ_i , S_i = 1) & \cdots & \mathbb{P}(\tilde L^{(1)}_i = C, \tilde L^{(2)}_i = C, T_i = t \mid \bZ_i , S_i = 1)
    \end{bmatrix}\\
    \bm{M}(\bm{Z}_i) &= 
    \begin{bmatrix}
        \mathbb{P}(\tilde L^{(1)}_i = 0, \tilde L^{(2)}_i = 0 \mid \bZ_i , S_i = 1) & \cdots & \mathbb{P}(\tilde L^{(1)}_i = 0, \tilde L^{(2)}_i = C \mid \bZ_i , S_i = 1)\\
        \vdots & \ddots & \vdots \\
        \mathbb{P}(\tilde L^{(1)}_i = C, \tilde L^{(2)}_i = 0 \mid \bZ_i , S_i = 1) & \cdots & \mathbb{P}(\tilde L^{(1)}_i = C, \tilde L^{(2)}_i = C \mid \bZ_i , S_i = 1)
    \end{bmatrix}.
\end{align*}
Now, notice that by Assumption \ref{monotonicity}, I have
\begin{align*}
    \mathbb{P}(\tilde L_i^{(j)} = l \mid L_i = l) &> \sum_{l\neq l'} \mathbb{P}(\tilde L_i^{(j)} = l' \mid  L_i = l) \\
    \P(\tilde L^{(j)} = \tilde L^{(j')} = l, T_i = t\mid \bZ_i , S_i = 1) &> \sum_{l\neq l'}\P(\tilde L^{(j)} = l', \tilde L^{(j')} = l, T_i = t\mid \bZ_i , S_i = 1)
\end{align*}
and by marginalizing over $T_i = t$,
\begin{align*}
    \P(\tilde L^{(j)} = \tilde L^{(j')} = l\mid \bZ_i , S_i = 1) > \sum_{l\neq l'}\P(\tilde L^{(j)} = l', \tilde L^{(j')} = l\mid \bZ_i , S_i = 1).
\end{align*}
These indicate the strong diagonal dominance, so $\bm{A}^{(j)}$, $\bm{M}(t, \bZ_i)$, and $\bm{M}(\bZ_i)$ are all invertible. Thus, I can write
\begin{align*}
    &\bm{M}(t, \bm{Z}_i) \bm{M}(\bm{Z}_i)^{-1} = \bm{A}^{(1)} \bm{D}(t, \bm{Z}_i) (\bm{A}^{(1)})^{-1}\\
    &\bm{M}(\bm{Z}_i)^\tp \bm{M}(t, \bm{Z}_i)^{-\tp} = \bm{A}^{(2)} \bm{D}(t)  (\bm{A}^{(2)})^{-1}
\end{align*}
where $X^{-\tp}$ denotes the transpose and inverse of the matrix $X$. Hence, I can use the standard eigenvalue decomposition to identify $\bm{A}^{(j)}$ and $\bm{D}(t, \bm{Z}_i)$ / $\bm{D}(\bm{Z}_i)$ up to scale and permutation. In this case, however, for the scale, since the column of $\bm{A}^{(j)}$ sums up to 1, the scale is fixed. The only problem is the permutation, but under Assumption \ref{monotonicity}, it is guaranteed that the diagonal element of $\bm{A}^{(j)}$ is the largest and unique, which fixes the challenge of permutation. As a result, I can pinpoint $\mathbb{P}(\tilde L^{(1)}_i = l \mid L_i = c)$ and $\mathbb{P}(\tilde L^{(2)}_i = l \mid L_i = c)$. Notice that I can also have
\begin{align*}
    &\mathbb{P}(\tilde L^{(1)}_i = l, \tilde L^{(2)}_i = m \mid T_i = t, \bZ_i , S_i = 1)\\
    &= \sum_{c = 0}^C \mathbb{P}(\tilde L^{(1)}_i = l, \tilde L^{(2)}_i = m \mid L_i = c, T_i = t, \bZ_i , S_i = 1) \ \P(L_i = c \mid T_i = t, \bZ_i , S_i = 1)\\
    &= \sum_{c = 0}^C \mathbb{P}(\tilde L^{(1)}_i = l, \tilde L^{(2)}_i = m \mid L_i = c, T_i = t, \bZ_i ) \ \P(L_i = c \mid T_i = t, \bZ_i )\\
    &= \sum_{c = 0}^C \mathbb{P}(\tilde L^{(1)}_i = l, \tilde L^{(2)}_i = m \mid L_i = c) \ \P(L_i = c \mid T_i = t, \bZ_i )\\
    &= \sum_{c = 0}^C \mathbb{P}(\tilde L^{(1)}_i = l \mid L_i = c) \ \P(L_i = c \mid T_i = t, \bZ_i ) \  \P(\tilde L^{(2)}_i = m \mid L_i = c)
\end{align*}
where the second one is by Assumptions \ref{labeling} and \ref{labeling2} (i.e., random sampling of human annotations), the third one is by exclusion restriction and conditional ignorability, the last one is by independent coding assumption. Hence, once $\mathbb{P}(\tilde L^{(1)}_i = l \mid L_i = c)$ and $\mathbb{P}(\tilde L^{(2)}_i = l \mid L_i = c)$ are obtained, I can insert them into the decomposition
\begin{equation}
    \begin{aligned}
        \mathbb{P}(\tilde L^{(1)}_i = l, \tilde L^{(2)}_i &= m \mid T_i = t, \bZ_i , S_i = 1)\\
        &= \sum_{c = 0}^C \mathbb{P}(\tilde L^{(1)}_i = l \mid L_i = c) \; \mathbb{P}(L_i = c \mid T_i = t, \bZ_i )\; \mathbb{P}(\tilde L^{(2)}_i = m \mid L_i = c)
    \end{aligned}
\end{equation}
and obtain $\mathbb{P}(L_i = c \mid T_i = t, \bZ_i )$ and can obtain the quantity of interest, since
\begin{align*}
    \Psi_{t,c} = \int_{\cZ} \mathbb{P}(L_i = c \mid T_i = t, \bZ_i) dF(\bZ_i).
\end{align*} 
In matrix form, the factorization above is written as
\begin{align*}
    \bm{B}(t, \bm{Z}_i) = \bm{A}^{(1)} \bm{\theta}(t, \bm{Z}_i) \{ \bm{A}^{(2)} \}^\tp
\end{align*}
where
\begin{align*}
    \bm{B}(t, \bm{Z}_i) &= 
    \begin{bmatrix}
        \mathbb{P}(\tilde L^{(1)}_i = 0, \tilde L^{(2)}_i = 0 \mid  T_i = t, \bZ_i , S_i = 1) & \cdots & \mathbb{P}(\tilde L^{(1)}_i = 0, \tilde L^{(2)}_i = C \mid T_i = t, \bZ_i , S_i = 1)\\
        \vdots & \ddots & \vdots \\
        \mathbb{P}(\tilde L^{(1)}_i = C, \tilde L^{(2)}_i = 0 \mid T_i = t, \bZ_i , S_i = 1) & \cdots & \mathbb{P}(\tilde L^{(1)}_i = C, \tilde L^{(2)}_i = C \mid T_i = t, \bZ_i , S_i = 1)
    \end{bmatrix}\\
   \bm{\theta}(t, \bm{Z}_i) &= \mathrm{diag}( [\mathbb{P}(L_i = 0 \mid T_i = t, \bZ_i ), \cdots, \mathbb{P}(L_i = C \mid T_i = t, \bZ_i )]^\tp )
\end{align*}
Hence, I can identify $\bm{\theta}(t, \bm{Z}_i)$ by
\begin{align*}
    \bm{\theta}(t, \bm{Z}_i) = (\bm{A}^{(1)})^{-1} \bm{B}(t, \bm{Z}_i) ( \bm{A}^{(2)})^{-\tp}
\end{align*}
\end{proof}

\newpage
\subsection{Theoretical Properties of Surrogate Outcome}\label{surrogate_proof}
I start with proving Equations~\eqref{unbiased_surroagte}. Notice that
\begin{align*}
    &\E[M_c(\tilde L_i^{(1)}, \tilde L_i^{(2)}) \mid L_i]\\
    &= \biggl( \frac{\P(\tilde L^{(1)} = c \mid L_i) - \P(\tilde L^{(1)} = c \mid L_i \neq c) }{
    \P(\tilde L^{(1)} = c \mid L_i = c) - \P(\tilde L^{(1)} = c \mid L_i \neq c)
    }\biggl)
    \biggl( \frac{\P(\tilde L^{(2)} = c \mid L_i) - \P(\tilde L^{(2)} = c \mid L_i \neq c) }{
    \P(\tilde L^{(2)} = c \mid L_i = c) - \P(\tilde L^{(2)} = c \mid L_i \neq c)
    }\biggl)\\
    &= \mathbbm{1}\{L_i = c\} \times \mathbbm{1}\{L_i = c\} = \mathbbm{1}\{L_i = c\}
\end{align*}
where the second equality is because when $L_i \neq c$, each component becomes 0 (the numerator becomes 0), whereas when $L_i = c$ each component becomes 1, which means that each component can be represented as an indicator function. This means that $M_c(\cdot, \cdot)$ is an unbiased surrogate for the unobserved indicator function $\mathbbm{1}\{L_i = c\}$ given the true label. As $\tilde L^{(j)}_i$ is observed for all $j \in \{1,2\}$ and I can estimate the probability $\P(\tilde L^{(j)}_i \mid L_i)$ as seen in the proof of Proposition \ref{identification2}, $M_c(\tilde L_i^{(1)}, \tilde L_i^{(2)})$ is identifiable from the observed data. Also, notice that
\begin{align*}
    &\E[M_c(\tilde L_i^{(1)}, \tilde L_i^{(2)}) \mid L_i, \bY_i]\\
    &= \biggl( \frac{\P(\tilde L^{(1)} = c \mid L_i, \bY_i) - \P(\tilde L^{(1)} = c \mid L_i \neq c) }{
    \P(\tilde L^{(1)} = c \mid L_i = c) - \P(\tilde L^{(1)} = c \mid L_i \neq c)
    }\biggl)
    \biggl( \frac{\P(\tilde L^{(2)} = c \mid L_i, \bY_i) - \P(\tilde L^{(2)} = c \mid L_i \neq c) }{
    \P(\tilde L^{(2)} = c \mid L_i = c) - \P(\tilde L^{(2)} = c \mid L_i \neq c)
    }\biggl)\\
    &= \biggl( \frac{\P(\tilde L^{(1)} = c \mid L_i) - \P(\tilde L^{(1)} = c \mid L_i \neq c) }{
    \P(\tilde L^{(1)} = c \mid L_i = c) - \P(\tilde L^{(1)} = c \mid L_i \neq c)
    }\biggl)
    \biggl( \frac{\P(\tilde L^{(2)} = c \mid L_i) - \P(\tilde L^{(2)} = c \mid L_i \neq c) }{
    \P(\tilde L^{(2)} = c \mid L_i = c) - \P(\tilde L^{(2)} = c \mid L_i \neq c)
    }\biggl)\\
    &= \mathbbm{1}\{L_i = c\} \times \mathbbm{1}\{L_i = c\} = \mathbbm{1}\{L_i = c\}
\end{align*}
where the second equality is because $\tilde L^{(j)}_i \; \indep \; \bY_i \mid L_i$ (i.e., exclusion restriction). Hence, by exploiting the independence, I have
\begin{align*}
    \mathbbm{1}\{L_i = c\} &= \E[M_c(\tilde L_i^{(1)}, \tilde L_i^{(2)}) \mid L_i, \bY_i]\\
    &= \E[M_c(\tilde L_i^{(1)}, \tilde L_i^{(2)}) \mid \bY_i] \\
    &= \E[M_c(\tilde L_i^{(1)}, \tilde L_i^{(2)}) \mid \bY_i, \bZ_i] 
\end{align*}
where the second equality is by Assumption \ref{coding} and the third equality is by Assumption \ref{annotation}. Therefore,
\begin{align*}
    \tilde \mu_c(\boldf(\bY_i, \bZ_i)) &= \E[M_c(\tilde L_i^{(1)}, \tilde L_i^{(2)}) \mid \boldf(\bY_i, \bZ_i)]\\
    &= \E[ \E[M_c(\tilde L_i^{(1)}, \tilde L_i^{(2)}) \mid \boldf(\bY_i, \bZ_i), \bY_i, \bZ_i] \mid \boldf(\bY_i, \bZ_i)]\\
    &= \E[ \E[M_c(\tilde L_i^{(1)}, \tilde L_i^{(2)}) \mid \bY_i, \bZ_i] \mid \boldf(\bY_i, \bZ_i)]\\
    &= \E[ \mathbbm{1}\{L_i = c\} \mid \boldf(\bY_i, \bZ_i)]\\
    &= \mu_c(\boldf(\bY_i, \bZ_i))
\end{align*}
which proves the claim mentioned after Equation~\eqref{unbiased_surroagte}.

\newpage
\subsection{Proof of Theorem \ref{theorem_eif2} (Influence function without true labels)} \label{proof_eif2}

Following Theorem \ref{eif}, I first derive the influence function for $\Psi_{t,c}$ with the full data.
\begin{lemma}[Influence function]
The influence function for the distribution of the covariate-adjusted outcome under $T_i = t$ is given by
\begin{align*}
    &\psi_{t,c}^F(T_i, L_i, S_i, \bY_i, \bZ_i; \boldf, \mu_c, \pi_t, \rho_t, \Bar{m}_{tc}, \Psi_{t,c})\\
    &= \frac{\mathbbm{1}\{S_i = 1\}}{\P(S_i = 1)} \cdot  \frac{\rho_t(\boldf( \bY_i, \bZ_i))}{\pi_t(\bZ_i)} \biggl(\mathbbm{1}\{L_i = c\} - \mu_c( \boldf(\bY_i, \bZ_i))\biggr)\\
    &\qquad\qquad + 
    \frac{\mathbbm{1}\{T_i = t\} }{\pi_t(\bZ_i)} \biggl( \mu( \boldf(\bY_i, \bZ_i)) - \Bar{m}_t(\bZ_i) \biggr) + \Bar{m}_{tc}(\bZ_i) - \Psi_t.
\end{align*}
where $\mu_{c}(\boldf(\bY_i)) = \P[L_i = c \mid \boldf(\bY_i)]$.
\end{lemma}
\begin{proof}
    Following the proof of Theorem \ref{eif}, I use the point-mass contamination approach. Now, the identification formula is
    \begin{align*}
        \Psi_{t,c}(\cP) &= \E_{\cP} \biggl[\E_{\cP}[\mathbbm{1}\{L_i = c\} \mid T_i = t, \bZ_i] \biggr]\\
        &= \E_{\cP}\biggl[\E_{\cP}[\E_{\cP}[\mathbbm{1}\{L_i = c\} \mid \boldf(\bY_i, \bZ_i)] \mid T_i = t, \bZ_i] \biggr]\\
        &= \int_{\cZ} \int_{\cY} \E_{\cP}[\mathbbm{1}\{L_i = c\} \mid \boldf(\bY_i, \bZ_i)]  dF(\bY_i \mid T_i = t, \bZ_i) dF(\bZ_i).
    \end{align*}
    Hence, by replacing only the original outcome model with $\E_{\cP}[\mathbbm{1}\{L_i = c\} \mid \boldf(\bY_i, \bZ_i)]$ in the proof of Theorem \ref{eif}, I obtain the following:
    \begin{align*}
    &\psi_{t,c}^F(T_i, L_i, S_i, \bY_i, \bZ_i; \boldf, \mu_c, \pi_t, \rho_t, \Bar{m}_{tc}, \Psi_{t,c})\\
    &= \frac{\mathbbm{1}\{S_i = 1\}}{\P(S_i = 1)} \cdot \frac{\rho_t(\boldf(\bY_i))}{\pi_t(\bZ_i)} \biggl(\mathbbm{1}\{L_i = c\} - \mu_c(\boldf(\bY_i))\biggr)\\
    &\qquad \qquad + \frac{\mathbbm{1}\{T_i = t\} }{\pi_t(\bZ_i)} \biggl( \mu_c(\bm{f}(\bY_i)) - \Bar{m}_{tc}(\bZ_i) \biggr) + \Bar{m}_{tc}(\bZ_i) - \Psi_{t,c}.
\end{align*}
\end{proof}

\noindent Based on this result, I have a guess that if I replace the hidden outcome $L$ with some estimable but unbiased substitute, it works as an influence function. To formally prove this, I first derive the most important lemma that characterizes the relationship between the observed data influence function and the full data influence function.

\begin{lemma}[Characterization of the observed data influence function] \label{characterization}
Any observed influence function $Q$ and any full data influence function $Q^F$ must satisfy
\begin{equation}
    \begin{aligned}
        \E[Q \mid S_i] &= \E[Q^F \mid S_i]\\
    \E[Q \mid \bY_i, T_i, \bZ_i] &= \E[Q^F \mid \bY_i, T_i, \bZ_i]
    \end{aligned}
\end{equation}
\end{lemma}

\begin{proof}
This proof goes analogous to that of Lemma 2 in \cite{zhou2024causal}. Let $G_{S_i}(\cdot)$ be a coarsening operator with coarsening variable $S_i$, which just means that
\begin{align*}
    G_{S_i}(\{\tilde L^{(j)}_i\}_{j = 1}^J) =
    \begin{cases}
    \{\tilde L^{(1)}_i, \cdots, \tilde L^{(j)}_i\} \quad \text{ if } S_i = 1\\
    \emptyset \quad \text{otherwise}
    \end{cases}
\end{align*} 
and $Q = Q(S_i, T_i, \bY_i, \bZ_i, G_{S_i}(\tilde L^{(1)}_i, \tilde L^{(2)}_i))$ be any observed data influence function. Then, $Q$ must satisfy the following \citep{tsiatis_semiparametric_2006, ichimura_influence_2022}:
\begin{align*}
    \lim_{\epsilon \to 0} \frac{\Psi_{t,c}(\cP_\epsilon) - \Psi_{t,c}(\cP) }{\epsilon} = \E\biggl[ Q \; \mathbb{E}[U^F \mid S_i, \bZ_i,  T_i, \bY_i, G_{S_i}(\tilde L^{(1)}_i, \tilde L^{(2)}_i)] \biggr] = \E\biggl[ Q \ U^F\biggr]
\end{align*}
where $U^F$ is the full data oracle score (see Chapter 7 of \citealt{tsiatis_semiparametric_2006} for the theoretical details of the observed data score function). Similarly, for the full data influence function, denoted as $Q^F$, I also have
\begin{align*}
    \lim_{\epsilon \to 0} \frac{\Psi_{t,c}(\cP_\epsilon) - \Psi_{t,c}(\cP) }{\epsilon} = \E\biggl[ Q^F U^F\biggr].
\end{align*}
Now, by factorization, $U^F$ (i.e., the full data score I considered in the previous setting) is written as
\begin{align*}
    U^F = g_s(S_i) +  g_z(\bZ_i) +   g_t(T_i, \bZ_i) + g_y(\bY_i, T_i, \bZ_i).
\end{align*}
for any $g_s \in \Lambda_S^F$, $g_z \in \Lambda_Z^F$, $g_t \in \Lambda_T^F$, and $g_y \in \Lambda_Y^F$, where
\begin{align*}
    &\Lambda_{S}^F = \biggl\{ g_s(S_i) \in \mathcal{H}^F  : \mathbb{E}[g_s(S_i)] = 0 \biggr\}\\
    &\Lambda_{Z}^F = \biggl\{ g_z(\bZ_i) \in \mathcal{H}^F: \mathbb{E}[g_z(\bZ_i)] = 0 \biggr\}\\
    &\Lambda_{T}^F = \biggl\{ g_t(T_i, \bZ_i) \in \mathcal{H}^F: \mathbb{E}[g_t(T_i, \bZ_i) \mid \bZ_i] = 0 \biggr\}\\
    &\Lambda_{\bY}^F = \biggl\{ g_{y}(\bY_i, T_i, \bZ_i) \in \mathcal{H}^F: \mathbb{E}[g_{y}(\bY_i, T_i, \bZ_i) \mid T_i, \bZ_i] = 0 \biggr\},
\end{align*}
and $\mathcal{H}^F$ is the Hilbert space of the oracle data (the data I are working with in the previous section's setting). Hence,
\begin{align*}
    &\lim_{\epsilon \to 0} \frac{\Psi_{t,c}(\cP_\epsilon) - \Psi_{t,c}(\cP) }{\epsilon} \\
    &= 
    \begin{cases}
        \E\biggl[ \E[Q \mid S_i] g_s(S_i) + \E[Q \mid \bZ_i]g_z(\bZ_i) + \E[Q \mid T_i]\ g_t(T_i, \bZ_i) + \E[Q \mid T_i, \bY_i]\ g_y(\bY_i, T_i, \bZ_i) \biggr]\\
        \E\biggl[ \E[Q^F \mid S_i] g_s(S_i) + \E[Q^F \mid \bZ_i]g_z(\bZ_i)  + \E[Q^F \mid T_i]\ g_t(T_i, \bZ_i) + \E[Q^F \mid T_i, \bY_i]\ g_y(\bY_i, T_i, \bZ_i) \biggr].
    \end{cases}
\end{align*}
Thus, to align these two expressions for any $g_s \in \Lambda_S^F$, $g_s \in \Lambda_T^F$, and $g_y \in \Lambda_Y^F$, I must have
\begin{align*}
    \E[Q \mid S_i] &= \E[Q^F \mid S_i],\\
    \E[Q \mid \bZ_i] &= \E[Q^F \mid \bZ_i],\\
    \E[Q \mid T_i, \bZ_i] &= \E[Q^F \mid T_i, \bZ_i],\\
    \E[Q \mid \bY_i, T_i, \bZ_i] &= \E[Q^F \mid \bY_i, T_i, \bZ_i].
\end{align*}
Note that the second and third conditions are redundant since the fourth implies the second and the third by the law of iterated expectations.
\end{proof}

\noindent Using these derived results, I prove Theorem \ref{theorem_eif2}.

\begin{proof}
From Lemma \ref{characterization}, I know the relationship between the full data influence function and the observed data influence function. Now, let
\begin{align*}
    Q^F &= \psi_{t,c}^F(T_i, L_i, S_i, \bY_i, \bZ_i; \boldf, \mu_c, \pi, \Psi_{t,c})
\end{align*}
and check if the expression
\begin{align*}
    \tilde{\psi}_{t,c}(T_i, \tilde L^{(1)}_i, \tilde L^{(2)}_i, S_i, \bY_i, \bZ_i&; \boldf, \tilde \mu, \pi_t, M_c, \Bar{m}_{tc}, \Psi_{t,c})\\
    &= \frac{\mathbbm{1}\{S_i = 1\}}{\P(S_i = 1)} \cdot  \frac{\rho_t(\boldf( \bY_i, \bZ_i))}{\pi_t(\bZ_i)} \biggl(M_c(\tilde L^{(1)}_i, \tilde L_i^{(2)}) - \tilde\mu_c( \boldf(\bY_i, \bZ_i))\biggr)\\
    &\quad\quad\quad + 
    \frac{\mathbbm{1}\{T_i = t\} }{\pi_t(\bZ_i)} \biggl( \tilde\mu_c( \boldf(\bY_i, \bZ_i)) - \Bar{m}_{t,c}(\bZ_i) \biggr) + \Bar{m}_{t,c}(\bZ_i) - \Psi_{t,c}
\end{align*}
satisfies the condition in the previous lemma. Now, because outcome model with the surrogate outcome model (i.e., $\tilde \mu_c$) is mathematically equivalent to $\mu_c$,
\begin{align*}
    \tilde{\psi}_{t,c}(\cdot) - \psi_{t,c}^F(\cdot) = \frac{\mathbbm{1}\{S_i = 1\}}{\P(S_i = 1)} \cdot \frac{\rho_t(\boldf( \bY_i, \bZ_i))}{\pi_t(\bZ_i)} \biggl\{ M_c(L_i^{(1)}, L_i^{(2)}) - \mathbbm{1}\{L_i = c\}\biggr\}.
\end{align*}
Thus,
\begin{itemize}
    \item $\E[\tilde{\psi}_{t,c}(\cdot) - \psi_{t,c}^F(\cdot)\mid S_i] = 0$: Now,
    \begin{align*}
        &\E[\tilde{\psi}_{t,c}(\cdot) - \psi_{t,c}^F(\cdot)\mid S_i]\\
        &= \E \biggl[ 
        \frac{\mathbbm{1}\{S_i = 1\}}{\P(S_i = 1)} \cdot \frac{\rho_t(\boldf( \bY_i, \bZ_i))}{\pi_t(\bZ_i)} \biggl\{ M_c(L_i^{(1)}, L_i^{(2)}) - \mathbbm{1}\{L_i = c\}\biggr\}
        \mid S_i \biggr]\\
        &= \frac{\mathbbm{1}\{S_i = 1\}}{\P(S_i = 1)} \cdot \E \biggl[ 
         \frac{\rho_t(\boldf( \bY_i, \bZ_i))}{\pi_t(\bZ_i)} \biggl\{ M_c(L_i^{(1)}, L_i^{(2)}) - \mathbbm{1}\{L_i = c\}\biggr\} \biggr]\\
        &= \frac{\mathbbm{1}\{S_i = 1\}}{\P(S_i = 1)} \cdot \E \biggl[ \E\biggl\{
         \frac{\rho_t(\boldf( \bY_i, \bZ_i))}{\pi_t(\bZ_i)} \biggl\{ M_c(L_i^{(1)}, L_i^{(2)}) - \mathbbm{1}\{L_i = c\}\biggr\} \mid \bY_i, \bZ_i \biggr\} \biggr]\\
         &= \frac{\mathbbm{1}\{S_i = 1\}}{\P(S_i = 1)} \cdot \E \biggl[ \frac{\rho_t(\boldf( \bY_i, \bZ_i))}{\pi_t(\bZ_i)} \E\biggl\{
          \biggl\{ M_c(L_i^{(1)}, L_i^{(2)}) - \mathbbm{1}\{L_i = c\}\biggr\} \mid \bY_i, \bZ_i \biggr\} \biggr]\\
          &= \frac{\mathbbm{1}\{S_i = 1\}}{\P(S_i = 1)} \cdot \E \biggl[ \frac{\rho_t(\boldf( \bY_i, \bZ_i))}{\pi_t(\bZ_i)}
          \biggl\{ \E[M_c(L_i^{(1)}, L_i^{(2)}) \mid \bY_i, \bZ_i] - \E[\mathbbm{1}\{L_i = c\} \mid \bY_i, \bZ_i] \biggr\} \biggr]\\
          &= \frac{\mathbbm{1}\{S_i = 1\}}{\P(S_i = 1)} \cdot \E \biggl[ \frac{\rho_t(\boldf( \bY_i, \bZ_i))}{\pi_t(\bZ_i)}
          \biggl\{ \E[M_c(L_i^{(1)}, L_i^{(2)}) \mid \bY_i] - \E[\mathbbm{1}\{L_i = c\} \mid \bY_i, \bZ_i] \biggr\} \biggr]\\
          &= \frac{\mathbbm{1}\{S_i = 1\}}{\P(S_i = 1)} \cdot \E \biggl[ \frac{\rho_t(\boldf( \bY_i, \bZ_i))}{\pi_t(\bZ_i)} 
          \biggl\{ \mathbbm{1}\{L_i = c\} - \mathbbm{1}\{L_i = c\} \biggr\} \biggr] = 0
    \end{align*}
    where the third line is by the law of iterated expectation, 
    the sixth line is by Assumption \ref{annotation},
    and the last line is by the property of the surrogate outcomes and Assumption \ref{coding}.
    \item $\E[\tilde{\psi}_{t,c}(\cdot) - \psi_{t,c}^F(\cdot)\mid \bY_i, T_i, \bZ_i] = 0$: Similarly,
    \begin{align*}
        &\E[\tilde{\psi}_{t,c}(\cdot) - \psi_{t,c}^F(\cdot)\mid \bY_i, T_i, \bZ_i]\\
        &= \E\biggl[\frac{\mathbbm{1}\{S_i = 1\}}{\P(S_i = 1)} \cdot \frac{\rho_t(\boldf( \bY_i, \bZ_i))}{\pi_t(\bZ_i)} \biggl\{ M_c(L_i^{(1)}, L_i^{(2)}) - \mathbbm{1}\{L_i = c\}\biggr\}
        \mid \bY_i, T_i, \bZ_i \biggr]\\
        &= \frac{\rho_t(\boldf( \bY_i, \bZ_i))}{\pi_t(\bZ_i)} \E\biggl[\frac{\mathbbm{1}\{S_i = 1\}}{\P(S_i = 1)} \cdot \biggl\{ M_c(L_i^{(1)}, L_i^{(2)}) - \mathbbm{1}\{L_i = c\}\biggr\}
        \mid \bY_i, T_i, \bZ_i \biggr]\\
        &= \frac{\rho_t(\boldf( \bY_i, \bZ_i))}{\pi_t(\bZ_i)} \E\biggl[ M_c(L_i^{(1)}, L_i^{(2)}) - \mathbbm{1}\{L_i = c\}
        \mid \bY_i, T_i, \bZ_i \biggr]\\
        &= \frac{\rho_t(\boldf( \bY_i, \bZ_i))}{\pi_t(\bZ_i)} \E\biggl[ M_c(L_i^{(1)}, L_i^{(2)}) - \mathbbm{1}\{L_i = c\}
        \mid \bY_i\biggr] = 0
    \end{align*}
\end{itemize}
where the fourth line is by Assumptions \ref{labeling} and \ref{labeling2}, and the last line is the exact same logic as in the previous case.
\end{proof}

\newpage
\subsection{Proof of Theorem \ref{asympnormal2} (Asymptotic Normality)} \label{proof_asympnormal2}

For this proof, I assume the following regularity conditions.

\begin{assumption}[Regularity Conditions] 
\label{reg2} Let $c_1$ be a positive constant and $\delta_n$ be a sequence of positive constants approaching zero as the sample size $n$ increases.
\begin{enumerate}[label=(\alph*)]
    \item (Positivity) There exists a fixed constant $c_2 \in (0, 1/2]$ such that for all $t \in \{0,1\}$
    \begin{align*}
        \P( c_2 \leq \pi_t(\bZ_i) \leq 1 - c_2) = 1, \quad
        \P( c_2 \leq \hat \pi_t(\bZ_i) \leq 1 - c_2) = 1.
    \end{align*}
    In addition, the estimated surrogacy score is also almost surely bounded so that 
    \begin{align*}
        \P(\hat\rho_t(\hat \boldf(\bY_i, \bZ_i)) \leq 1) = 1
    \end{align*}
    for all $t \in \{0,1\}$.
    \item (Outcome Model Estimation) Outcome models with an estimated surrogate outcome satisfy
    \begin{align*}
        &\E\biggl[
        \biggl|\widehat{\E[\hat M_c(\tilde L_i^{(1)}, \tilde L_i^{(2)}) \mid \hat\boldf(\bY_i, \bZ_i)]} -  \E[\hat M_c(\tilde L_i^{(1)}, \tilde L_i^{(2)}) \mid \boldf(\bY_i, \bZ_i) ] \biggr|^2
        \biggr]^{\frac{1}{2}}
        \leq c_1\\
        &\E\biggl[
        \biggl|\widehat{\E[\hat M_c(\tilde L_i^{(1)}, \tilde L_i^{(2)}) \mid \hat \boldf(\bY_i, \bZ_i)]}  - \E[\hat M_c(\tilde L_i^{(1)}, \tilde L_i^{(2)}) \mid \boldf(\bY_i, \bZ_i)] \biggr|^2
        \biggr]^{\frac{1}{2}}
        \leq n^{-1/4}\delta_n
    \end{align*}
    for all $c \in \cL$.
    \item (Nuisance Function Estimation) Surrogate scores are consistently estimated so that
     \begin{align*}
        \E\biggl[
        \biggl|\hat{\rho}_t(\hat{\boldf}(\bY_i, \bZ_i)) - {\rho}_t({\boldf}(\bY_i, \bZ_i)) \biggr|^2
        \biggr]^{\frac{1}{2}} \leq c_1, &\quad  \E\biggl[
        \biggl|\hat{\rho}_t(\hat{\boldf}(\bY_i, \bZ_i)) - {\rho}_t({\boldf}(\bY_i, \bZ_i)) \biggr|^2
        \biggr]^{\frac{1}{2}} \leq n^{-1/4}\delta_n\\
        \E\biggl[
        \biggl|\hat{\pi}_t(\bZ_i) - \pi_t(\bZ_i) \biggr|^2
        \biggr]^{\frac{1}{2}} \leq c_1, &\quad \E\biggl[
        \biggl|\hat{\pi}_t(\bZ_i) - \pi_t(\bZ_i) \biggr|^2
        \biggr]^{\frac{1}{2}} \leq n^{-1/4}\delta_n\\
        \E\biggl[
        \biggl|\hat{\bar{m}}_{t,c}(\bZ_i) - \bar{m}_{t,c}(\bZ_i) \biggr|^2
        \biggr]^{\frac{1}{2}} \leq c_1, &\quad \E\biggl[
        \biggl|\hat{\bar{m}}_{t,c}(\bZ_i) - \bar{m}_{t,c}(\bZ_i) \biggr|^2
        \biggr]^{\frac{1}{2}}\leq n^{-1/4}\delta_n
    \end{align*}
    for all $t \in \{0,1\}$ and $c \in \cL$.
    \item (Coder Classification Rate Estimation) Human coder classification rates $\P(\tilde L_i^{(j)}  = c \mid L_i)$ are estimated so that
    \begin{align*}
        &\E\biggl[
        \biggl|\widehat{\P(\tilde L_i^{(j)}  = c \mid L_i) } - \P(\tilde L_i^{(j)}  = c \mid L_i) \biggr|^2
        \biggr]^{\frac{1}{2}}
        \leq c_1, \\
        &\E\biggl[
        \biggl|\widehat{\P(\tilde L_i^{(j)}  = c \mid L_i) } - \P(\tilde L_i^{(j)}  = c \mid L_i) \biggr|^2
        \biggr]^{\frac{1}{2}}
        \leq n^{-1/4}\delta_n
    \end{align*}
    for all coders $j$ and all category levels $c \in \cL$.
\end{enumerate}
\end{assumption}
As in Section \ref{sec::method}, the required rates are standard in the literature, and various machine learning models can satisfy them. For the outcome model, the required convergence is evaluated with respect to the predicted surrogate outcome $\hat M_c(\tilde L_i^{(1)}, \tilde L_i^{(2)})$. This implies that Assumption \ref{reg2}(b) focuses on convergence based on these predicted surrogate outcomes. Since the estimation process involves two steps and treats the predicted surrogate outcome as fixed in the second stage, the stated condition aligns with the proposed estimation procedure and can be satisfied using a standard neural network architecture \citep{bauer_deep_2019, schmidt-hieber_nonparametric_2020, farrell_deep_2021}.

\begin{proof}
While the influence function used in Theorem \ref{asympnormal2} takes the mathematically equivalent form to Theorem \ref{asymp_normal}, the main difference is that the outcome model is learned with the predicted surrogate outcome and the regularity conditions in Assumption \ref{reg2} only have the convergence of the outcome model given the predicted outcomes. Therefore, I need to bound the $L^2$ norm of the difference between the estimated outcome model with the predicted surrogate outcome and the true outcome model. Then, notice that 
\begin{align*}
    &\E\biggl[\biggl| \widehat{\E[\hat M_c(\tilde L_i^{(1)}, \tilde L_i^{(2)}) \mid \hat \boldf(\bY_i, \bZ_i)]} - \E[M_c(\tilde L_i^{(1)}, \tilde L_i^{(2)}) \mid \boldf(\bY_i, \bZ_i)] \biggr|^2\biggr]^{\frac{1}{2}}\\
    &= \E\biggl[\biggl| \widehat{\E[\hat M_c(\tilde L_i^{(1)}, \tilde L_i^{(2)}) \mid \hat \boldf(\bY_i, \bZ_i)]} -\E[\hat M_c(\tilde L_i^{(1)}, \tilde L_i^{(2)}) \mid \boldf(\bY_i, \bZ_i)]\\
    &\quad\quad\quad + \E[\hat M_c(\tilde L_i^{(1)}, \tilde L_i^{(2)}) - M_c(\tilde L_i^{(1)}, \tilde L_i^{(2)}) \mid \boldf(\bY_i, \bZ_i)]\biggr|^2\biggr]^{\frac{1}{2}}\\
    &\leq \E\biggl[\biggl| \widehat{\E[\hat M_c(\tilde L_i^{(1)}, \tilde L_i^{(2)}) \mid \hat \boldf(\bY_i, \bZ_i)]} -\E[\hat M_c(\tilde L_i^{(1)}, \tilde L_i^{(2)}) \mid \boldf(\bY_i, \bZ_i)] \biggr|^2\biggr]^{\frac{1}{2}}\\
    &\quad\quad\quad + \E\biggl[\biggl| \E[\hat M_c(\tilde L_i^{(1)}, \tilde L_i^{(2)}) - M_c(\tilde L_i^{(1)}, \tilde L_i^{(2)}) \mid \boldf(\bY_i, \bZ_i)]\biggr|^2\biggr]^{\frac{1}{2}}\\
    &= o_p(n^{-1/4}) + \E\biggl[\biggl| \E[\hat M_c(\tilde L_i^{(1)}, \tilde L_i^{(2)}) - M_c(\tilde L_i^{(1)}, \tilde L_i^{(2)}) \mid \boldf(\bY_i, \bZ_i)]\biggr|^2\biggr]^{\frac{1}{2}}
\end{align*}
where the inequality in the third line is by the triangular inequality and the fourth line is by Assumption \ref{reg2}. Thus, I only need to show that the second term is asymptotically negligible at the same rate. Then,
\begin{align*}
    &\E\biggl[\biggl| \E[\hat M_c(\tilde L_i^{(1)}, \tilde L_i^{(2)}) - M_c(\tilde L_i^{(1)}, \tilde L_i^{(2)}) \mid \boldf(\bY_i, \bZ_i)]\biggr|^2\biggr]^{\frac{1}{2}}\\
    &\leq 
    \E\biggl\{ \E\biggl[ \biggl|\hat M_c(\tilde L_i^{(1)}, \tilde L_i^{(2)}) - M_c(\tilde L_i^{(1)}, \tilde L_i^{(2)}) \biggr|^2 \mid \boldf(\bY_i, \bZ_i)\biggr]\biggr\}^{\frac{1}{2}}\\
    &= \E\biggl[ \biggl|\hat M_c(\tilde L_i^{(1)}, \tilde L_i^{(2)}) - M_c(\tilde L_i^{(1)}, \tilde L_i^{(2)})\biggr|^2 \biggr]^{\frac{1}{2}}
\end{align*}
where the first inequality is by Jensen's inequality and the second line is by the law of iterated expectation. I thus only need to bound the $L^2$ convergence of surrogate outcome. Let
\begin{align*}
    g(X, u,v) = \frac{X - u}{v - u}, \quad X \in \{0,1\},\; 0 \leq u < v \leq 1.
\end{align*}
As $u < v$, there always exists a constant $c_3 > 0$ such that $v - u \geq c_3$. Then, by fixing $X$, the quotient rule yields
\begin{align*}
\begin{cases}
    \partial_u g &= \frac{-(u-v) - (X-u)(-1)}{(u-v)^2} = \frac{-(u-v) + (X-u)}{(u-v)^2} = \frac{X-v}{(u-v)^2}\\
    \partial_v g &= \frac{-(X-u)}{(u-v)^2} =  \frac{X-u}{(u-v)^2}
\end{cases}
\end{align*}
Hence, for any $(u,v)$, I have
\begin{align*}
    |\partial_u g| = \frac{|X-v|}{(v-u)^2} \leq \frac{1}{(v-u)^2} \leq \frac{1}{c_3^2}, \quad |\partial_v g| = \frac{|X-u|}{(v-u)^2} \leq \frac{1}{(v-u)^2} \leq \frac{1}{c_3^2}
\end{align*}
Let $\gamma(t) = [u_1, v_1]^\tp + t \cdot[u_2, v_2]^\tp$ and set $h(t) = g(\gamma(t))$. By mean value theorem, there exists $t^* \in (0,1)$ such that
\begin{align*}
    g(X, u_2, v_2) - g(X, u_1, v_1) = h(1) - h(0) = h'(t^*)
\end{align*}
Notice that by chain rule, $h'(t) = \nabla g(\gamma(t)) \cdot ([u_2, v_2]^\tp - [u_1, v_1]^\tp)$. Hence, by triangular inequality,
\begin{equation}
    \begin{aligned}
        |g(X, u_2, v_2) - g(X, u_1, v_1)| &\leq |\partial_u g(\gamma(t^*))| \cdot |u_2 - u_1| + |\partial_v g(\gamma(t^*))| \cdot |v_2 - v_1| \\
    &\leq \frac{1}{c_3^2}\biggl(|u_1 - u_2| + |v_1 - v_2|\biggr). \label{meanvalue}
    \end{aligned}
\end{equation}
Also, notice that
\begin{align*}
    |X- u| = 
    \begin{cases}
        u \quad (x=0)\\
        1-u \quad (x = 1)
    \end{cases}
    \leq 1
\end{align*}
and hence
\begin{align}
    |g(X,u,v)| = \frac{|X-u|}{v-u} \leq \frac{1}{v-u} \leq \frac{1}{c_3}. \label{inequality_g}
\end{align}
Now, let
\begin{align*}
    p_c^{(j)} = \P(\tilde L^{(j)}_i = c \mid L_i \neq c), \quad q_c^{(j)} = \P(\tilde L^{(j)}_i = c \mid L_i = c), \quad X^{(j)}_c = \mathbbm{1}\{\tilde L^{(j)}_i = c\}
\end{align*}
for all $j$ and $k \in \{1, \cdots, K\}$, and denote $\hat{p}_c^{(j)} = \widehat{\P(\tilde L^{(j)}_i = c \mid L_i \neq c)}$ and $\hat{q}_c^{(j)} = \widehat{\P(\tilde L^{(j)}_i = c \mid L_i = c)}$ as the predicted quantities. Using the derived inequality, I obtain
\begin{align*}
    &\biggl| \hat M_c(\tilde L_i^{(1)}, \tilde L_i^{(2)}) - M_c(\tilde L_i^{(1)}, \tilde L_i^{(2)}) \biggr|\\
    &=  \biggl| g\bigl(X_c^{(1)}, \hat{p}_c^{(1)}, \hat{q}_c^{(1)}\bigr) g\bigl(X_c^{(2)}, \hat{p}_c^{(2)}, \hat{q}_c^{(2)}\bigr) - g\bigl(X_c^{(1)}, {p}_c^{(1)}, {q}_c^{(1)}\bigr) g\bigl(X_c^{(2)}, {p}_c^{(2)}, {q}_c^{(2)}\bigr)  \biggr|\\
    &=  \biggl| g\bigl(X_c^{(1)}, \hat{p}_c^{(1)}, \hat{q}_c^{(1)}\bigr) \biggl(
    g\bigl(X_c^{(2)}, \hat{p}_c^{(2)}, \hat{q}_c^{(2)}\bigr) - g\bigl(X_c^{(2)}, {p}_c^{(2)}, {q}_c^{(2)}\bigr) \biggr) \\
    &\quad\quad\quad\quad + g\bigl(X_c^{(2)}, {p}_c^{(2)}, {q}_c^{(2)}\bigr) \biggl(
    g\bigl(X_c^{(1)}, \hat{p}_c^{(1)}, \hat{q}_c^{(1)}\bigr) - g\bigl(X_c^{(1)}, {p}_c^{(1)}, {q}_c^{(1)}\bigr) \biggr) 
    \biggr|\\
    &\leq \biggl| g\bigl(X_c^{(1)}, \hat{p}_c^{(1)}, \hat{q}_c^{(1)}\bigr) \biggl(
    g\bigl(X_c^{(2)}, \hat{p}_c^{(2)}, \hat{q}_c^{(2)}\bigr) - g\bigl(X_c^{(2)}, {p}_c^{(2)}, {q}_c^{(2)}\bigr) \biggr)\biggr| \\
    &\quad\quad\quad\quad + \biggl| g\bigl(X_c^{(2)}, {p}_c^{(2)}, {q}_c^{(2)}\bigr) \biggl(
    g\bigl(X_c^{(1)}, \hat{p}_c^{(1)}, \hat{q}_c^{(1)}\bigr) - g\bigl(X_c^{(1)}, {p}_c^{(1)}, {q}_c^{(1)}\bigr) \biggr) 
    \biggr|\\
    &\leq |g\bigl(X_c^{(1)}, \hat{p}_c^{(1)}, \hat{q}_c^{(1)}\bigr) |\biggl| 
    g\bigl(X_c^{(2)}, \hat{p}_c^{(2)}, \hat{q}_c^{(2)}\bigr) - g\bigl(X_c^{(2)}, {p}_c^{(2)}, {q}_c^{(2)}\bigr)\biggr| \\
    &\quad\quad\quad\quad + |g\bigl(X_c^{(2)}, {p}_c^{(2)}, {q}_c^{(2)}\bigr)| \biggl| 
    g\bigl(X_c^{(1)}, \hat{p}_c^{(1)}, \hat{q}_c^{(1)}\bigr) - g\bigl(X_c^{(1)}, {p}_c^{(1)}, {q}_c^{(1)}\bigr) 
    \biggr|\\
    &\leq \frac{1}{c_3}\biggl( \biggl| 
    g\bigl(X_c^{(2)}, \hat{p}_c^{(2)}, \hat{q}_c^{(2)}\bigr) - g\bigl(X_c^{(2)}, {p}_c^{(2)}, {q}_c^{(2)}\bigr)\biggr| + \biggl| 
    g\bigl(X_c^{(1)}, \hat{p}_c^{(1)}, \hat{q}_c^{(1)}\bigr) - g\bigl(X_c^{(1)}, {p}_c^{(1)}, {q}_c^{(1)}\bigr) 
    \biggr| \biggr)\\
    &\leq \frac{1}{c_3^3}\biggl(|\hat{p}_c^{(1)} - {p}_c^{(1)}| + |\hat{q}_c^{(1)} - {q}_c^{(1)}| + |\hat{p}_c^{(2)} - {p}_c^{(2)}| + |\hat{q}_c^{(2)} - {q}_c^{(2)}| \biggr)
\end{align*}
where the first inequality is by triangular inequality, the second is by Cauchy-Schwartz inequality, the third is by Equation~\eqref{inequality_g}, and the last is by Equation~\eqref{meanvalue}. Now, by Cauchy-Schwartz,
\begin{align*}
    (a_1 + a_2 + a_3 + a_4)^2 = 
    \biggl(\underbrace{\begin{bmatrix}
        a_1 & a_2 & a_3 & a_4
    \end{bmatrix}^\tp}_{:= \bm{a}} \underbrace{\begin{bmatrix}
        1 & 1 & 1 & 1
    \end{bmatrix}}_{:= \mathbf{1}^\tp}
    \biggr)^2 \leq \norm{\bm{a}}^2 \norm{\mathbf{1}}^2 = 4 (a_1^2 + a_2^2 + a_3^2 + a_4^2)
\end{align*}
for any $a_1, a_2, a_3, a_4$. Applying this, I obtain
\begin{align*}
    \biggl| \hat M_c(\tilde L_i^{(1)}, \tilde L_i^{(2)}) - M_c(\tilde L_i^{(1)}, \tilde L_i^{(2)}) \biggr|^2 &\leq \frac{4}{c_3^6}\biggl(|\hat{p}_c^{(1)} - {p}_c^{(1)}|^2 + |\hat{q}_c^{(1)} - {q}_c^{(1)}|^2 + |\hat{p}_c^{(2)} - {p}_c^{(2)}|^2 + |\hat{q}_c^{(2)} - {q}_c^{(2)}|^2 \biggr).
\end{align*}
By taking the expectation, I finally get
\begin{align*}
    &\E \biggl[ \biggl| \hat M_c(\tilde L_i^{(1)}, \tilde L_i^{(2)}) - M_c(\tilde L_i^{(1)}, \tilde L_i^{(2)}) \biggr|^2 \biggr]\\
    &\leq \frac{4}{c_3^6}\biggl\{ \E\biggl[|\hat{p}_c^{(1)} - {p}_c^{(1)}|^2\biggr]  + \E\biggl[|\hat{q}_c^{(1)} - {q}_c^{(1)}|^2\biggr] + \E\biggl[|\hat{p}_c^{(2)} - {p}_c^{(2)}|^2\biggr] + \E\biggl[|\hat{q}_c^{(2)} - {q}_c^{(2)}|^2 \biggr] \biggr\}\\
    &= \frac{4}{c_3^6} \biggl(o_p(n^{-1/2}) + o_p(n^{-1/2}) + o_p(n^{-1/2}) + o_p(n^{-1/2})\biggr) = o_p(n^{-1/2})
\end{align*}
where the last line is by using Assumption \ref{reg2}. Hence, combined with the first inequality I derived, I prove that the outcome model with the estimated surrogate outcomes achieves $L^2$ convergence under all the stated conditions at the convergence rate $n^{-1/4}$, which in the end implies the rate double robustness I have in Assumption \ref{reg}. The exact same proof applies to the part that the $L^2$ norm is bounded by the constant $c_1$. As the efficient influence function takes mathematically equivalent expressions, the rest of the proof follows directly from Theorem \ref{asymp_normal}.   
\end{proof}

\newpage

\newpage 

\section{Simulation Studies}\label{sec::simulation}
\subsection{Case 1: Perfect Human Annotations}
\subsubsection{Simulation setup}
I conduct simulation studies to evaluate the empirical performance of the proposed SRI estimator and to compare it with existing estimators. Because it is extremely difficult to generate realistic text data that are systematically influenced by predictors in a synthetic experiment, I instead generate high-dimensional vectors with dimensions comparable to standard sentence embeddings and allow them to depend on predictors through the data-generating process. Specifically, I generate data for $n=20,000$ units according to the following process:
\begin{equation}
    \begin{aligned}
            &T_i \sim \mathrm{Bernoulli}(0.5)\\
            &\text{For } j \in [1, \cdots, 768]\\
            &\quad \alpha_{1j} \sim \mathrm{Unif}[0.0, 1.0]\\
            &\quad Y_{ij} = \alpha_{1j} T_i + \epsilon_{ij}, \quad \epsilon_{ij} \sim \mathcal{N}(0, 1)\\
            &\tilde  L_i = \mathbbm{1}\{\sigma(1.0 + 0.2\cdot  \mathbf{1}^\intercal \bm{Y}_i) > 0.5\}.
    \end{aligned} \label{dgp}
\end{equation}
I then artificially generate machine-predicted labels $\hat{L}_i$ with accuracies ranging from 70\% to 90\% by randomly flipping the value of $\tilde  L_i$ according to the selected error percentage (e.g., when the accuracy of $\tilde  L_i$ s set to 90\%, 10\% of the values of $\hat{L}_i$ are randomly changed).

To implement the proposed SRI estimator, I use 2-fold cross-fitting with the neural network architecture illustrated in Figure~\ref{dragonnet}. The network consists of three shared layers with dimensions 100, 100, and 50, producing the shared representation $\boldf(\bR_i)$. This shared layer is then connected to two prediction heads—one for the outcome and one for the predictor of interest—each composed of two layers with dimensions 50 and 1, respectively. I use ReLU activations and solve the optimization problem in Equation~\ref{loss_main} with $\alpha = 1$ using the Adam optimizer, a learning rate of 0.00002, a maximum of 200 epochs, and a batch size of 256 \citep{kingma_adam_2017}. I reserve 20\% of the data for validation and apply early stopping with a patience of five epochs to prevent overfitting (i.e., training stops if the validation loss does not decrease for five consecutive epochs). Once the outcome model and surrogacy score are estimated, I compute the difference in outcome means using the estimating equation in Equation~\eqref{est_eq}. 

As in the empirical application in Section~\ref{sec::application}, I evaluate the performance of the proposed estimator by comparing it with existing bias-correction methods, Prediction-Powered Inference (PPI, \citealt{angelopoulos2023prediction} and Design-Based Supervised Learning (DSL; \citealt{egami_neulips_2023, egami2024using}). For simplicity, I omit control variables $\bZ_i$ from this simulation. When $\bZ_i$ is included, the proposed estimator must additionally estimate the propensity score, which may introduce efficiency differences relative to alternative methods that do not require this step \citep{hirano_efficient_2003}. Given the simplicity of the data-generating process, this simulation design allows me to clearly assess the extent to which the imposed independence constraint improves efficiency.
 
\subsubsection{Simulation results}

\begin{figure}[t]
    \centering
    \includegraphics[width=1.0\linewidth]{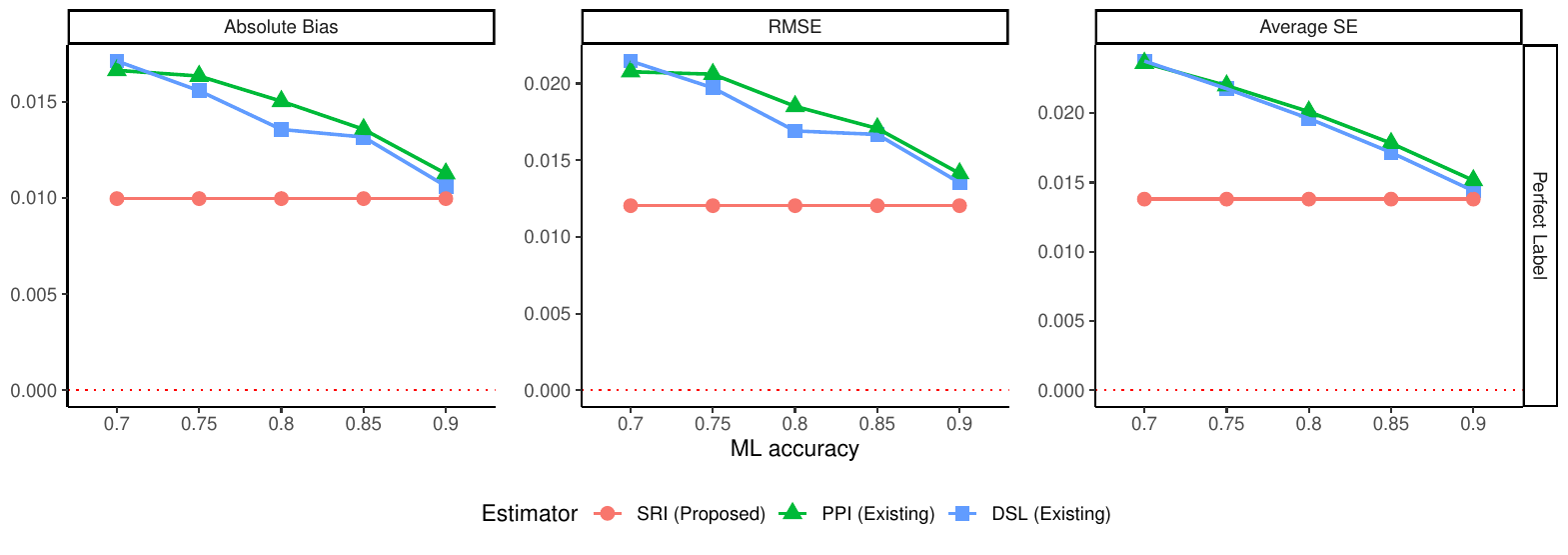}
    \caption{Performance comparison between the Surrogate Representation Inference (SRI, proposed) and the existing bias-correction algorithms that corrects measurement error from machine learning predictions (PPI; \citealt{angelopoulos2023prediction} and DSL; \citealt{egami_neulips_2023, egami2024using}). I use the human annotations with no mistakes, and vary the classification accuracy of machine learning predictions (from 0.70 to 0.90). All results are based on 200 Monte Carlo simulations.}
    \label{result}
\end{figure}

Figure~\ref{result} presents the results of the simulation studies with perfectly observed human-annotated labels based on 200 Monte Carlo replications. Numerical results are reported in Table~\ref{tab:result1}. 

The results show that the proposed SRI method outperforms existing bias-correction techniques across all levels of classification accuracy for the proxy labels $\hat{L}_i$. Although existing bias-correction approaches correct for prediction errors, they suffer from large root mean squared errors (RMSEs) and large standard errors, particularly when classification accuracy ranges from 70\% to 80\%. As a consequence, their estimated standard errors are also inflated. For example, at 80\% classification accuracy, the average standard error for the existing approach is 0.020, compared to 0.014 for the proposed SRI estimator. This difference indicates that even at moderately high classification accuracy levels, such as 80\%, the proposed estimator achieves a substantial reduction in estimation uncertainty.

Overall, these results confirm that the SRI estimator provides a useful alternative to existing bias-correction techniques, especially in settings where researchers only have access to machine-learning predictions with moderate classification accuracy.

\begin{table}[h]
    \centering
   \begin{tabular}{lccccc}
    \toprule
        & & & &  \multicolumn{2}{c}{95\% Conf. Int} \\
        & \multicolumn{1}{c}{ML Acc.}  
        & \multicolumn{1}{c}{Bias} 
        & \multicolumn{1}{c}{RMSE} 
        & \multicolumn{1}{c}{Coverage} 
        & \multicolumn{1}{c}{Avg. Length} \\
    \midrule
        Surrogate Representation Inference (SRI) 
        & -- 
        & 0.001 
        & 0.012 
        & 0.985 
        & 0.054 \\
    \midrule
        Existing Approach (DSL) & 0.70 & -0.001 & 0.022 & 0.955 & 0.093 \\
         & 0.75 & -0.003 & 0.020 & 0.950 & 0.086 \\
         & 0.80 & -0.001 & 0.017 & 0.955 & 0.077 \\
         & 0.85 &  0.000 & 0.017 & 0.945 & 0.067 \\
         & 0.90 &  0.002 & 0.014 & 0.940 & 0.056 \\
    \midrule
        Existing Approach (PPI) & 0.70 & -0.002 & 0.021 & 0.970 & 0.093 \\
         & 0.75 & -0.003 & 0.021 & 0.950 & 0.086 \\
        & 0.80 & -0.002 & 0.019 & 0.960 & 0.079 \\
         & 0.85 &  0.000 & 0.017 & 0.945 & 0.070 \\
         & 0.90 &  0.002 & 0.014 & 0.935 & 0.060 \\
    \bottomrule
\end{tabular}
    \caption{Simulation Results with 200 Monte Carlo Trials with Perfect Human Annotations}
    \label{tab:result1}
\end{table}

\subsection{Case 2: Human Annotations with Non-Differential Errors}
\subsubsection{Simulation setup}\label{sec::sim_proximal_setup}
Next, I examine the validity of the proposed methodology in Section \ref{sec:proximal} under a setting in which researchers do not observe the true label $L_i$ and instead have access only to two noisy proxies, $\tilde L_i^{(1)}$ and $\tilde L_i^{(2)}$. I use the same data-generating process as in Equation~\eqref{dgp} with $n = 20{,}000$, but without observing the true human-annotated label $L_i$. Specifically, I artificially generate two imperfect human-annotated labels, $\tilde L_i^{(1)}$ and $\tilde L_i^{(2)}$, by randomly flipping the value of $L_i$ with classification accuracy ranging from 80\% to 90\%, in the same manner used to construct the machine-learning predictions $\hat L_i$. Because the classification accuracy of the imperfect labels is controlled and the errors are generated at random, all assumptions introduced in Section \ref{sec:proximal} are satisfied.

For estimation, I implement the proposed methodology using the 5-fold cross-fitting procedure described in Section \ref{sec:proximal}. In each fold, I first estimate the coder classification probabilities using half of the sample and then train the neural network on the remaining half. As in the previous case, the neural network consists of two hidden layers with identical dimensions leading to a shared representation layer $\boldf(\bR_i)$. This shared layer is connected to separate prediction heads for the predictor of interest, the human-annotated label, and the outcome for each category class. Each head adopts the same architecture as in the previous case. I use ReLU activations and solve the optimization problem in Equation~\eqref{loss_proximal} with $\beta = 1.0$, employing the Adam optimizer with the same hyperparameters and early-stopping criteria.

As before, I compare the proposed SRI estimator with existing bias-correction approaches, PPI and DSL, but in this setting I use noisy human-annotated labels in conjunction with machine-learning predictions. Because the two noisy human-annotated labels are artificially generated in the symmetric way, I use the first proxy as the human-annotated label for the existing approaches. For the machine-learning predictions, I again vary the classification accuracy from 70\% to 90\%. Since the ground-truth human-annotated labels are unobserved, all estimators except the proposed SRI estimator are expected to exhibit substantial bias and large RMSE.

\subsubsection{Simulation results}

\begin{figure}[t]
    \centering
    \includegraphics[width=1.0\linewidth]{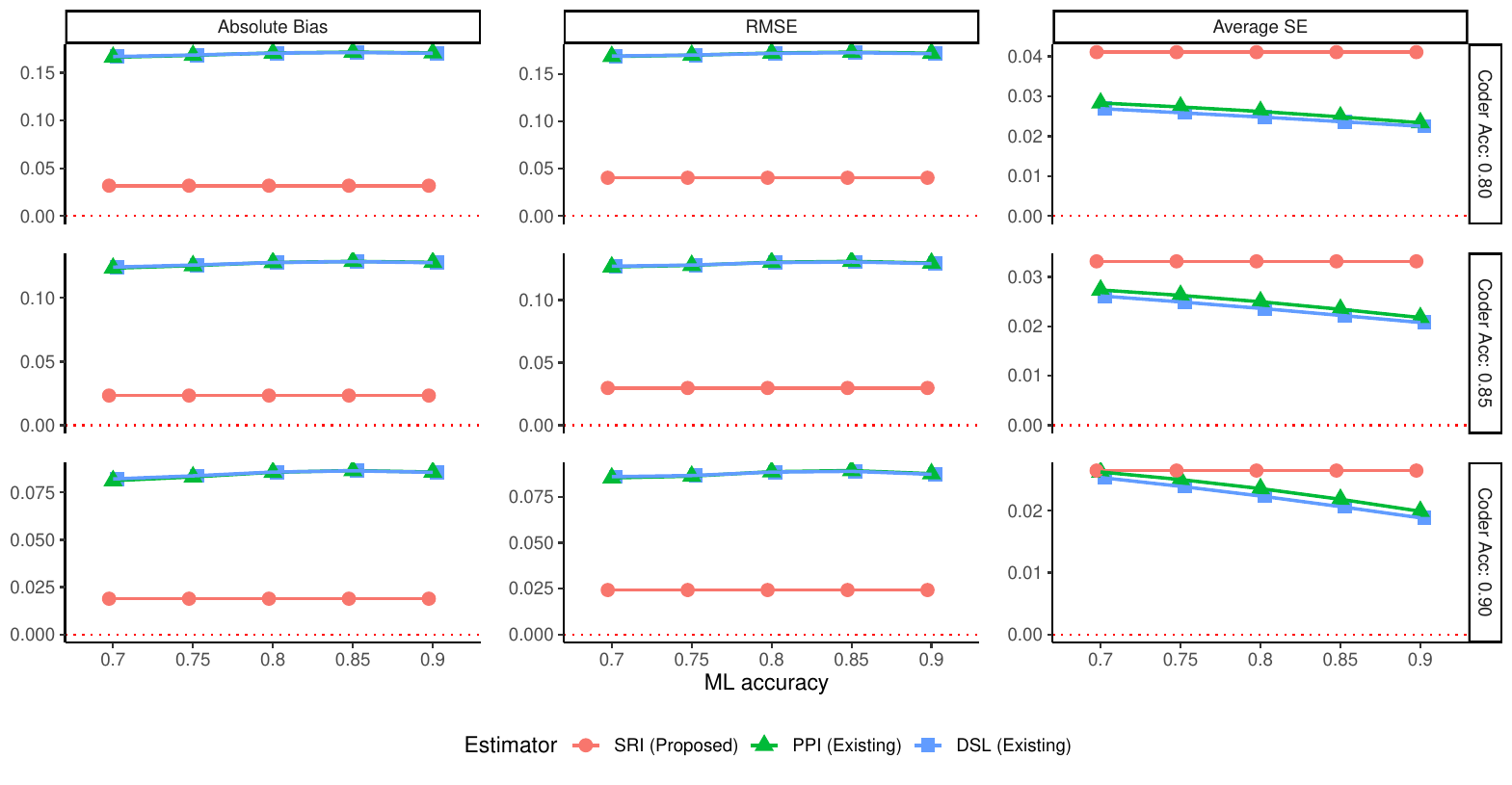}
    \caption{Performance comparison between the Surrogate Representation Inference (SRI, proposed) and the existing bias-correction algorithm that corrects measurement error from machine learning predictions (PPI; \citealt{angelopoulos2023prediction} and DSL; \citealt{egami_neulips_2023, egami2024using}). I use the human annotations with non-differential errors, and vary both the classification accuracy of machine learning predictions (from 0.70 to 0.90, x-axis) and the classification accuracy of human annotations (from 0.80 to 0.90, y-axis). All results are based on 200 Monte Carlo simulations.}
    \label{result2}
\end{figure}

Figure \ref{result2} presents the results of the simulation study based on 200 Monte Carlo trials using two human-annotated labels that contain non-differential errors. Numerical results are reported in Table~\ref{tab:result2}.

The results show that the proposed SRI methodology successfully corrects the bias introduced by noisy human annotations, yielding smaller absolute bias and RMSE than the existing bias-correction approaches across all levels of machine-learning classification accuracy. In contrast to the previous scenario, the performance of the existing bias-correction approaches does not improve as the quality of machine-learning predictions increases, because the human annotations contain unknown non-differential errors.

Across all settings, however, the SRI estimator generally exhibits larger standard errors than the existing bias-correction approaches, particularly when the classification accuracy of the human annotations is lower. This occurs because the proposed estimator infers the latent true labels from multiple noisy annotations; when those annotations are more error-prone, the resulting estimates become noisier. Therefore, although the SRI estimator is theoretically guaranteed to remain unbiased in the presence of random annotation errors, researchers should nevertheless strive to obtain high-quality human annotations to minimize estimation variance.

\begin{table}[]
    \centering
    \begin{tabular}{lcccccc}
\toprule
 & & & & & \multicolumn{2}{c}{95\% Conf. Int} \\
Method 
& Coder Acc.
& ML Acc.
& Bias 
& RMSE 
& Coverage 
& Avg. Length \\
\midrule
SRI & 0.90 & -- & 0.009 & 0.024 & 0.970 & 0.104 \\
 & 0.85 & -- & 0.012 & 0.030 & 0.975 & 0.130 \\
 & 0.80 & -- & 0.013 & 0.041 & 0.960 & 0.161 \\
\midrule
DSL & 0.90 & 0.70 & -0.082 & 0.086 & 0.080 & 0.100 \\
 &  & 0.75 & -0.084 & 0.087 & 0.025 & 0.094 \\
 &  & 0.80 & -0.086 & 0.089 & 0.015 & 0.088 \\
 &  & 0.85 & -0.087 & 0.089 & 0.005 & 0.081 \\
 &  & 0.90 & -0.086 & 0.087 & 0.000 & 0.074 \\
\midrule
DSL & 0.85 & 0.70 & -0.124 & 0.127 & 0.000 & 0.103 \\
 &  & 0.75 & -0.126 & 0.128 & 0.000 & 0.098 \\
 &  & 0.80 & -0.128 & 0.130 & 0.000 & 0.092 \\
 &  & 0.85 & -0.128 & 0.130 & 0.000 & 0.087 \\
 &  & 0.90 & -0.128 & 0.129 & 0.000 & 0.082 \\
\midrule
DSL & 0.80 & 0.70 & -0.167 & 0.169 & 0.000 & 0.105 \\
& & 0.75 & -0.169 & 0.170 & 0.000 & 0.101 \\
& & 0.80 & -0.171 & 0.172 & 0.000 & 0.097 \\
& & 0.85 & -0.171 & 0.173 & 0.000 & 0.093 \\
& & 0.90 & -0.170 & 0.172 & 0.000 & 0.088 \\
\midrule
PPI & 0.90 & 0.70 & -0.081 & 0.085 & 0.110 & 0.103 \\
& & 0.75 & -0.083 & 0.086 & 0.045 & 0.098 \\
& & 0.80 & -0.086 & 0.089 & 0.025 & 0.093 \\
& & 0.85 & -0.087 & 0.089 & 0.030 & 0.086 \\
& & 0.90 & -0.086 & 0.088 & 0.005 & 0.078 \\

\midrule
PPI & 0.85 & 0.70 & -0.123 & 0.126 & 0.005 & 0.107 \\
& & 0.75 & -0.125 & 0.128 & 0.000 & 0.103 \\
& & 0.80 & -0.128 & 0.130 & 0.005 & 0.098 \\
& & 0.85 & -0.129 & 0.131 & 0.005 & 0.092 \\
& & 0.90 & -0.128 & 0.130 & 0.000 & 0.085 \\
\midrule
PPI & 0.80 & 0.70 & -0.166 & 0.169 & 0.000 & 0.111 \\
& & 0.75 & -0.168 & 0.170 & 0.000 & 0.107 \\
& & 0.80 & -0.171 & 0.172 & 0.000 & 0.103 \\
& & 0.85 & -0.172 & 0.173 & 0.000 & 0.098 \\
& & 0.90 & -0.171 & 0.172 & 0.000 & 0.092 \\
\bottomrule
\end{tabular}
    \caption{Simulation Results with 200 Monte Carlo Trials with Imperfect Human Annotations}
    \label{tab:result2}
\end{table}

\end{document}